\documentclass{article}

\usepackage[main, final]{neurips_2026}
\usepackage{amsmath}
\usepackage{makecell}
\usepackage{bm}
\usepackage{graphicx}
\usepackage{algorithm}
\usepackage{algorithmicx}
\usepackage{algpseudocode}
\usepackage{multirow}
\usepackage{subcaption}
\usepackage{tikz}
\usepackage{wrapfig}
\usepackage{amsthm}
\usepackage{caption}

\newtheorem{theorem}{Theorem}[section]
\newtheorem{lemma}[theorem]{Lemma}
\newtheorem{proposition}[theorem]{Proposition}
\newtheorem{assumption}[theorem]{Assumption}
\newtheorem*{proposition*}{Proposition}
\newcommand{\DKL}{D_{\text{KL}}}

\usepackage[utf8]{inputenc} 
\usepackage[T1]{fontenc}    
\usepackage{hyperref}       
\usepackage{url}            
\usepackage{booktabs}       
\usepackage{amsfonts}       
\usepackage{nicefrac}       
\usepackage{microtype}      
\usepackage{xcolor}         
\usepackage{enumitem}

\title{Primal-Dual Flow Matching for Sample-Wise Constrained Generation}

\author{%
  Zhengyan Huan\textsuperscript{1,2}\qquad
  Peter Y. Lu\textsuperscript{1,2}\qquad
  Shuchin Aeron\textsuperscript{1,2}\\[3pt]
  {\normalfont\textsuperscript{1}Department of Electrical and Computer Engineering, Tufts University}\\{\normalfont\textsuperscript{2}NSF AI Institute for Artificial Intelligence and Fundamental Interactions}\\[3pt]
  {\normalfont\texttt{\{Zhengyan.Huan, Peter.Lu, Shuchin.Aeron\}@tufts.edu}}%
}

\begin{document}

\maketitle

\begin{abstract}

We consider the problem of learning to generate data while satisfying a prescribed \emph{sample-wise} constraint with a user-specified probability of constraint satisfaction. Starting with a population-level constrained negative log-likelihood minimization and using a flow-matching (FM) parameterization, under regularity assumptions and sufficient expressivity of the velocity field family, we derive guarantees on the probability of constraint satisfaction for the corresponding primal-dual optimization algorithm. We provide a practical implementation of this approach, called Primal-Dual Flow Matching (PDFM), which has direct control over constraint satisfaction rate and requires only binary membership oracle feedback for constraint satisfaction, avoiding the need for differentiable distances, projections, or special structure of the constraint set, such as convexity. We show that PDFM improves constraint satisfaction over vanilla FM and achieves competitive performance against constraint-aware baselines while maintaining distributional fidelity. Our code is available at \url{https://github.com/ZhengyanHuan/PDFM}.
\end{abstract}

\section{Introduction}

The utility of generative models for scientific applications highly depends on their ability to generate samples that obey scientific laws or satisfy an application-specific criterion, modeled via constraints that can be validated numerically. 
For example, flow- and diffusion-based generative models are increasingly used to emulate chaotic spatiotemporal systems \citep{lienen2024from,Du2024,Price2025}, where the chaotic dynamics result in high sensitivity to initial conditions and therefore unpredictable trajectories that can be interpreted and modeled probabilistically \citep{medio_lines_2001}. However, in these settings, distributional fidelity alone is rarely sufficient: a generated sample is useful only if it is also \emph{physically consistent}, e.g., satisfying boundary conditions, respecting conservation laws, or following a governing PDE \citep{PhysRevLett.126.098302,pcfm2025}. Another important example comes from molecular generation \citep{hoogeboom2022equivariant, song2023equivariant}, where one wants to generate new molecules such that they conform to structures with chemically useful properties \citep{bickerton2012quantifying,giaginis2018impact} that are numerically quantifiable.
\paragraph{Related work:} For constrained generation, one may divide existing literature into (a) Type of constraints, e.g., distributional constraints \citep{Khalafi} or sample-wise constraints [this work]. Among sample-wise constraints, one may have constraints with topological specification, such as convexity \citep{mirror}; constraints with known boundaries \citep{loureflect, fishmanreflect}; constraints with differentiable distances to them \citep{feng2025neural, pcfm2025}; or constraints that are known only via a membership oracle \citep{FMRE}; and (b) {Methods} used -- whether it is inference-time 
\citep{Christopher2024ConstrainedSW} \citep{pcfm2025} \citep{li2026hardflowhardconstrainedsamplingflowmatching} \citep{doobh} or training-time [this work], \citep{loureflect, fishmanreflect} \citep{reflectedfm} \citep{feng2025neural}  \citep{Giorgio} \citep{FMRE}. A more detailed review of these related works is provided in Appendix \ref{sec:relatedworks} and throughout the exposition we will highlight other closely related works.

For sample-wise constrained generation when the constraints are accessible via a membership oracle, we make the following \textbf{main contributions}:

\begin{enumerate}[leftmargin=*]
\setlength \itemsep{0pt}
    \item We develop a theoretically grounded approach, with provable guarantees on the probability of constraint satisfaction. In particular, leveraging an FM parameterization of a constrained negative log-likelihood minimization (Sections \ref{sec:setup}, \ref{sec:FMpara}), in Theorem \ref{theorem:constraintsatisfaction} we show that under mild assumptions the Total Variation (TV) norm between the optimal solution and the FM parameterized solution can be controlled in terms of approximation error between the flows due to parametric modeling, which in turn leads to overall control on constraint satisfaction. We then show that a \emph{theoretical} Primal-Dual Flow Matching (PDFM) Alg.~\ref{alg:pd} achieves this bound. 
    \item We then propose a practical implementation of PDFM  in Alg.~\ref{alg:pd_practical}, with theoretical guarantees discussed in Sec.~\ref{sec:PDFMguarantee}. Key innovations in this algorithm are motivated by deterministic policy gradient methods, where we build differentiable surrogates for  constraint satisfaction probability under the FM parameterization. This enables the method to work with \emph{only membership oracle access to the constraint set.} We note that this type of access to the constraint is the most flexible when the constraints are complex. 
    \item We provide an extensive evaluation of PDFM and show that compared to existing methods, it achieves competitive constraint satisfaction rates and distribution fidelity. For scientific applications, we consider the problem of learning to generate PDE-constrained spatiotemporal dynamics of chaotic systems and show that our method achieves superior performance while requiring only binary feedback. In a molecule generation task, PDFM improves molecule stability and distributional fidelity over FM over FM, without task-specific designs.
\end{enumerate}

\section{Constrained Generation Setup}

\label{sec:setup}

All notations used are summarized in Table~\ref{tab:notations} in Appendix~\ref{app:notations} for reference. 
Let $Q$ be the training-data distribution on $\mathbb{R}^d$ with density $q$ and let $P$ be the distribution on $\mathbb{R}^d$ to be modeled, with the corresponding density $p$. Let $\mathcal{C}\subset \mathbb{R}^d$ be a closed, application-specific constraint set, for example, the set of particle trajectories that satisfy a prescribed conservation-law tolerance. We denote $\bm{1}_{\mathcal{C}}(x)$ by the indicator function of the constraint set $\mathcal{C}$ that takes the value $1$ if $x \in \mathcal{C}$ and $0$ otherwise. 
Given a target constraint satisfaction rate $b\in(0,1)$, consider the following population-level problem:
\begin{equation}\label{eq:cvxoriginal}
\operatorname*{minimize}_{P}~~
   \mathbb{E}_{x\sim Q}[-\log p(x)]~~\mathrm{s.t.}~~ b-\int_\mathcal{C} p(x) dx\leq 0. \tag{OPT}
\end{equation}
\eqref{eq:cvxoriginal} requires that under the modeled distribution $P$, the probability of constraint satisfaction is at least $b$. In other words $\mathbb{E}_{x\sim P}[\mathbf{1}_{\mathcal{C}}(x)] = \int_{\mathcal{C}} p(x) dx \geq b$. Assume additionally that $\mathbb {E}_{x\sim Q}[\log q(x)]$ is finite, the problem can equivalently be restated as:
\begin{equation}
 \operatorname*{minimize}_{P}~~ \mathrm{KL}(Q||P) ~~\mathrm{s.t.}~~ b-\int_\mathcal{C} p(x) dx\leq 0,
\end{equation} where $\mathrm{KL}(Q||P) = \mathbb{E}_{x \sim Q}[\log \frac{q(x)}{p(x)}]$.

Define $Q(\mathcal{C}):= \int_{\mathcal C} q(x)dx, P(\mathcal{C}): = \int_{\mathcal C} p(x)dx $.
Let us consider two regimes: 1. $0<Q(\mathcal{C})<b$, and 2. $b\leq Q(\mathcal{C})\leq 1$. For the first regime, a quick observation shows that the solution $P=Q$ is infeasible due to the constraint. 
We additionally note that at the boundary $b=1$: if $Q(\mathcal{C})< 1$, then every feasible $P$ must satisfy $P(\mathcal{C})=1$ and $P(\mathcal{C}^c)=0$, while $Q(\mathcal{C}^{c})>0$ leads to infinite objective value. 
For the second regime, even though  $P=Q$ is the optimal solution, this problem remains practically non-trivial since the learned generative model may still violate the constraint due to the finite model capacity and approximation via limited samples.

One can immediately note that \ref{eq:cvxoriginal} is a convex optimization problem. 
Introducing a multiplier $\lambda \geq 0$, the Lagrangian is given by
\begin{equation}
    L(p,\lambda) = \mathbb{E}_{x\sim Q}[-\log p(x)]+\lambda\left(b-\int_{\mathcal{C}}p(x)dx\right).
\end{equation}
 
\begin{assumption}\label{ass:feasibility} 
Assume that $0<Q(\mathcal C)<b<1$, that there exists a strictly feasible
distribution $P$ satisfying
$b-\int_{\mathcal C}p(x)\,dx<0$,
and that, for every $\lambda\geq 0$ considered in the analysis, the minimum of
$L(\cdot,\lambda)$ is attained. We denote the corresponding minimizer by
$p^{(\lambda)}\in\arg\min_p L(p,\lambda)$, and let $P^{(\lambda)}$ be the
distribution with density $p^{(\lambda)}$.
\end{assumption}

\begin{lemma} \label{lemma:pstar}
Under Assumption~\ref{ass:feasibility},
    \begin{equation}
        p^{(\lambda^*)}(x) = \frac{bq(x)}{Q(\mathcal{C})}\bm{1}_{\mathcal{C}}(x)+\frac{(1-b)q(x)}{1-Q(\mathcal{C})}\bm{1}_{\mathcal{C}^c}(x),~~\lambda^* = \frac{1-Q(\mathcal{C})}{1-b}-\frac{Q(\mathcal{C})}{b}.
    \end{equation}
\end{lemma}
This lemma, which will be proved in Appendix~\ref{sec:proofpart0}, states that the optimal population-level solution corresponds to the re-weighting of the density of $Q$ inside and outside of $\mathcal{C}$.

In the regime $0<Q(\mathcal C)<b<1$, a simple
alternative to enforce constraint satisfaction is to discard infeasible samples and train on the truncated
distribution $Q' = Q(\cdot\mid \mathcal C)$, whose density is $ q'(x)=\frac{q(x)\mathbf 1_{\mathcal C}(x)}{Q(\mathcal C)}.$
However, practically, especially when the training data are noisy,
this rejection step can substantially reduce the effective sample size and
degrade the training distribution. Moreover, even with sufficient truncated
samples, finite-capacity generative models may still fail to achieve the target
constraint satisfaction rate. This will be shown in Sec.~\ref{sec:2dtoy} and \ref{sec:PDE}. 
\subsection{Problem Regimes and Learning Objective}
We can now state our problem regimes and the learning objective. Let
$x^{(j)}\overset{\mathrm{i.i.d.}}{\sim}Q$, $j\in[n]$, be the training
samples, and let $b\in(0,1)$ be the target constraint satisfaction rate. When
$0<Q(\mathcal C)<b<1$, the goal is to train an FM model whose generated distribution approximates the optimal constrained distribution
$P^{(\lambda^*)}$ characterized in Lemma~\ref{lemma:pstar} with the desired
constraint satisfaction rate. This regime is
studied both theoretically and empirically. When
$0<b\leq Q(\mathcal C)\leq 1$, the distribution-level optimum is $Q$, and therefore the objective is to train an FM model that matches $Q$ while achieving the desired
constraint satisfaction rate, and we study this setting empirically. The constraint set $\mathcal{C}$ is specified only through a membership oracle: given a sample, the oracle determines whether it belongs to $\mathcal{C}$. Differentiable distances, projections, and structural properties such as convexity are unavailable.

\section{Constraint Guarantees with FM Parameterization}\label{sec:FMpara}

\eqref{eq:cvxoriginal} yields a convex distribution-level problem with a unique optimizer given by Lemma~\ref{lemma:pstar}. In practical generative modeling through FM, the distribution is represented with a velocity field, whose approximation error controls KL divergence through Lemma~\ref{lemma:TVgap1}. The development in this section analyzes the total variation (TV) between the FM-parameterized solution and the distribution-level optimum. The following Lemma~\ref{lemma:TVandconstraint}, proved in Appendix~\ref{sec:proofpart0}, translates TV into a bound on the constraint satisfaction rate.

\begin{lemma}\label{lemma:TVandconstraint}
   Let $P_1,P_2$ be probability measures defined on the same measurable space. If the TV distance between $P_1$ and $P_2$ is given as $\|P_1-P_2\|_{\mathrm{TV}}=\delta $,  then
\begin{equation}
    \max\{0, P_2(\mathcal{C})-\delta\} \leq P_1(\mathcal{C}) \leq \min\{1, P_2(\mathcal{C})+\delta\} .
\end{equation}
\end{lemma}

Let $\theta \in \Theta$ denote the parameters of the FM model, and let
$u_\theta:\mathbb{R}^d \times [0,1] \to \mathbb{R}^d$ be the corresponding
parameterized velocity field. The induced velocity field class is defined as
$U_\Theta := \{u_\theta : \theta \in \Theta\}$. 
Given an initial distribution $P_0$, the velocity field $u_\theta$ defines a flow map $\Phi_{u_\theta, t}$, leading to the generated distribution $P_{u_\theta}$,
\begin{equation}
    \Phi_{u_\theta, t}(x) = x+\int_{0}^tu_\theta(\Phi_{u_\theta, s}(x),s)ds,\quad P_{u_\theta} := (\Phi_{u_\theta, 1})_{\#}P_0,
\end{equation}
where for any map $T$, $T(x)$ denotes the transported sample starting from $x$, while $(T)_{\#}P_0$ denotes the pushforward distribution obtained by transporting $P_0$ through $T$, i.e., if $x_0\sim P_0$ then $\Phi_{u_\theta, t}(x_0)\sim (\Phi_{u_\theta, t})_{\#}P_0$.
The FM-parameterized Lagrangian is defined as
\begin{equation}
    \bar{L}(u_\theta,\lambda)
    :=
    L(p_{u_\theta},\lambda)
    =
    \mathbb{E}_{x\sim Q}\big[-\log p_{u_\theta}(x)\big]
    +
    \lambda
    \left(
        b-\int_{\mathcal{C}} p_{u_\theta}(x)\,dx
    \right).
\end{equation}
The corresponding parameterized dual function is
$\bar{g}(\lambda):=\inf_{u\in U_\Theta} \bar{L}(u,\lambda)$, and we denote an optimal dual variable by
$  \bar{\lambda}^* \in \arg\max_{\lambda\geq 0} \bar{g}(\lambda)$.
For each fixed $\lambda\geq 0$, let $
    u_\theta^{(\lambda)}\in\arg\min_{u\in U_\Theta} \bar{L}(u,\lambda)$
denote a minimized parameterized velocity field, and let
$P_{u_\theta}^{(\lambda)} := (\Phi_{u_{\theta}^{(\lambda)}, 1})_{\#}P_0$ denote the distribution induced by
$u_\theta^{(\lambda)}$, with density $p_{u_\theta}^{(\lambda)}$. $\bar{L}(u_\theta,\lambda)$ is generally nonconvex in both $\theta$ and $u_\theta$.

Defining the velocity field $u^{(\lambda)}$ for $P^{(\lambda)}$ in Assumption~\ref{ass:feasibility} requires specifying both a coupling and an
interpolation strategy; although many choices are possible, we consider the
following construction. Let $\pi^{(\lambda)}$ be a coupling between the initial distribution $P_0$ and $P^{(\lambda)}$. For
$(x_0,x_1)\sim \pi^{(\lambda)}$, define the linear interpolant $x_t=(1-t)x_0+t x_1, t\in[0,1]$. The corresponding true velocity field is defined by
$u^{(\lambda)}(x,t):= \mathbb{E}_{(x_0,x_1)\sim\pi^{(\lambda)}}\left[x_1-x_0 \mid x_t=x\right]$.
\begin{assumption}\label{ass:esterror}
Assume that, for every dual variable $\lambda$ considered in the analysis, the
target probability path, the corresponding true
velocity field $u^{(\lambda)}$, and the parameterized velocity field
$u_\theta^{(\lambda)}$ satisfy the regularity conditions of Assumptions~3.2
and~3.3 in~\cite{su2025flowmatchingkldivergence}: the target path has uniformly
bounded score and score Jacobian, and the true and parameterized velocity fields
are sufficiently smooth with bounded velocity norms, Jacobians, divergence
discrepancy, and gradients of divergence.
In addition, the parameter family $\Theta$ is sufficiently rich so that the
FM parameterization error is bounded by $\nu>0$: for every $\lambda$ considered
in the analysis, 
\begin{equation}
     \begin{aligned}
\mathbb{E}_{t,x_t}\!\left[\|u^{(\lambda)}_\theta(x_t,t)-u^{(\lambda)}(x_t,t)\|_2^2\right] \le \nu^2.
     \end{aligned}
 \end{equation}
     \end{assumption}
\begin{assumption}
     \label{ass:dual_upper_bound}
          There exists a constant $\lambda_{\max}<\infty$ such that $\lambda^*,\bar{\lambda}^* \in [0,\lambda_{\max}]$.
     \end{assumption}
     
     \begin{assumption}\label{ass:densityrange}
For every $\lambda$ considered and every
$x$ such that $q(x)>0$, both $p_{u_\theta}^{(\lambda)}(x)>0$ and
$p^{(\lambda)}(x)>0$ hold. 
\end{assumption}

\begin{theorem}
    \label{theorem:constraintsatisfaction}
For problem \eqref{eq:cvxoriginal} with FM parameterization, under Assumptions \ref{ass:feasibility}, \ref{ass:esterror}, \ref{ass:dual_upper_bound}, \ref{ass:densityrange},
\begin{equation}
\begin{aligned}
        \|P^{(\bar{\lambda}^*)}_{u_{\theta}}-P^{(\lambda^*)}\|_{\mathrm{TV}}\leq \mathcal{E}_{\mathrm{para}},\quad P_{u_{\theta}}^{(\bar{\lambda}^*)}(\mathcal{C})\geq b -\mathcal{E}_{\mathrm{para}}   ,     
        \end{aligned}
        \end{equation}
\begin{equation}
    \mathcal{E}_{\mathrm{para}} = \frac{1}{\left((1-Q({\mathcal{C}}))^{1/3}+Q({\mathcal{C}})^{1/3}\right)^3}\sqrt{\frac{2}{k}\mu(\bar{\lambda}^*) (A_1\nu + A_2\nu^2)}+\sqrt{\frac{1}{2}(A_1 \nu + A_2 \nu^2 )}
\end{equation}
in which $A_1,A_2>0$ are constants that will be specified in Lemma~\ref{lemma:TVgap1}, $k$ is the strong concavity constant of $g(\lambda)$, which will be specified in  Lemma~\ref{lemma:kconcave}, and $\mu(\lambda) = \frac{1}{2}\left(\lambda+1+\sqrt{(\lambda+1)^2-4\lambda(1-Q(\mathcal{C}))}\right)$.
\end{theorem}

Theorem \ref{theorem:constraintsatisfaction}  shows that, provided the FM parameterization error $\nu$ is sufficiently small, the FM-parameterized constrained solution induces a distribution $P^{(\bar{\lambda}^*)}_{u_\theta}$ that remains close in total variation to the distribution-level optimal constrained distribution $P^{(\lambda^*)}$. This bound separates two main sources of error: the FM
approximation error, controlled by $A_1\nu+A_2\nu^2$, and the dual mismatch
between $\bar{\lambda}^*$ and $\lambda^*$, which is implicitly induced by the FM
approximation error and is controlled through the strong
concavity constant $k$. Consequently, $P^{(\bar{\lambda}^*)}_{u_\theta}$ satisfies the target
constraint rate $b$ up to an explicit error term that vanishes as $\nu$
decreases. The proof of Theorem~\ref{theorem:constraintsatisfaction} is deferred to Appendix~\ref{sec:proofpart1}.

\section{Primal-Dual Flow Matching (PDFM) for Constrained Generation}

The analysis in the previous section assumes access to the optimal parameterized velocity field $u_\theta^{(\lambda)}$ for each fixed dual variable $\lambda$. In practice, however, both the velocity field
and the dual variable must be learned through an implementable optimization procedure. In this section, we first introduce a theoretical PDFM algorithm for solving \ref{eq:cvxoriginal}, then propose
practical implementations of PDFM, and finally analyze the errors in the practical PDFM algorithm.

\subsection{Theoretical PDFM Algorithm}

To solve the problem practically, we begin with an idealized primal-dual procedure shown in Alg.~\ref{alg:pd} that serves as the theoretical template for PDFM. Primal-dual methods are a standard framework for constrained optimization
\citep{boyd2004convex} and have also been applied to generative modeling with
distributional constraints, i.e., requiring the generated distribution to remain
within a prescribed divergence from reference distributions \citep{Khalafi}.
At each iteration, \begin{wrapfigure}{r}{0.47\textwidth}
\centering
\begin{minipage}{0.45\textwidth}
\hrule height 0.8pt
\vspace{2pt}
\captionsetup[algorithm]{
    justification=raggedright,
    singlelinecheck=false,
    labelfont=bf
}
\captionof{algorithm}{Theoretical PDFM}
\label{alg:pd}
\vspace{-0.6em}
\hrule
\small
\begin{algorithmic}
    \State \textbf{Input:} $b$, learning rate $\eta_i$
    \State $\lambda(0)\gets 0$
    \For{$i=0,1,\cdots$}
        \State $u^{(\lambda(i))}_\theta
        \gets \arg\min_{u_\theta\in U_\Theta}
        \bar{L}(u_\theta,\lambda(i))$
        \State $\lambda(i+1)\gets
        [\lambda(i)+\eta_i(b-P^{(\lambda(i))}_{u_\theta}(\mathcal C))]_+$
    \EndFor
\end{algorithmic}
\vspace{2pt}
\hrule
\end{minipage}
\end{wrapfigure}the primal step aims to minimize the parameterized Lagrangian $\bar{L}(u_\theta,\lambda)$ over $U_\Theta$ for the current dual variable $\lambda$, 
while the dual step updates $\lambda$ according to $\nabla_{\lambda}\bar{L}(u_\theta,\lambda) = b-P_{u_\theta}(\mathcal{C})$, i.e., the constraint satisfaction gap. Since $\bar{L}(u_\theta,\lambda)$ is affine in $\lambda$ for every fixed $u_\theta$,
$\bar{g}$ is concave as the pointwise infimum of affine functions. Suppose that the infimum defining $\bar g(\lambda)$ is attained for all $\lambda$ considered in the analysis and that the learning rate $\eta_i$ satisfies
$\eta_i>0$, $\sum_i\eta_i=\infty$, and $\sum_i\eta_i^2<\infty$. And note that the supergradients $\nabla_{\lambda}\bar{L}(u^{(\lambda)}_\theta,\lambda) = b-P^{(\lambda)}_{u_\theta}(\mathcal{C})$ are uniformly bounded in $[b-1,b]$. The dual update is a projected supergradient ascent method for maximizing the concave function
$\bar{g}$. Thus, under Alg.~\ref{alg:pd} with exact primal minimization and exact dual supergradient updates, $\lambda$ converges to some
$\bar{\lambda}^*\in\arg\max_{\lambda\geq 0}\bar{g}(\lambda)$.
If, in addition, the induced primal solution at $\bar{\lambda}^*$ is unique and the primal solution correspondence is continuous at $\bar{\lambda}^*$, then the corresponding primal distributions converge to the parameterized solution
$P_{u_{\theta}}^{(\bar{\lambda}^*)}$
characterized in Sec.~\ref{sec:FMpara}.

\subsection{Practical Implementation of PDFM} \label{sec:gradimp}

\begin{wrapfigure}{r}{0.47\textwidth}
\centering
\begin{minipage}{0.45\textwidth}
\hrule height 0.8pt
\vspace{2pt}
\captionsetup[algorithm]{
    justification=raggedright,
    singlelinecheck=false,
    labelfont=bf
}
\captionof{algorithm}{PDFM}
\label{alg:pd_practical}
\vspace{-0.6em}
\hrule
\small
    \begin{algorithmic}
    \State \textbf{Input:}  $b$, parameter learning rate $\eta^{(\theta)},\eta^{(\phi)}$, dual variable learning rate $\eta_i$
    \State $\lambda'(0)\gets 0$
    \Repeat
    \State \resizebox{0.9\textwidth}{!}{$\theta \gets \theta-\eta^{(\theta)}\nabla_\theta\mathbb{E}_{t,x_0,x_1}\left[ \| u_\theta(x_t, t) - (x_1-x_0)\|^2\right]$}
    \Until converged
    \Repeat
    \State $\phi \gets \phi-\eta^{(\phi)}\text{grad}_\phi$
    \Until converged
    \State $\lambda'(1)\gets[\lambda'(0) +\eta_{0}(b-P_{u_{\theta}}(\mathcal{C}))]_{+}$
    \For{$i=1,2,\cdots$}
    \Repeat
    \State $\theta \gets \theta-\eta^{(\theta)}  \text{grad}_\theta$
    \State $\phi \gets \phi-\eta^{(\phi)}\text{grad}_\phi$
    \Until converged
\State $\lambda'(i+1)\gets[\lambda'(i) +\eta_{i}(b-P_{u_{\theta}}(\mathcal{C}))]_{+}$
    \EndFor
    \State Note: $\mathrm{grad}_{\phi},\mathrm{grad}_\theta$ are given in \eqref{eq:gradphi}, \eqref{eq:gradtheta}.
    \end{algorithmic}
\vspace{2pt}
\hrule
\end{minipage}
\end{wrapfigure}

For implementation, we replace the likelihood objective $\mathbb{E}_{x\sim Q}[-\log p_{u_\theta}(x)]$ with the conditional-FM (CFM) objective. This is motivated by the fact that CFM is highly tractable \citep{lipman2023flow}, and a recent result by \cite{su2025flowmatchingkldivergence}, where it is shown that under smoothness assumptions and sufficient spatio-temporal regularity of the velocity fields, a bounded velocity error yields a bounded KL divergence between the target distribution and the learned distribution. 

With the linear interpolant and random coupling strategy, i.e., $x_0\sim P_0, x_1\sim Q, x_t = tx_1 + (1-t)x_0$, in which $x_0$ and $x_1$ are independent,
we define the surrogate Lagrangian as follows:
\begin{equation}\label{eq:Lprime}
\resizebox{0.44\textwidth}{!}{$
\begin{aligned}
    L'(u_\theta,\lambda) =&w\mathbb{E}_{t,x_0,x_1}\left[\| u_\theta(x_t, t) - (x_1-x_0)\|^2\right]\\&+ \lambda  (b-P_{u_{\theta}}(\mathcal{C})),
    \end{aligned}$}
\end{equation}
in which $t\sim \mathrm{Unif}[0,1]$ and $w>0$ is a weighting constant that calibrates the scale of the
velocity error term relative to $\mathbb{E}_{x\sim Q}[-\log p_{u_\theta}(x)]$ in $\bar L$.

To minimize $L'$, a natural approach is to update $\theta$ by gradient descent,
which requires estimable gradients for each term in
\eqref{eq:Lprime}. The gradient of $\mathbb{E}_{t,x_0,x_1}\left[\| u_\theta(x_t, t) - (x_1-x_0)\|^2\right]$ follows from the standard FM optimization process. However, 
estimating $-\nabla_\theta P_{u_{\theta}}(\mathcal{C})$ requires further development, especially in the case where the constraints are available only via a membership oracle.

In this paper, we propose a gradient estimation strategy that avoids Gaussian randomization of the velocity field proposed in \cite{FMRE}, drawing intuition from deterministic policy gradient methods \citep{DPG}.

We consider an $N$-step forward Euler discretization of the flow. Let $\Delta t=1/N\ll1$  be the selected interval for discretization and restrict time to the grid
$\{t=i\Delta t:i=0,\ldots,N\}$. Starting from $x_0\sim P_0$, the forward
Euler discretization is:
$x_{t+\Delta t}=    x_t+u_\theta(x_t,t)\Delta t$.
Equivalently, 
\begin{equation}\label{eq:euler}
    x_{i\Delta t}= \Phi^{\Delta t}_{u_\theta, i\Delta t}(x_0) = \begin{cases}x_0,&i=0,\\
   x_0+\sum_{j=0}^{i-1} u_\theta(x_{j\Delta t},j\Delta t)\Delta t, &i = 1,2,\cdots, N.
    \end{cases}
\end{equation}

Note that $ P_{u_\theta}(\mathcal{C}) = \lim_{\Delta t \to 0}\int p_0(x_0) \mathbf{1}_{\mathcal{C}}\left[\Phi^{\Delta t}_{u_\theta, 1}(x_0)\right] dx_0$. For sufficiently small $\Delta t$, we use the finite-step quantity inside the limit to approximate $P_{u_\theta}(\mathcal{C})$. 
We now note the main difficulty. Direct path-wise differentiation of $P_{u_\theta}(\mathcal{C})$ with respect to $\theta$ requires differentiating the indicator function $\mathbf{1}_{\mathcal{C}}$, for which the gradient is $0$ almost everywhere and undefined on the constraint boundary. To address this issue, assume there exists a differentiable function $f^{(u_\theta)}:\mathbb{R}^d\times[0,1]\rightarrow[0,1]$ with the following properties: 
(a) \textbf{[Bellman consistency]} For every
$i\in\{0,1,\ldots,N-1\}$ and $x\in\mathbb R^d$,
\begin{equation}
        f^{(u_\theta)}(x,i\Delta t)
    =
    f^{(u_\theta)}
    \left(
        x+u_\theta(x,i\Delta t)\Delta t,
        (i+1)\Delta t
    \right);
\end{equation}
(b) \textbf{[No explicit terminal parameter dependence]} $\nabla_\theta f^{(u_\theta)}(x,1)=0$; and (c) \textbf{[A good estimator of the indicator}] $   \mathbb{E}_{x_0\sim P_0}\left[\left| f^{(u_\theta)}(x_0,0)-\mathbf{1}_{\mathcal{C}}\left[\Phi_{u_\theta, 1}(x_0)\right] \right|\right]\leq \rho$, in which $\rho>0$ is a small approximation tolerance. 
The main idea is that the function $f^{(u_\theta)}$ gives a surrogate for $P_{u_\theta}(\mathcal{C})$ via $\mathcal{F}_{u_\theta}(\mathcal{C}): = \int p_0(x_0) f^{(u_\theta)}(x_0,0) dx_0$. Indeed, 
\begin{equation}
\resizebox{0.98\textwidth}{!}{$
 |P_{u_\theta}(\mathcal{C}) - \mathcal{F}_{u_\theta}(\mathcal{C})| 
 \leq \int p_0(x_0)\left| f^{(u_\theta)}(x_0,0)-\mathbf{1}_{\mathcal{C}}\left[\Phi_{u_\theta, 1}(x_0)\right] \right|dx_0 
 = \mathbb{E}_{x_0\sim P_0}\left[\left| f^{(u_\theta)}(x_0,0)-\mathbf{1}_{\mathcal{C}}\left[\Phi_{u_\theta, 1}(x_0)\right] \right|\right]\leq \rho,
$}
\end{equation}
leads to
\begin{equation}
\label{eq:surrogate_obj}
        L'(u_\theta,\lambda) \leq w\mathbb{E}_{t,x_0,x_1}\left[\| u_\theta(x_t, t) - (x_1-x_0)\|^2\right]+ \lambda  (b-\mathcal{F}_{u_{\theta}}(\mathcal{C}))+\lambda \rho.
\end{equation}
We therefore minimize the objective on RHS of \eqref{eq:surrogate_obj} to control $ L'(u_\theta,\lambda)$ which is a good approximation when $\rho$ is small.

\begin{proposition} \label{prop:grad}
Under the assumption that
$\nabla_\theta f^{(u_\theta)}(x_{i\Delta t},i\Delta t)$ exists and is uniformly
bounded for all $i\in\{0,\ldots,N-1\}$, $i\sim \text{Unif}\{0,\cdots,N-1\}$ and 
$x_{i\Delta t}\sim(\Phi^{\Delta t}_{u_\theta,i\Delta t})_{\#}P_0$,
\begin{equation}
\resizebox{0.92\textwidth}{!}{$
     \nabla_\theta\mathcal{F}_{u_\theta}(\mathcal{C})= N\mathbb{E}_{i,x_{i\Delta t}}\left[\nabla_\theta u_\theta(x_{i\Delta t},i\Delta t)\nabla_uf^{(u_\theta)}(x_{i\Delta t}+u\Delta t, (i+1)\Delta t)|_{u = u_\theta(x_{i\Delta t},i\Delta t)}\right],$}
\end{equation}
in which $\nabla_\theta u_\theta(x_t,t)$ is a Jacobian matrix such that each
column is the gradient $\nabla_\theta [u_\theta(x_t,t)]_j$ of the $j$-th dimension of the $d$-dimensional velocity with respect to the velocity parameters, i.e.,
\begin{equation}
    \nabla_\theta u_\theta(x_t,t) = [\nabla_\theta [u_\theta(x_t,t)]_1\quad \nabla_\theta [u_\theta(x_t,t)]_2\quad \cdots \quad\nabla_\theta [u_\theta(x_t,t)]_d].
\end{equation}
\end{proposition}
Proposition~\ref{prop:grad} with its proofs in Appendix~\ref{sec:proof4propgrad}, states that one can approximate $\nabla_\theta\mathcal{F}_{u_\theta}(\mathcal{C})$ via Monte Carlo. But we still need $f^{(u_\theta)}$ that satisfies properties (a), (b), and (c) above.
For this we propose to train $f_\phi$, an NN parameterized by $\phi$ to approximate $f^{(u_\theta)}$.
It is trained via supervised learning on sampled trajectories. Specifically, we
sample $x_0\sim P_0$, generate an Euler trajectory
$\{x_{i\Delta t}\}_{i\in\{0, 1,\cdots,N\}}$ under the current velocity field
$u_\theta$, and assign the terminal label
$   y=\mathbf{1}_{\mathcal C}(x_{N\Delta t})$.
Then $f_\phi$ is trained to predict $y$ from intermediate states by minimizing
the binary cross-entropy loss, thereby the gradient of $\phi$ can be given as
\begin{equation}\label{eq:gradphi}
\mathrm{grad}_{\phi} = \nabla_\phi
\mathbb{E}_{
i, x_{i\Delta t}}
\left[
    -y\log f_\phi(x_{i\Delta t},i\Delta t)
    -(1-y)\log\big(1-f_\phi(x_{i\Delta t},i\Delta t)\big)
\right].
\end{equation}
At each sampled trajectory, this structure encourages $\{f_\phi (x_{i\Delta t}, i\Delta t)\}_{i\in\{0, 1,\cdots,N\}}$ to target at the same terminal label $y=\mathbf{1}_{\mathcal{C}}(x_{N\Delta t})$, which is held fixed during the $\phi$-update, thereby encouraging to fulfill property (a),(b),(c).  $f_\phi$ serves as a differentiable predictor of terminal constraint
satisfaction from $(x_t,t)$.

Practically, we aim to minimize $\frac{1}{w} L'(u_\theta,\lambda) =\mathbb{E}_{t,x_0,x_1}\left[\| u_\theta(x_t, t) - (x_1-x_0)\|^2\right]+ \lambda'  (b-P_{u_{\theta}}(\mathcal{C}))$ with $\lambda'=\lambda/w$, which has the same minimizer as $L'(u_\theta,\lambda)$. Therefore, by substituting the learned predictor $f_\phi$ for $f^{(u_\theta)}$ in Proposition~\ref{prop:grad}, the gradient used in practice to update $\theta$ is given by
\begin{equation}\label{eq:gradtheta}
\resizebox{0.92\textwidth}{!}{$
\begin{aligned}
    \mathrm{grad}_{\theta} =& \nabla_\theta\mathbb{E}_{t,x_0,x_1}\left[ \| u_\theta(x_t, t) - (x_1-x_0)\|^2\right]\\&-
    \lambda'N\mathbb{E}_{i,x_{i\Delta t}}\left[\nabla_\theta u_\theta(x_{i\Delta t},i\Delta t)\nabla_uf_{\phi}(x_{i\Delta t}+u\Delta t, (i+1)\Delta t)|_{u = u_\theta(x_{i\Delta t},i\Delta t)}\right].
    \end{aligned}$}
\end{equation}

Therefore, we can approximately solve
$    \arg\min_{\theta} \frac{1}{w}L'(u_\theta,\lambda)$ by alternating between updating $\theta$ and $\phi$. With $\theta$ fixed, we update $f_\phi$ with $\mathrm{grad}_\phi$ in \eqref{eq:gradphi} to predict terminal constraint satisfaction following the current velocity field. With $\phi$ fixed, we update $\theta$ with $\mathrm{grad}_\theta$ in \eqref{eq:gradtheta}. 

With $\theta$ updated, we update the dual variable, for which the gradient is proportional to $b-P_{u_\theta}(\mathcal{C})$. In practice, $P_{u_\theta}(\mathcal{C})$ is estimated by Monte Carlo simulation, i.e., by generating terminal samples using $u_\theta$ and computing the probability of generating samples belonging to $\mathcal{C}$. The algorithm for practical PDFM is provided in Alg.~\ref{alg:pd_practical}.

\subsection{Constraint Guarantees on PDFM}\label{sec:PDFMguarantee}

Let the velocity field solved by minimizing $\frac{1}{w}L'$ at $\lambda' = \lambda/w$, which is analyzed in the previous section be denoted by $u'^{(\lambda)}_\theta$, the dual variable of $\bar{L}$ at the $i$-th iteration be denoted by $\lambda(i)$ and the distribution induced by the velocity field $u'^{(\lambda)}_\theta$ be denoted by $P_{u'_\theta}^{(\lambda)}:=(\Phi_{u'^{(\lambda)_\theta},1})_{\#}P_0$. In this subsection, we analyze the asymptotic gap between the distribution
$P_{u'_\theta}^{(\lambda^{(i)})}$ produced by the practical primal-dual
algorithm and the optimal distribution $P^{(\lambda^*)}$ in \eqref{eq:cvxoriginal} as $i\to\infty$.

\begin{assumption}\label{ass:lambdamax}
Note the scaling relationship $\lambda = w\lambda'$. For every iteration of Alg.~\ref{alg:pd_practical}, we assume that there exists a constant $\lambda_{\max}<\infty$ such that
    \begin{equation}
        \forall i, \lambda'(i)\leq \lambda_{\max}/w,\quad\mathrm{equivalently}, \quad \lambda(i)\leq \lambda_{\max}.
    \end{equation}
\end{assumption}

We next make the following assumption to quantify the optimization error introduced by the approximation of $P^{(\lambda)}_{u_\theta}$ via the strategy in Sec.~\ref{sec:gradimp}.
\begin{assumption}\label{ass:esterrorGD}
 For every fixed $0\leq\lambda\leq \lambda_{\max}$, $ P_{u'_{\theta}}^{(\lambda)}$ is $\varepsilon$-suboptimal with respect to the FM-parameterized optimum, 
    \begin{equation}
L\left(p_{u'_{\theta}}^{(\lambda)}, \lambda\right) - \bar{g}(\lambda)\leq \varepsilon,\quad\mathrm{equivalently}, \quad L\left(p_{u'_{\theta}}^{(\lambda)}, \lambda\right) - \inf_{u_\in U_\Theta} \bar{L}(u,\lambda)\leq \varepsilon.
    \end{equation}
\end{assumption}

For the dual update, we need to evaluate 
$b-P_{u'_\theta}(\mathcal C)$, where $P_{u'_\theta}(\mathcal C)$ is estimated in
practice by Monte Carlo simulation. To simplify the analysis, we make the
following assumption.
\begin{assumption}\label{ass:exactdualgrad}
The empirical estimator of $b-P_{u'_\theta}(\mathcal C)$ is exact at every dual
update.
\end{assumption}
If this assumption is relaxed to a weaker condition, such as unbiased Monte Carlo
estimation, then the convergence guarantee should be stated in expectation. In
particular, one can generally bound
$\liminf_{i\to\infty}\mathbb{E}\!\left[
P_{u'_{\theta}}^{(\lambda(i))}(\mathcal C)
\right]$
rather than the pathwise quantity
$\liminf_{i\to\infty}
P_{u'_{\theta}}^{(\lambda(i))}(\mathcal C)$. 
The corresponding statements and proofs are provided in Appendix~\ref{app:unbiased_constraint_estimation}.

\begin{theorem} \label{theorem:bound2}Under Assumptions~\ref{ass:feasibility}, \ref{ass:esterror}-\ref{ass:densityrange}, and \ref{ass:lambdamax}-\ref{ass:exactdualgrad}, with $\forall i, 0\leq \eta_i\leq\frac{2}{kw}$, $\sum_{i=1}^\infty\eta_i=\infty, \sum_{i=1}^\infty\eta_i^2<\infty $,
\begin{equation}
    \limsup_{i\to\infty}  \left\|{P_{u'_{\theta}}^{(\lambda(i))}}-P^{(\lambda^*)} \right\|_{\mathrm{TV}} \leq \mathcal{E}_{\mathrm{PDFM}}, \quad\liminf_{i\to\infty}P_{u'_{\theta}}^{(\lambda(i))}(\mathcal{C})\geq b-\mathcal{E}_{\mathrm{PDFM}},
    \end{equation}
    \begin{equation}
    \mathcal{E}_{\mathrm{PDFM}} =  e_{\max}+\frac{2e_{\max}}{k\left((1-Q({\mathcal{C}}))^{1/3}+Q({\mathcal{C}})^{1/3}\right)^3} ,\quad e(\lambda):=\sqrt{\frac{\varepsilon+\mu(\lambda) (A_1\nu + A_2\nu^2)}{2(\mu(\lambda) - \lambda)}}, 
\end{equation}
     in which $\mu(\lambda) = \frac{1}{2}\left(\lambda+1+\sqrt{(\lambda+1)^2-4\lambda(1-Q(\mathcal{C}))}\right)$ and $e_{\max} =\max_{\lambda\in[0,\lambda_{\max}]} e(\lambda)$.
\end{theorem}

Theorem~\ref{theorem:bound2} gives the asymptotic guarantee for the practical primal-dual algorithm in Alg.~\ref{alg:pd_practical}, with detailed proofs in Appendix~\ref{sec:proofpart2}. Under the stated boundedness, approximation, and exact dual-gradient assumptions, $P_{u'_{\theta}}^{(\lambda(i))}$  remains within an explicit TV neighborhood of the distribution-level optimum $P^{(\lambda^*)}$. The radius of this neighborhood is controlled by $e_{\max}$, which is determined by the optimization error $\varepsilon$, the FM parameterization error $A_1 \nu + A_2 \nu^2$, and their worst-case behavior over the bounded dual range $[0,\lambda_{\max}]$. Consequently, the asymptotic constraint satisfaction rate is at least the target level $b$ minus the same TV error; the gap decreases as the practical optimization error $\varepsilon$ and FM parameterization error $\nu$ decrease.

\section{Experiments} 
\subsection{Illustrative $2$-D Constrained Generation ($Q(\mathcal{C})< b$)}\label{sec:2dtoy}

We begin with a synthetic $2$-D constrained generation example. Let
$z=(z_1,z_2)\in\mathbb R^2$ denote a sample. The clean distribution is a uniform distribution over two disconnected
regions,
\[
\mathcal C=\{(z_1,z_2):1\leq |z_1|\leq 5,\;1\leq |z_2|\leq 5,\;z_1z_2>0\}.
\]
Samples from the clean distribution are corrupted by
additive zero-mean Gaussian noise with variance $0.01$, yielding the training
distribution $Q$, for which $Q(\mathcal C)\approx 96\%$. Following the
formulation in \eqref{eq:cvxoriginal}, the goal is to learn a distribution that
remains as close as possible to $Q$ while achieving a constraint satisfaction
rate of at least $b=99.5\%$.


We compare PDFM
against two FM-based baselines: (1) vanilla FM trained on
samples from $Q$; and (2) FM trained on the truncated distribution
$Q'=Q(\cdot\mid\mathcal C)$, obtained by discarding infeasible training samples,
with density $
    q'(x)={q(x)\mathbf 1_{\mathcal C}(x)}/{Q(\mathcal C)}$. The visualization of the clean distribution, $Q$, $P^{(\lambda^*)}$ (computed for this simple case via Lemma~\ref{lemma:pstar}), the distributions induced by baselines (1), (2), and PDFM are shown in
Appendix~\ref{app:toy2d} Fig.~\ref{fig:2Dtoy}. We evaluate two metrics: the constraint satisfaction rate and the Sliced Wasserstein Distance (SWD)~\citep{swd} between samples from the trained model and samples from $P^{(\lambda^*)}$. SWD measures distributional
similarity, with smaller values indicating better similarity. Reported metrics use $50000$ generated samples, averaged over $100$ trials.

Baseline (1) shows that vanilla FM trained directly on $Q$ produces a distribution
with constraint satisfaction rate $96.37\%$, close to $Q(\mathcal C)$, but below
the target level $b=99.5\%$; its SWD is $0.1384$. Baseline (2) shows that even
when FM is trained on the truncated distribution, whose training samples are all
feasible, it still fails to fully enforce the constraint, achieving constraint
satisfaction rate $98.74\%$ and SWD $0.1156$. In this synthetic example,
discarding infeasible samples removes only about $4\%$ of the data; \emph{in more
realistic settings, as shown in Sec.~\ref{sec:PDE}, this rejection step can
discard a much larger fraction of the training set and thereby degrade distributional matching}. In contrast, PDFM achieves the target constraint level,
with a constraint satisfaction rate $99.51\%$, while also attaining the lowest SWD,
$0.1096$. We therefore conclude that, with samples from $Q$, PDFM achieves better
constraint satisfaction than FM baselines trained on either $Q$ or the truncated
distribution $Q'$, while attaining superior distributional fidelity.

\subsection{Constrained Generation with $\ell_2$-Ball Constraints ($Q(\mathcal{C})\geq b$)} \label{sec:l2ball}

We next evaluate PDFM on higher-dimensional
synthetic tasks in this regime, together with several baselines, MDM \citep{mirror}, NAMM \citep{feng2025neural}, PDM \citep{Christopher2024ConstrainedSW}, PCFM \citep{pcfm2025}, and FM-RE \citep{FMRE}, introduced in Appendix~\ref{sec:relatedworks}.

The target distribution $Q$ is a Gaussian mixture model, with a constraint defined as an $\ell_2$ ball constraint $
    \mathcal{C} = \left\{x\in \mathbb{R}^d\Big|\|x\|_2\leq 1\right\}$.
  Specifically, we consider $d$
isotropic Gaussians, each with variance $0.05$, centered at each of the $d$ standard unit vectors, and reject samples
outside $\mathcal{C}$, i.e., $Q(\mathcal{C}) = 1$. This setting is also considered in \citet{mirror}.

FM achieves $P(\mathcal{C})=97.633\%$ and $90.918\%$ in the 8-D and 20-D cases, respectively, showing that constrained generation remains nontrivial even when $Q(\mathcal{C}) \geq b$, in which case the population-level optimum is $Q$. For a direct comparison with FM-RE \citep{FMRE}, which also relies on membership oracle feedback, we set the target constraint satisfaction rate $b$ to match the $P(\mathcal{C})$ attained by FM-RE, namely $b=99.923\%$ for the $8$-D case and $b=99.748\%$ for the $20$-D case. The results in Table~\ref{tb:820l2ball} show that PDFM achieves $P(\mathcal{C})$ close to the target $b$ while yielding a smaller SWD than FM-RE, demonstrating superior performance in both constraint satisfaction and distributional matching. It is worth noting that PDFM can be configured to match the constraint satisfaction level achieved by FM-RE, whereas the reverse is not straightforward. The reason is that PDFM directly controls $P(\mathcal{C})$ through $b$, while FM-RE can only influence $P(\mathcal{C})$ indirectly by adjusting the weight of the constraint term. Therefore, compared with FM-RE, PDFM substantially reduces the burden of hyperparameter tuning. 

Compared with MDM, which requires convexity and a special structure of
$\mathcal C$, NAMM and PCFM, which both require a differentiable distance to $\mathcal C$,
and PDM, which requires projection onto $\mathcal C$, \emph{PDFM works with arguably the most flexible
requirement on the constraint set specification: access via a membership oracle.}
Despite this weaker requirement, PDFM achieves strong constraint satisfaction
while maintaining competitive distributional fidelity. 

\begin{table}[t]
\centering
\caption{Results on $8$-D and $20$-D constrained generation tasks. $P(\mathcal{C})$ denotes the mean constrained-satisfaction rate of $10000$ samples, averaged over $100$ generation trials. SWD is computed with $10000$ generated samples and samples from $Q$, averaged over $100$ generation trials.  }
\scalebox{0.9}{
\begin{tabular}{cc|ccccccc}
\hline
 &  & FM & MDM & NAMM & PDM & PCFM & FM-RE & PDFM \\
\hline
\multirow{2}{*}{$8$-D $\ell_2$ ball}
& $P(\mathcal C)$ $(\uparrow)$
& $97.633\%$ & $100\%$ & $99.986\%$ & $100\%$ & $100\%$ & $99.923\%$ & $99.906\%$ \\
& SWD $(\downarrow)$
& $0.0193$ & $0.0217$ & $0.0228$ & $0.0200$ & $0.0284$ & $0.0194$ & $0.0172$ \\
\hline
\multirow{2}{*}{$20$-D $\ell_2$ ball}
& $P(\mathcal C)$ $(\uparrow)$
& $90.918\%$ & $100\%$ & $99.711\%$ & $100\%$ & $100\%$ & $99.748\%$ & $99.755\%$ \\
& SWD $(\downarrow)$
& $0.0087$ & $0.0144$ & $0.0205$ & $0.0257$ & $0.0148$ & $0.0132$ & $0.0101$ \\
\hline
\end{tabular}}
\label{tb:820l2ball}
\end{table}

\subsection{Unconditional PDE-Constrained Generation of Spatiotemporal Dynamics} \label{sec:PDE}

A central challenge in applying generative modeling to applications involving dynamical systems, including chaotic systems, is generating samples that respect known physical laws such as the underlying governing equations. We illustrate PDFM on the Kuramoto--Sivashinsky (KS) equation, a canonical model of spatiotemporal chaos \citep{10.1143/PTP.55.356,SIVASHINSKY19771177}, where the constraint set $\mathcal C$ is defined by the residual of generated space-time patches against the numerical PDE solver. This setting is representative of generative emulation tasks in which a forward solver can evaluate whether a sample obeys the dynamics. 

Given a dataset of non-overlapping space--time patches $x\in\mathbb R^{64\times 64}$, the objective is to learn a generative model
that produces patches satisfying the KS dynamics up to a prescribed tolerance
with a user-specified probability. Specifically, a forward solver, which is assumed to be \textbf{non-differentiable},  provides the residual
feedback $\delta_{L_n}$, where $\delta_{L_n}=0$ indicates strict alignment with the KS
dynamics up to solver precision. A sample is considered feasible if
$\delta_{L_n}<\tau$, where $\tau$ denotes the prescribed error tolerance. Details of the dataset construction and residual computation are provided in Appendix~\ref{app:PDE}. Two generation cases are considered.

\paragraph{Case 1:} The clean patches $x \in \mathbb{R}^{64 \times 64}$ with $\delta_{L_n}=0$ from the solver are used as the training data. The threshold is $\tau = 0.033$, and the target constraint satisfaction rate is $98\%$. In this case, all samples from the training distribution satisfy the constraints, indicating $Q(\mathcal{C})=1$.
\paragraph{Case 2:} We create degraded samples from the space--time patches, simulating measurements from a sensor with reduced spatial resolution, where only part of the spatial grid is directly observed and the missing locations are reconstructed by simple linear interpolation. Specifically, for each sample, along the spatial dimension, every $9$-th position is replaced by a linearly interpolated value between the adjacent values. This leads to an average $\delta_{L_n}=0.0555$ for the training samples. Setting the threshold as $\tau = 0.05$, the training samples constraint satisfaction rate is $39.47\%$, i.e., $Q(\mathcal{C})\approx39.47\%$. The target constraint satisfaction rate is $90\%$.

\begin{table}[t!]
\centering
\caption{Comparison across the considered methods for Case 1 and Case 2. Results are means over $10$ trials, using $2000$ generated samples per method per trial.}\label{tab:PDE}
\resizebox{\textwidth}{!}{
\begin{tabular}{c|ccc|cccc}
\hline
\multirow{2}{*}{Method}
& \multicolumn{3}{c|}{Case 1 $(Q(\mathcal C)=1)$}
& \multicolumn{4}{c}{Case 2 $(Q(\mathcal C)\approx 39.47\%)$} \\
\cline{2-8}
& $P(\mathcal C)$ $(\uparrow)$
& Avg $\delta_{L_n}$ $(\downarrow)$
& HE $(\downarrow)$
& $P(\mathcal C)$ $(\uparrow)$
& Avg $\delta_{L_n}$ $(\downarrow)$
& $\mathrm{HE}_{\mathrm{clean}}$ $(\downarrow)$
& $\mathrm{HE}_{\mathrm{deg}}$ $(\downarrow)$ \\
\hline
FM ($Q$)
& {14.19\%}
& {0.0409}
& {(0.0087,0.0453)}
& {9.82\%}
& {0.0622}
& {(0.0094,0.0465)}
& {(0.0086,0.0565)} \\
FM ($Q'$)
& N/A
& N/A
& N/A
& {33.24\%}
& {0.0539}
& {(0.0135,0.0477)}
& {(0.0114,0.0489)}\\
PCFM
& {6.03\%}
& {0.0484}
& {(0.0156,0.1547)}
& {0.06\%}
& {0.1022}
& {(0.0173,0.0729)}
& {(0.0186,0.0823)} \\
FM-RE
& {43.17\%}
& {0.0339}
& {(0.0117,0.0329)}
& {46.33\%}
& {0.0513}
& {(0.0118,0.0558)}
& {(0.0097,0.0570)} \\
PDFM
& {97.74\%}
& {0.0235}
& {(0.0107,0.0269)}
& {90.48\%}
& {0.0406}
& {(0.0118,0.0435)}
& {(0.0103,0.0457)} \\
\hline
\end{tabular}}
\end{table}

The statistical results are reported in Table~\ref{tab:PDE}, and visualizations of the generated samples and corresponding histograms are provided in Appendix~\ref{app:PDE}. In Table~\ref{tab:PDE}, $P(\mathcal{C})$ denotes the empirical constraint satisfaction rate. To evaluate the distribution matching quality, we follow \citet{jiang2023training}, and provide the mean histogram error (denoted as $\mathrm{HE}$) -- the mean absolute difference between the generated and reference summary-statistic histograms over the bin grids, averaged over $2000$ samples.  $\mathrm{HE}$ in Table~\ref{tab:PDE} is in the format $(\mathrm{HE}^{(1)},\mathrm{HE}^{(2)})$, where $\mathrm{HE}^{(1)}$ corresponds to the histogram  error of
$\partial u/\partial t$ vs. $\partial u/\partial x$ and $\mathrm{HE}^{(2)}$
corresponds to the histogram error of
$\partial^2 u/\partial x^2$ vs. $\partial u/\partial x$. A lower histogram error denotes a better match in distribution.
In Case 2, ``$\mathrm{HE}_{\mathrm{clean}}$'' and ``$\mathrm{HE}_{\mathrm{deg}}$'' denote the error relative to the clean samples and the degraded training samples, respectively.

In Case~1, we compare PDFM with FM trained on $Q$, denoted by FM~$(Q)$, PCFM,
and FM-RE.  In Case~2, a natural alternative is to discard infeasible training samples and
train FM on the truncated distribution $Q'$, with density $q'(x)=\frac{q(x)\mathbf 1_{\mathcal C}(x)}{Q(\mathcal C)}.$
We additionally include this baseline, denoted by FM~$(Q')$. This rejection step removes
approximately $60\%$ of the training data. Although FM~$(Q')$ improves the
constraint satisfaction rate from $9.82\%$ for FM~$(Q)$ to $33.24\%$, it
remains far below the target rate of $90\%$. Moreover, fewer training samples lead to a degraded generated distribution, as reflected by the increased histogram error.

The results in Table~\ref{tab:PDE} illustrate four properties of PDFM that are particularly relevant in this setting. First, even when the training distribution itself satisfies the constraint perfectly (Case~1, $Q(\mathcal C)=1$), vanilla FM only achieves $P(\mathcal C)=14.19\%$, confirming that local consistency does not transfer for free from the training distribution to the trained model; PDFM raises this rate to $97.74\%$ at the prescribed target while keeping the marginal summary statistics close to the training data. Second, in the more realistic Case~2 with degraded training data ($Q(\mathcal C)\approx 39.47\%$), PDFM closes the gap from $9.82\%/33.24\%$ to $90.48\%$, demonstrating that the chance-constrained formulation lets the generative model produce samples of higher dynamical quality than the data it was trained on, by leveraging the solver-based residual oracle to suppress non-physical regions of the learned distribution. Third, while PCFM and NAMM, requiring the fully differentiable residual, fail in this task,  PDFM reaches these satisfaction rates without differentiating through the numerical stepper, which is essential when the stepper is non-differentiable, written in a legacy language, or otherwise available only as a black box. Fourth, compared with FM-RE, which also requires only membership oracle feedback, PDFM provides direct control over the constraint satisfaction rate and can therefore target a prescribed feasibility level rather than relying on an implicitly tuned penalty weight.

\subsection{Additional Applications}

\textbf{Fingerprint image generation.}
We further evaluate PDFM on SOCOFing fingerprint \citep{shehu2018detection} generation with a ridge-connectivity constraint, assessed by a non-differentiable image-processing oracle. The oracle applies thresholding and skeletonization, then checks whether the number of connected ridge components is below a prescribed threshold. PDFM achieves a mean constraint satisfaction rate of $99.28\%$, meeting the target $b=99\%$, compared with $72.74\%$ for FM and $88.14\%$ for FM-RE. Meanwhile, PDFM maintains NLLs comparable to FM and preserves visual fidelity. Experimental details and sample visualizations are provided in Appendix~\ref{app:fing}.

\textbf{$3$-D molecule generation.}
Trained on QM9 \citep{Ramakrishnan2014}, we aim to generate samples satisfying a molecule-stability constraint requiring all atoms to satisfy valid valency rules after bond inference. Using only binary feasibility feedback for constraint satisfaction, PDFM improves the mean molecule stability rate from $32.77\%$ to $48.75\%$ and reduces NLL from $-19.74$ to $-20.70$ relative to FM. Although the stability rate remains below the target $b=60\%$, these results demonstrate both improved feasibility and improved likelihood quality, without task-specific designs. Further details are provided in Appendix~\ref{App:qm9}.

\section{Conclusion}

In this paper, we present PDFM, a primal-dual framework for constrained generation that requires only membership oracle feedback, provides direct control over the constraint satisfaction rate, and admits theoretical guarantees. We also propose a non-stochastic constraint-guided approach for updating the velocity field, enabling efficient use of oracle feedback during primal optimization. 
One of the limitations of PDFM arises from the need to obtain sufficient positive constraint satisfaction feedback during training to learn $f_\phi$ reliably. Further details are discussed in Appendix~\ref{app:discussion}.

\section*{Acknowledgment}
Zhengyan Huan is supported via NSF under Cooperative Agreement PHY-2019786, PHY-2525568, and DOE DE-SC002396. Peter Y. Lu and Shuchin Aeron are supported via NSF under Cooperative Agreement PHY-2019786, PHY-2525568. Shuchin Aeron would also like to acknowledge funding via NSF DMS 2309519.

\newpage
\bibliographystyle{plainnat}
\bibliography{references}

\newpage
\appendix

\section{Related work}\label{sec:relatedworks}
Constrained generative modeling has been studied from several complementary perspectives, depending on whether the constraint is imposed at the distribution level or at the level of individual generated samples. We refer to these two settings as \textbf{distributional constraints} and \textbf{sample-wise constraints}, respectively. This paper investigates \textbf{sample-wise constraints} with the objective of generating samples within the constraint set $\mathcal{C}$.

A closely related work of constrained generative models in the domain of \textbf{distributional constraints} is \citet{Khalafi}, which requires the generated distribution to remain close in KL divergence, to one or more reference distributions. While \citet{Khalafi} and this work both utilize primal-dual frameworks for constrained generative modeling, their scopes and theoretical guarantees differ substantially due to different constraint settings. Algorithmically, this work's optimization strategy designed for sample-wise constraints in Sec.~\ref{sec:gradimp} deviates from the score/flow loss in classical generative models.

For sample-wise constraints, existing approaches can be broadly divided into \textbf{training-time} and \textbf{inference-time methods}. Training-time methods incorporate constraint information during model training, for example, by modifying the objective, architecture, or learned dynamics. This typically requires additional training effort, but they can generally use standard/slightly-modified sampling procedures at inference time. In contrast, inference-time methods can often be applied to regularly trained generative models, but they must steer, project, or correct the sampling trajectory during generation, which usually increases inference-time computational cost. This paper is a \textbf{training-time method}, that informs the FM model of the constraints during training and uses standard Euler sampling with the learned velocity field.

\subsection{Training-Time Methods}

Among training-time methods assuming $Q(\mathcal{C})=1$, approaches that enforce strict feasibility of the generated samples, i.e., $P(\mathcal{C})=1$, generally require stronger structures and knowledge of the constraint set $\mathcal{C}$. Reflection-based methods, such as the reflected diffusion model (RDM) \citep{loureflect, fishmanreflect} and reflected flow matching (RFM) \citep{reflectedfm} enforce constraint satisfaction via reflected generation trajectories within $\mathcal{C}$, requiring convexity of $\mathcal{C}$ and explicit knowledge of the constraint boundary $\partial\mathcal{C}$. Mirror-based approaches, such as mirror diffusion model (MDM) \citep{mirror}, enforce constraint satisfaction by performing DDPM training in the domain of an unconstrained mapping of $\mathcal{C}$, and obtaining generated samples from the inverse mapping of the unconstrained domain. MDM requires convexity of $\mathcal{C}$ and an available closed-form bijective mapping between $\mathcal{C}$ and the unconstrained space.

Convexity of $\mathcal C$, boundary information, and tractable bijective maps
are often unavailable for complex constraints in practical settings. To broaden
applicability, several existing methods relax strict feasibility guarantees and
instead aim to improve, but not necessarily ensure, constraint satisfaction. Neural approximate mirror map (NAMM) \citep{feng2025neural} aims to approximate the bijective projection in MDM via NNs, relaxing the requirement of convexity and a closed-form bijective map to a differentiable distance to $\mathcal{C}$.
Trajectory Alignment \citep{Giorgio} trains diffusion models to align their sampling paths with the intermediate iterates of a physics-based optimizer, improving constrained design generation but requiring access to optimizer-generated trajectories during training.
Flow Matching with Random Exploration (FM-RE)~\citep{FMRE} improves feasibility by adding an explicit constraint-aware term to the FM training objective and injecting Gaussian noise into the generation trajectories to enable likelihood-ratio gradient estimation, thereby requiring only membership oracle feedback for $\mathcal{C}$. However, FM-RE's $P(\mathcal{C})$ is only implicitly controlled by the weight of the constraint-aware term, and FM-RE does not provide a theoretical guarantee for feasibility. PDFM improves upon FM-RE while retaining the same weak requirement of $\mathcal{C}$, i.e.,
membership oracle access, and further provides theoretical
guarantees and direct control over the constraint satisfaction rate
$P(\mathcal C)$.

\subsection{Inference-Time Methods}
Inference-time methods improve or enforce constraint satisfaction by modifying the sampling procedure of a pretrained generative model, typically through guidance, projection, or trajectory correction during generation. Projected diffusion model (PDM) \citep{Christopher2024ConstrainedSW} enforces strict feasibility by repeatedly projecting stochastic diffusion updates onto $\mathcal{C}$ at each step, requiring an available projection function, and potentially causing deviations from the training distribution. Physics-constrained flow matching (PCFM) \citep{pcfm2025} enforces constraints through forward-shooting and Gauss--Newton correction of predicted terminal states, which requires a differentiable constraint residual and access to its Jacobian. HardFlow \citep{li2026hardflowhardconstrainedsamplingflowmatching} enforces hard constraints by reformulating FM sampling as a trajectory-optimization problem, requiring an explicit constraint function and a tractable constrained-optimization solver during sampling. Doob-Oriented Inference-time Transformation (DOIT)~\citep{doobh} improves constraint satisfaction 
by applying Doob's $h$-transform to add a dynamic correction to the diffusion sampling process, which can handle non-differentiable reward feedback but requires Monte Carlo reward-query estimates during sampling. While inference-time methods avoid retraining, they typically require a nontrivial design of the guidance, correction, or optimization procedure and impose additional computational burden during sampling.

\section{Notations}\label{app:notations}

\begin{table}[htb]
            \caption{Definitions and descriptions of the notations used in this paper. }
    \resizebox{\textwidth}{!}{
    \begin{tabular}{|c|c|c|}
        \hline
        Symbol &  Definition & Description\\
        \hline
        $P,Q$&& Measures\\\hline
        $p,q$&&Probability densities\\\hline$P(\mathcal{C}),Q(\mathcal{C})$&$P(\mathcal{C}) = \int_{\mathcal{C}}p(x) dx, Q(\mathcal{C}) = \int_{\mathcal{C}}q(x) dx$& constraint satisfaction rates of $P$ and $Q$
                \\\hline
        $L(p,\lambda)$ & $L(p,\lambda) = \mathbb{E}_{x\sim Q}[-\log p(x)]+\lambda \left(b-\int_\mathcal{C} p(x) dx\right)$ & The Lagrangian
        \\
        \hline
        $u$ &  &The velocity field
\\\hline 
        $g(\lambda)  $ & $g(\lambda) = \inf_{p}L(p,\lambda)$ & The dual function \\\hline
        $\lambda^*$ & $\lambda^* \in \arg\max_{\lambda\geq 0}g(\lambda)$  & The optimal dual variable \\
        \hline
        $p^{(\lambda)}$ & $p^{(\lambda)}  \in \arg\min_{p}L(p,\lambda)$ & Distribution-level optimal solution at ${\lambda}$
        \\\hline $u^{(\lambda)}$&\makecell{Let $\pi^{(\lambda)}$ be the coupling of $P_0$ and $P^{(\lambda)}$, i.e., \\$(x_0,x_1)\sim \pi^{(\lambda)}, x_t=(1-t)x_0+t x_1.$\\$u^{(\lambda)}(x,t):= \mathbb{E}_{(x_0,x_1)\sim\pi^{(\lambda)}}\left[x_1-x_0 \mid x_t=x\right]$}& The velocity field of  $P^{(\lambda)}$\\\hline $u_\theta$& & The velocity field parameterized by $\theta$\\\hline $\Phi_{u, t}$&$\Phi_{u, t}(x) = x+\int_{0}^tu(\Phi_{u, s}(x),s)ds$&Pushforward operator
                \\\hline $P_{u}$& $ P_{u} := (\Phi_{u, 1})_{\#}P_0$&Distribution induced by $u$\\\hline \rule{0pt}{2.5ex}
        $\Theta, U_{\Theta}$& $U_{\Theta}:=\{u_{\theta}:\theta \in\Theta\}$& \makecell{The parameter family \\and the image of the parameterization.} 
        \\\hline \rule{0pt}{2.5ex}
        $\bar{L}(u_\theta,\lambda)$& $\bar{L}(u_\theta,\lambda) = L(p_{u_\theta},\lambda)$& The Lagrangian after parameterization. 
        \\\hline \rule{0pt}{2.5ex}
        $\bar{g}(\lambda)$& $\bar{g}(\lambda) = \inf_{u_\in U_\Theta} \bar{L}(u,\lambda)$ & The dual function after parameterization. 
        \\\hline \rule{0pt}{2.5ex}
        $\bar{\lambda}^*$ & $\bar{\lambda}^* \in \arg\max_{\lambda\geq 0}\bar{g}(\lambda)$  &
        \\\hline
        \rule{0pt}{2.5ex}$u_\theta^{(\lambda)}$&$u_\theta^{(\lambda)}\in\arg\min_{u\in U_\Theta}\bar{L}(u,\lambda)$&
        \\\hline\rule{0pt}{2.5ex} {$P_{u_\theta}^{(\lambda)}$} &  $P_{u_\theta}^{(\lambda)} = (\Phi_{u_{\theta}^{(\lambda)}, 1})_{\#}P_0$&\makecell{{ Distribution induced by $u_{\theta}^{(\lambda)}$}}\\\hline$\Phi^{\Delta t}_{u_\theta, t}$& Defined in \eqref{eq:euler} &Euler sampler
        \\\hline $L'(u_\theta,\lambda)$& \makecell{
            $ L'(u_\theta,\lambda) =\lambda  (b-P_{u_{\theta}}(\mathcal{C}))$\\$+w\mathbb{E}_{t\sim \mathrm{Unif}[0,1],x_0,x_1}\left[\| u_\theta(x_t, t) - (x_1-x_0)\|^2\right]$}
         & \makecell{Surrogate Lagrangian\\ for practical optimization}
         \\\hline
    $ u'^{(\lambda)}_{\theta}$ && \makecell{The learned velocity field by \\the iterative method in Sec.~\ref{sec:gradimp}}
         \\\hline
         $ P^{(\lambda)}_{u'_{\theta}}$&$P_{u'_\theta}^{(\lambda)} = (\Phi_{{u'}_{\theta}^{(\lambda)}, 1})_{\#}P_0$&The distribution induced by ${u'}^{(\lambda)}_{\theta}$ \\
         \hline
         \end{tabular}}
    \label{tab:notations}
        \end{table}

\section{Proofs for Lemma~\ref{lemma:pstar} and Lemma~\ref{lemma:TVandconstraint}}\label{sec:proofpart0}

\paragraph{Statement of Lemma~\ref{lemma:pstar}.}\emph{
Under Assumption~\ref{ass:feasibility},}
    \begin{equation}
        p^{(\lambda^*)}(x) = \frac{bq(x)}{Q(\mathcal{C})}\bm{1}_{\mathcal{C}}(x)+\frac{(1-b)q(x)}{1-Q(\mathcal{C})}\bm{1}_{\mathcal{C}^c}(x),~~\lambda^* = \frac{1-Q(\mathcal{C})}{1-b}-\frac{Q(\mathcal{C})}{b}.
    \end{equation}

\begin{proof}
Under Assumption~\ref{ass:feasibility}, the primal problem is convex and
satisfies Slater's condition. Therefore, strong duality holds and the duality
gap is zero.

The optimization is over probability densities $p\geq 0$ satisfying $\int p(x)dx=1$. Every optimizer with finite objective satisfies $p>0$, $Q$-almost everywhere. We introduce the Lagrangian:
\begin{equation}
    \mathcal{J}(p,\mu) = \int -q(x)\log p(x)d{x}+\lambda \left(b-\int_\mathcal{C} p(x) dx\right)+\mu\left(\int p(x)dx-1\right),
\end{equation}
in which $\lambda$ is a fixed multiplier for the inequality constraint and $\mu$ is the multiplier for the equality constraint. The KKT condition shows
\begin{equation}
    \frac{\partial\mathcal{J}}{\partial p(x)} = -\frac{q(x)}{p(x)}-\lambda \bm{1}_{\mathcal{C}}(x)+\mu =0.
\end{equation}
The solution of $p(x)$ is
\begin{equation}
    p(x) = \frac{q(x)}{\mu-\lambda\bm{1}_{\mathcal{C}}(x)},\quad \text{equivalently},\quad p(x) = \frac{q(x)}{\mu-\lambda}\bm{1}_{\mathcal{C}}(x)+\frac{q(x)}{\mu}\bm{1}_{\mathcal{C}^c}(x).
\end{equation}
Since $Q(\mathcal{C})<b$, the constraint is active (otherwise $p=q$ would be the solution), therefore
\begin{equation}
\begin{cases}
    \int_{\mathcal{C}} p^*(x)dx = \frac{Q(\mathcal{C})}{\mu-\lambda}=b,\\
    \int p^*(x)dx =\frac{Q(\mathcal{C})}{\mu-\lambda}+\frac{1-Q(\mathcal{C})}{\mu} = 1,
    \end{cases} \Rightarrow\begin{cases}
        \lambda^* = \frac{1-Q(\mathcal{C})}{1-b}-\frac{Q(\mathcal{C})}{b},\\
        \mu^* = \frac{1-Q(\mathcal{C})}{1-b}.
    \end{cases}
\end{equation}
Given $\lambda^*,\mu^*$, we have 
\begin{equation}
     p^{(\lambda^*)}(x) = \frac{bq(x)}{Q(\mathcal{C})}\bm{1}_{\mathcal{C}}(x)+\frac{(1-b)q(x)}{1-Q(\mathcal{C})}\bm{1}_{\mathcal{C}^c}(x).
\end{equation}
Since the feasible set is convex and the objective
$\int -q(x)\log p(x)dx$
is strictly convex in $p$ on the set where $q(x)>0$, the KKT solution is unique
on $\{x:q(x)>0\}$ and the optimizer is unique.

\end{proof}
\paragraph{Statement of Lemma~\ref{lemma:TVandconstraint}.}\emph{
  Let $P_1,P_2$ be probability measures defined on the same measurable space $(\Omega,\mathcal{A})$ with $\|P_1-P_2\|_{\mathrm{TV}}=\delta $, }
\begin{equation}
    \max\{0, P_2(\mathcal{C})-\delta\} \leq P_1(\mathcal{C}) \leq \min\{1, P_2(\mathcal{C})+\delta\} .
\end{equation}

\begin{proof}
 Note that the definition of total variation is
\begin{equation}
    \delta =\|P_1-P_2\|_{\mathrm{TV}}=\sup_{A\in \mathcal{A}} |P_1(A)-P_2(A)|\geq |P_1(\mathcal{C})-P_2(\mathcal{C})|.
\end{equation}
Therefore $P_1(\mathcal{C})\in [P_2(\mathcal{C})-\delta,P_2(\mathcal{C})+\delta]$. Also consider that $P_1(\mathcal{C})\in[0,1]$ by definition, we can conclude that $    \max\{0, P_2(\mathcal{C})-\delta\} \leq P_1(\mathcal{C}) \leq \min\{1, P_2(\mathcal{C})+\delta\}$.
\end{proof}

\section{Proofs for Constraint guarantees with FM Parameterization}\label{sec:proofpart1}

To control the total variation distance between
the FM-parameterized solution $P^{({\bar{\lambda}^*})}_{u_{\theta}}$ and the optimal solution to \eqref{eq:cvxoriginal} $P^{(\lambda^*)}$, 
we decompose this error as
\begin{equation}
    \label{eq:TVsketch1}\|P^{({\bar{\lambda}^*})}_{u_{\theta}}- P^{(\lambda^*)}\|_{\mathrm{TV}}\leq \|P^{(\bar{\lambda}^*)}_{u_{{\theta}}}- P^{(\bar{\lambda}^*)}\|_{\mathrm{TV}}+\|P^{(\bar{\lambda}^*)}- P^{(\lambda^*)}\|_{\mathrm{TV}},
\end{equation}
in which the first term on the RHS captures the approximation error introduced by restricting the
solution to the FM-parameterized class $U_\Theta$ at fixed dual variable
$\bar{\lambda}^*$, while the second term captures the distributional error caused
by the dual-variable mismatch between $\bar{\lambda}^*$ and $\lambda^*$.

\subsection{Bound $\|P^{(\bar{\lambda}^*)}_{u_{{\theta}}}- P^{(\bar{\lambda}^*)}\|_{\mathrm{TV}}$}

\begin{lemma}\label{lemma:TVgap1}
Under Assumption~\ref{ass:esterror}, for any dual variable $\lambda$ considered
in the analysis, there exist uniform constants $A_1,A_2>0$ such that
\begin{equation}
    \mathrm{KL}\left(P^{(\lambda)}\middle\|P^{(\lambda)}_{u_\theta}\right)
    \leq A_1\nu + A_2\nu^2
    ,\quad
    \left\|P^{(\lambda)}_{u_\theta}-P^{(\lambda)}\right\|_{\mathrm{TV}}
    \leq
    \sqrt{\frac{1}{2}\left(A_1\nu+A_2\nu^2\right)} .
\end{equation}
where $A_1,A_2>0$ are uniform upper bounds on the regularity constants in Theorem~3.1 of~\citet{su2025flowmatchingkldivergence} over all dual variables $\lambda$ considered in the analysis. They depend only on the worst-case smoothness and boundedness of the corresponding probability paths and velocity fields in Assumption~\ref{ass:esterror}.
\end{lemma}
\begin{proof}
The KL bound $\mathrm{KL}\left(P^{(\lambda)}\middle\|P^{(\lambda)}_{u_\theta}\right)
    \leq A_1\nu + A_2\nu^2$ follows directly from Theorem~3.1 of~\cite{su2025flowmatchingkldivergence} under Assumption~\ref{ass:esterror}, with $A_1,A_2$
 chosen as uniform upper bounds over all dual variables $\lambda$ considered in the analysis. The total variation bound $\left\|P^{(\lambda)}_{u_\theta}-P^{(\lambda)}\right\|_{\mathrm{TV}}
    \leq
    \sqrt{\frac{1}{2}\left(A_1\nu+A_2\nu^2\right)}$ then follows immediately from Pinsker's inequality.
\end{proof}

Plugging in $\lambda = \bar{\lambda}^*$ in Lemma~\ref{lemma:TVgap1} leads to
\begin{equation}
    \|P^{(\bar{\lambda}^*)}_{u_{{\theta}}}- P^{(\bar{\lambda}^*)}\|_{\mathrm{TV}}\leq
    \sqrt{\frac{1}{2}\left(A_1\nu+A_2\nu^2\right)}.
\end{equation}

\subsection{Bound $\|P^{(\bar{\lambda}^*)}- P^{(\lambda^*)}\|_{\mathrm{TV}}$}

\begin{lemma}\label{lemma:lambdabound}
Under Assumption~\ref{ass:feasibility}, for dual variables  $\lambda_1,\lambda_2\geq 0$, 
$\|P^{(\lambda_1)}- P^{(\lambda_2)}\|_{\mathrm{TV}}$ is bounded in the sense that
\begin{equation}
    \|P^{(\lambda_1)}- P^{(\lambda_2)}\|_{\mathrm{TV}}\leq \frac{1}{\left((1-Q({\mathcal{C}}))^{1/3}+Q({\mathcal{C}})^{1/3}\right)^3}|\lambda_1-{\lambda_2}|.
\end{equation} 
\begin{proof}
    To compare the TV distance between $P^{(\lambda_1)}$ and $P^{(\lambda_2)}$, we first fix $\lambda$ and find the solution of 
    \begin{equation}
        \mathrm{minimize}~ \mathbb{E}_{x\sim Q}[-\log p(x)]+\lambda \left(b-\int_\mathcal{C} p(x) dx\right),~~\mathrm{s.t.}~~\int p(x)dx=1, p(x)\geq 0.
    \end{equation}
$p(x)\geq0$ automatically holds since the primal objective includes $\log p(x)$. Then we form a Lagrangian:
\begin{equation}
    \mathcal{J}(p,\lambda,\mu) = \int -q(x)\log p(x)d{x}+\lambda \left(b-\int_\mathcal{C} p(x) dx\right)+\mu\left(\int p(x)dx-1\right),
\end{equation}
in which $\lambda$ is a fixed multiplier for the inequality constraint and $\mu$ is the multiplier for the equality constraint. The KKT condition shows
\begin{equation}
    \frac{\partial\mathcal{J}}{\partial p(x)} = -\frac{q(x)}{p(x)}-\lambda \bm{1}_{\mathcal{C}}(x)+\mu =0.
\end{equation}
The solution of $p^{(\lambda)}(x)$ is
\begin{equation}\label{eq:sol1}
    p^{(\lambda)}(x) = \frac{q(x)}{\mu-\lambda\bm{1}_{\mathcal{C}}(x)},\quad \text{equivalently},\quad p^{(\lambda)}(x) = \frac{q(x)}{\mu-\lambda}\bm{1}_{\mathcal{C}}(x)+\frac{q(x)}{\mu}\bm{1}_{\mathcal{C}^c}(x).
\end{equation}
Integrate in $\mathcal{C}$ and $\mathcal{C}^c$ respectively leads to 
\begin{equation}
    \begin{cases}
        P^{(\lambda)}(\mathcal{C}) = \int_\mathcal{C}p^{(\lambda)}(x) dx = \frac{Q(\mathcal{C})}{\mu-\lambda},\\
        1-P^{(\lambda)}(\mathcal{C}) = \int_{\mathcal{C}^c}p^{(\lambda)}(x) dx =\frac{Q(\mathcal{C}^c)}{\mu}.
    \end{cases}
\end{equation}
This leads to a relationship between $\lambda,\mu$ and $P^{(\lambda)}(\mathcal{C})$:
\begin{equation} \label{eq:lambdaeq}
    \lambda = \frac{Q(\mathcal{C}^c)}{1-P^{(\lambda)}(\mathcal{C})}-\frac{Q(\mathcal{C})}{P^{(\lambda)}(\mathcal{C})}, ~~\mu=\frac{Q(\mathcal{C}^c)}{1-P^{(\lambda)}(\mathcal{C})}.
\end{equation}
Substituting back into \eqref{eq:sol1} leads to
\begin{equation}\label{eq:sol2}
    p^{(\lambda)}(x)= \frac{P^{(\lambda)}(\mathcal{C})}{Q(\mathcal{C})}q(x)\bm{1}_{\mathcal{C}}(x)+\frac{1-P^{(\lambda)}(\mathcal{C})}{Q(\mathcal{C}^c)}q(x)\bm{1}_{\mathcal{C}^c}(x).
\end{equation}
Now we evaluate the total variation using the expression in \eqref{eq:sol2},
\begin{equation}
\resizebox{0.92\textwidth}{!}{$
\begin{aligned}
    \|P^{(\lambda_1)}-P^{(\lambda_2)}\|_{\mathrm{TV}}&=\frac{1}{2}\int|p^{(\lambda_1)}(x)-p^{(\lambda_2)}(x)|dx\\
    &=\frac{1}{2}\left(\int_{\mathcal{C}}|p^{(\lambda_1)}(x)-p^{(\lambda_2)}(x)|dx+\int_{\mathcal{C}^c}|p^{(\lambda_1)}(x)-p^{(\lambda_2)}(x)|dx\right)\\&=
    \frac{1}{2}\left(\int_{\mathcal{C}}\left|\frac{P^{(\lambda_1)}(\mathcal{C})-P^{(\lambda_2)}(\mathcal{C})}{Q(\mathcal{C})}q(x)\right|dx+\int_{\mathcal{C}^c}\left|\frac{P^{(\lambda_1)}(\mathcal{C})-P^{(\lambda_2)}(\mathcal{C})}{Q(\mathcal{C}^{c}}))q(x)\right|dx\right)\\&=
    \frac{1}{2}\left(|P^{(\lambda_1)}(\mathcal{C})-P^{(\lambda_2)}(\mathcal{C})|+|P^{(\lambda_1)}(\mathcal{C})-P{(\lambda_2)}(\mathcal{C})|\right)
    \\&=|P^{(\lambda_1)}(\mathcal{C})-P^{(\lambda_2)}(\mathcal{C})|,
    \end{aligned}$}
\end{equation}
in which the third ``='' is by plugging in \eqref{eq:sol2}.

Based on \eqref{eq:lambdaeq},
define a function $\Lambda(r)$ as follows:
\begin{equation}
    \Lambda(r):=\frac{Q(\mathcal{C}^c)}{1-r}-\frac{Q(\mathcal{C})}{r},\quad\text{with derivative}\quad \Lambda'(r)=\frac{Q(\mathcal{C}^c)}{(1-r)^2}+\frac{Q(\mathcal{C})}{r^2}.
\end{equation}

Note that $\forall r\in(0,1), \Lambda'(r)>0$, the mean-value theorem shows 

$\exists\xi\in \left(\min\{P^{(\lambda_1)}(\mathcal{C}),P^{(\lambda_2)}(\mathcal{C})\},\max\{P^{(\lambda_1)}(\mathcal{C}),P^{(\lambda_2)}(\mathcal{C})\}\right),$
\begin{equation}
    \begin{aligned}
        |P^{(\lambda_1)}(\mathcal{C})-P^{(\lambda_2)}(\mathcal{C})|&=\frac{|\lambda_1-\lambda_2|}{\Lambda'(\xi)}\\&
\leq \frac{|\lambda_1-\lambda_2|}{\inf_{r\in \left(\min\{P^{(\lambda_1)}(\mathcal{C}),P^{(\lambda_2)}(\mathcal{C})\},\max\{P^{(\lambda_1)}(\mathcal{C}),P^{(\lambda_2)}(\mathcal{C})\}\right)} \frac{Q(\mathcal{C}^c)}{(1-r)^2}+\frac{Q(\mathcal{C})}{r^2}}
\\&\leq \frac{|\lambda_1-\lambda_2|}{\inf_{r\in (0,1)} \frac{Q(\mathcal{C}^c)}{(1-r)^2}+\frac{Q(\mathcal{C})}{r^2}}.
    \end{aligned}
\end{equation}
Let $\frac{\partial\Lambda'(r)}{\partial r}=0$ leads to $r_{min}=\frac{t}{1+t}, t=\left(\frac{Q(\mathcal{C})}{1-Q(\mathcal{C})}\right)^{1/3}$, $\Lambda'(r_{min}) = (1+t)^2(1-Q(\mathcal{C})+Q(\mathcal{C})t^{-2})$.
This leads to the bound 
\begin{equation}\label{eq:rgaptight}
    \|P^{(\lambda_1)}- P^{(\lambda_2)}\|_{\mathrm{TV}}\leq \frac{|\lambda_1-{\lambda_2}|}{(1+t)^2(1-Q(\mathcal{C})+Q(\mathcal{C})t^{-2})} = \frac{|\lambda_1-{\lambda_2}|}{\left((1-Q({\mathcal{C}}))^{1/3}+Q({\mathcal{C}})^{1/3}\right)^3}.
\end{equation}
\end{proof}

Substituting $\bar{\lambda}^*$ and $\lambda^*$ into Lemma~\ref{lemma:lambdabound}, we can obtain 
\begin{equation}\label{eq:tv2}
    \|P^{(\bar{\lambda}^*)}- P^{({\lambda}^*)}\|_{\mathrm{TV}}\leq \frac{1}{\left((1-Q({\mathcal{C}}))^{1/3}+Q({\mathcal{C}})^{1/3}\right)^3}|\bar{\lambda}^*-\lambda^*|.
\end{equation} 
\end{lemma}

\begin{lemma}\label{lemma:kconcave} Under Assumptions~\ref{ass:feasibility} and \ref{ass:dual_upper_bound},
    $g(\lambda)$ is $k$-strongly concave in $\lambda\in [0,\lambda_{\max}]$, in which
    \begin{equation}
    \begin{aligned}
        &k=\min_{\lambda\in [0,\lambda_{\max}]} -\frac{Q(\mathcal{C})}{(\mu(\lambda)-\lambda)^2}\left(\frac{-\mu(\lambda)+\lambda+Q(\mathcal{C})}{2\mu(\lambda)-\lambda-1}\right), \\
        &\mu(\lambda) = \frac{1}{2}\left(\lambda+1+\sqrt{(\lambda+1)^2-4\lambda(1-Q(\mathcal{C}))}\right).
        \end{aligned}
    \end{equation}
\end{lemma}

\begin{proof}
    Continue from \eqref{eq:sol1}
    $p^{(\lambda)}(x) = \frac{q(x)}{\mu-\lambda\bm{1}_{\mathcal{C}}(x)},\mu>\lambda$.
The requirement $\int p(x) dx=1$ leads to
\begin{equation}\label{eq:kkt}
    \frac{Q(\mathcal{C})}{\mu-\lambda}+\frac{1-Q(\mathcal{C})}{\mu}=1,\quad \mu^2-(\lambda+1)\mu+\lambda(1-Q(\mathcal{C}))=0.
\end{equation}
The solutions for $\mu$ are
\begin{equation}
    \mu(\lambda) = \frac{1}{2}\left(\lambda+1\pm\sqrt{(\lambda+1)^2-4\lambda(1-Q(\mathcal{C}))}\right).
\end{equation}
Note that $\mu-\lambda>0$ must hold. If $\mu(\lambda) = \frac{1}{2}\left(\lambda+1-\sqrt{(\lambda+1)^2-4\lambda(1-Q(\mathcal{C}))}\right)$,
\begin{equation}
    \begin{aligned}
        \mu(\lambda)-\lambda=\frac{1}{2}\left(-\lambda+1-\sqrt{(\lambda-1)^2+4\lambda Q(\mathcal{C})}\right)\leq0.
    \end{aligned}
\end{equation}
This root satisfies $\mu(\lambda)-\lambda\leq 0$, with equality at $\lambda=0$, and therefore violates the requirement $\mu>\lambda$.
This eliminates one solution, and the solution to $\mu$ is
\begin{equation}\label{eq:musol}
    \mu(\lambda) = \frac{1}{2}\left(\lambda+1+\sqrt{(\lambda+1)^2-4\lambda(1-Q(\mathcal{C}))}\right).
\end{equation}
Substitute the solutions for $\lambda,p^{(\lambda)}$ into $g(\lambda)$ and note that $\lambda Q(\mathcal C) = (\mu-1)(\mu-\lambda)$ from \eqref{eq:kkt},
\begin{equation}
\begin{aligned}
    g(\lambda) &= \mathbb{E}_{x\sim Q}[-\log p^{(\lambda)}(x)]+\lambda (b-\int_\mathcal{C}p^{(\lambda)}(x)dx)\\
    &=\mathbb{E}_{x\sim Q}[-\log q(x)]+\mathbb{E}_{x\sim Q}[\log(\mu(\lambda)-\lambda\bm{1}_{\mathcal{C}}(x))]+\lambda b-\lambda\frac{Q(\mathcal{C})}{\mu(\lambda)-\lambda}
    \\& = 1+\mathbb{E}_{x\sim Q}[-\log q(x)]+\mathbb{E}_{x\sim Q}[\log(\mu(\lambda)-\lambda\bm{1}_{\mathcal{C}}(x))]+\lambda b-\mu(\lambda).
    \end{aligned}
\end{equation}
Then we can check the concavity of $g$,
\begin{equation}
    \begin{aligned}
     g'(\lambda)=   \frac{\partial g}{\partial\lambda} &= \mathbb{E}_{x\sim Q}\left[\frac{\frac{\partial \mu}{\partial \lambda}-\bm{1}_\mathcal{C}(x)}{\mu(\lambda)-\lambda\bm{1}_\mathcal{C}(x)}\right]+b-\frac{\partial \mu}{\partial \lambda}\\&=
        -\frac{Q(\mathcal{C})}{\mu(\lambda)-\lambda}+b+\frac{\partial \mu}{\partial \lambda}\left(\frac{Q(\mathcal{C})}{\mu(\lambda)-\lambda}+\frac{1-Q(\mathcal{C})}{\mu(\lambda)}-1\right)
        \\& = -\frac{Q(\mathcal{C})}{\mu(\lambda)-\lambda}+b.
    \end{aligned}
    \end{equation}
    \begin{equation}
        \begin{aligned}
            g''(\lambda)=\frac{\partial^2g}{\partial \lambda^2}&=\frac{Q(\mathcal{C})}{(\mu(\lambda)-\lambda)^2}(\frac{\partial\mu}{\partial\lambda}-1)
            \\&=\frac{Q(\mathcal{C})}{(\mu(\lambda)-\lambda)^2}\left(\frac{1}{2}\left(1+\frac{\lambda+1-2(1-Q(\mathcal{C}))}{\sqrt{(\lambda+1)^2-4\lambda(1-Q(\mathcal{C}))}}\right)-1\right)
            \\&=\frac{Q(\mathcal{C})}{(\mu(\lambda)-\lambda)^2}\left(\frac{1}{2}\left(-1+\frac{\lambda-1+2Q(\mathcal{C})}{2\mu(\lambda)-\lambda-1}\right)\right)
            \\&=\frac{Q(\mathcal{C})}{(\mu(\lambda)-\lambda)^2}\left(\frac{-\mu(\lambda)+\lambda+Q(\mathcal{C})}{2\mu(\lambda)-\lambda-1}\right).
        \end{aligned}
    \end{equation}
$2\mu(\lambda)-\lambda-1>0$ since
it equals $\sqrt{(\lambda+1)^2-4\lambda(1-Q(\mathcal{C}))}$. Note that $0<Q(\mathcal{C})<b<1$,
\eqref{eq:kkt} shows that 
\begin{equation}
    \frac{\mu(\lambda)-\lambda}{Q(\mathcal{C})}=\frac{\mu(\lambda)}{\mu(\lambda)-1+Q(\mathcal{C})}>1.
\end{equation}
Therefore $-\mu(\lambda)+\lambda+Q(\mathcal{C})<0$ and $g''(\lambda)<0$. Moreover, $-g''(\lambda)$ is continuous and strictly positive on the
compact interval $[0,\lambda_{\max}]$. Therefore,
    $k:=\min_{\lambda\in[0,\lambda_{\max}]}-g''(\lambda)>0$ and this proves that $g$ is $k$-strongly concave on this interval.
\end{proof}

\begin{lemma}\label{lemma:dualfuncgap}
Under Assumptions~\ref{ass:feasibility}, \ref{ass:esterror} and \ref{ass:densityrange}, 
    for fixed $\lambda$, 
    \begin{equation}\label{eq:densityrange}
           \bar{g}(\lambda) -  g(\lambda)\leq \mu(\lambda) (A_1\nu + A_2\nu^2),
    \end{equation}
    in which $\mu(\lambda) = \frac{1}{2}\left(\lambda+1+\sqrt{(\lambda+1)^2-4\lambda(1-Q(\mathcal{C}))}\right)$.
\end{lemma}

\begin{proof}
    From \eqref{eq:sol1} and \eqref{eq:musol}, we can obtain
    \begin{equation}\label{eq:mulambda}
        \frac{q(x)}{p^{(\lambda)}(x)}=\mu(\lambda)-\lambda\mathbf 1_{\mathcal{C}}(x)\le \mu(\lambda),
    \end{equation}
    in which $\mu(\lambda) = \frac{1}{2}\left(\lambda+1+\sqrt{(\lambda+1)^2-4\lambda(1-Q(\mathcal{C}))}\right)$.
Therefore,
\begin{equation}\label{eq:gbarg}
    \begin{alignedat}{2}
        \bar{g}(\lambda) -  g(\lambda) &=  
         \int q(x)\log \frac{p^{(\lambda)}(x)}{p_{u_\theta}^{(\lambda)}(x)}dx+\lambda\int_{\mathcal{C}} (p^{(\lambda)}(x)-p_{u_\theta}^{(\lambda)}(x)) dx\\
         & = \int q(x)\log \frac{p^{(\lambda)}(x)}{p_{u_\theta}^{(\lambda)}(x)}dx+\int \lambda\bm{1}_{\mathcal{C}}(x)(p^{(\lambda)}(x)-p_{u_\theta}^{(\lambda)}(x)) dx
         \\
         & = \int q(x)\log \frac{p^{(\lambda)}(x)}{p_{u_\theta}^{(\lambda)}(x)}dx+\int\left(\mu-\frac{q(x)}{p^{(\lambda)}(x)}\right)(p^{(\lambda)}(x)-p_{u_\theta}^{(\lambda)}(x)) dx
\\&\quad\text{(Since $\lambda\bm{1}_{\mathcal{C}}(x) = \mu-\frac{q(x)}{p^{(\lambda)}(x)}$ from \eqref{eq:sol1})}
         \\
         & = \int q(x)\log \frac{p^{(\lambda)}(x)}{p_{u_\theta}^{(\lambda)}(x)}dx+\int\mu(p^{(\lambda)}(x)-p_{u_\theta}^{(\lambda)}(x)) dx -\int\left(\frac{q(x)}{p^{(\lambda)}(x)}\right)(p^{(\lambda)}(x)-p_{u_\theta}^{(\lambda)}(x)) dx
         \\
         & = \int q(x)\log \frac{p^{(\lambda)}(x)}{p_{u_\theta}^{(\lambda)}(x)}dx+\int q(x) \left(\frac{{p_{u_\theta}^{(\lambda)}(x)}}{p^{(\lambda)}(x)}\right) dx-1
         \\& = \int q(x)\left[\log \frac{p^{(\lambda)}(x)}{p_{u_\theta}^{(\lambda)}(x)}dx+\left(\frac{{p_{u_\theta}^{(\lambda)}(x)}}{p^{(\lambda)}(x)}\right) -1 \right]dx
                  \\& = \int p^{(\lambda)} (x)\frac{q(x)}{p^{(\lambda)}(x)}\left[\log \frac{p^{(\lambda)}(x)}{p_{u_\theta}^{(\lambda)}(x)}dx+\left(\frac{{p_{u_\theta}^{(\lambda)}(x)}}{p^{(\lambda)}(x)}\right) -1 \right]dx
                  \\& \leq\mu(\lambda) \int p^{(\lambda)} (x)\left[\log \frac{p^{(\lambda)}(x)}{p_{u_\theta}^{(\lambda)}(x)}dx+\left(\frac{{p_{u_\theta}^{(\lambda)}(x)}}{p^{(\lambda)}(x)}\right) -1 \right]dx\\&\quad \text{(Since \eqref{eq:mulambda} and $\forall a>0, -\log a+a-1\geq 0$)}
\\& =\mu(\lambda) \int\left[ p^{(\lambda)} (x)\log \frac{p^{(\lambda)}(x)}{p_{u_\theta}^{(\lambda)}(x)}dx+{{p_{u_\theta}^{(\lambda)}(x)}} -p^{(\lambda)} (x)\right]dx
\\& =\mu(\lambda)\left( \DKL(P^{(\lambda)}\|P_{u_\theta}^{(\lambda)})+1-1\right)
\\&\leq \mu(\lambda) (A_1\nu + A_2\nu^2)\quad\text{(By Lemma~\ref{lemma:TVgap1})}
    \end{alignedat}
\end{equation}
\end{proof}

\begin{lemma}\label{lemma:lambdagap1}
Under Assumptions \ref{ass:feasibility}, \ref{ass:esterror}, \ref{ass:dual_upper_bound}, \ref{ass:densityrange}, the dual variable gap between the FM-parameterized Lagrangian's $\bar{\lambda}^*$ and the distribution-level Lagrangian's $\lambda^*$ can be bounded by
    \begin{equation}
        |\bar{\lambda}^*-\lambda^*|^2\leq \frac{2}{k}\mu(\bar{\lambda}^*) (A_1\nu + A_2\nu^2).
    \end{equation}
\end{lemma}

\begin{proof}
    By $k$-strong concavity,
    \begin{equation}\label{eq:firstorderconcave}
        g(\bar{\lambda}^*)\leq g(\lambda^*)+\nabla_\lambda g(\lambda^*)(\bar{\lambda}^*-\lambda^*)-\frac{k}{2}|\bar{\lambda}^*-\lambda^*|^2.
    \end{equation}
    First check $\nabla_\lambda g(\lambda^*)(\bar{\lambda}^*-\lambda^*)$.
    Note that $\nabla_\lambda g(\lambda^*)=0$, \eqref{eq:firstorderconcave} becomes $g(\bar{\lambda}^*)\leq g(\lambda^*)-\frac{k}{2}|\bar{\lambda}^*-\lambda^*|^2$, i.e.,
    \begin{equation}
    \begin{aligned}
        \frac{k}{2}|\bar{\lambda}^*-\lambda^*|^2
        &\leq g(\lambda^*)-g(\bar{\lambda}^*)\leq g(\lambda^*)-\left(\bar{g}(\bar{\lambda}^*)-\mu(\bar{\lambda}^*) (A_1\nu + A_2\nu^2)\right).
        \end{aligned}
    \end{equation}

Consider the fact that 
\begin{equation}
    g(\lambda^*)\leq \bar{g}(\lambda^*)\leq \bar{g}(\bar{\lambda}^*), \quad\text{equivalently,}\quad g(\lambda^*)-\bar{g}(\bar{\lambda}^*)\leq 0,
\end{equation}
we can finish this proof.
\end{proof}

\subsection{Proof of Theorem~\ref{theorem:constraintsatisfaction}}

\begin{proof}
Using \eqref{eq:TVsketch1}, Lemma~\ref{lemma:TVgap1}, Lemma~\ref{lemma:lambdabound} and Lemma~\ref{lemma:lambdagap1},
the total variation bound can be given by
\begin{equation}
    \|P^{(\bar{\lambda}^*)}_{u_{\theta}}-P^{(\lambda^*)}\|_{\mathrm{TV}}\leq\mathcal{E}_{\mathrm{para}},
\end{equation}
in which 
\begin{equation}
    \mathcal{E}_{\mathrm{para}} = \frac{1}{\left((1-Q({\mathcal{C}}))^{1/3}+Q({\mathcal{C}})^{1/3}\right)^3}\sqrt{\frac{2}{k}\mu(\bar{\lambda}^*) (A_1\nu + A_2\nu^2)}+\sqrt{\frac{1}{2}(A_1 \nu + A_2 \nu^2 )}
\end{equation}

From Lemma~\ref{lemma:TVandconstraint}, the
constraint satisfaction rate $P_{u_{\theta}}^{(\bar{\lambda}^*)}(\mathcal{C})$ relates to the same $\mathcal{E}_{\mathrm{para}}$,
\begin{equation}
    P_{u_{\theta}}^{(\bar{\lambda}^*)}(\mathcal{C})\geq b-\mathcal{E}_{\mathrm{para}}.
\end{equation}
\end{proof}

\vspace{-10pt}

\section{Proof for Proposition~\ref{prop:grad}} \label{sec:proof4propgrad}

\paragraph{Statement of Proposition~\ref{prop:grad}.}\emph{
Under the assumption that
$\nabla_\theta f^{(u_\theta)}(x_{i\Delta t},i\Delta t)$ exists and is uniformly
bounded for all $i\in\{0,\ldots,N-1\}$, $i\sim \text{Unif}\{0,\cdots,N-1\}$ and 
$x_{i\Delta t}\sim(\Phi^{\Delta t}_{u_\theta,i\Delta t})_{\#}P_0$,
\begin{equation}
\resizebox{0.92\textwidth}{!}{$
     \nabla_\theta\mathcal{F}_{u_\theta}(\mathcal{C})= N\mathbb{E}_{i,x_{i\Delta t}}\left[\nabla_\theta u_\theta(x_{i\Delta t},i\Delta t)\nabla_uf^{(u_\theta)}(x_{i\Delta t}+u\Delta t, (i+1)\Delta t)|_{u = u_\theta(x_{i\Delta t},i\Delta t)}\right],$}
\end{equation}
in which $\nabla_\theta u_\theta(x_t,t)$ is a Jacobian matrix such that each
column is the gradient $\nabla_\theta [u_\theta(x_t,t)]_j$ of the $j$-th dimension of the $d$-dimensional velocity with respect to the velocity parameters, i.e.,
\begin{equation}
    \nabla_\theta u_\theta(x_t,t) = [\nabla_\theta [u_\theta(x_t,t)]_1\quad \nabla_\theta [u_\theta(x_t,t)]_2\quad \cdots \quad\nabla_\theta [u_\theta(x_t,t)]_d].
    \end{equation}}

\begin{proof}

For $t = i\Delta t, i \in \{0,1,\cdots, N-1\}$,
  \begin{equation}
  \resizebox{0.92\textwidth}{!}{$
\begin{aligned}
    &\nabla_\theta f^{(u_\theta)}(x_t,t) \\& = \nabla_\theta f^{(u_\theta)}(x_t+u_\theta(x_t,t)\Delta t, t+\Delta t) \qquad \text{(by property (a))}
    \\& = \nabla_\theta\ u_\theta(x_t,t)\nabla_u f^{(u_\theta)}(x_t+u\Delta t,t+\Delta t)|_{u=u_\theta(x_t,t)}+\nabla_\theta f^{(u_\theta)}(x_t+u\Delta t,t+\Delta t)|_{u=u_\theta(x_t,t)}
    \\&=\nabla_\theta\ u_\theta(x_t,t)\nabla_u f^{(u_\theta)}(x_t+u\Delta t,t+\Delta t)|_{u=u_\theta(x_t,t)}+\nabla_\theta f^{(u_\theta)}(x_{t+\Delta t} ,t+\Delta t).
    \end{aligned}$}
\end{equation}
Repeating this expansion yields
\begin{equation}\resizebox{0.92\textwidth}{!}{$
    \begin{aligned}
        &\nabla_\theta f^{(u_\theta)}(x_0,0) \\&= \sum_{i=0}^{N-1}\nabla_\theta\ u_\theta(x_{i\Delta t},i\Delta t)\nabla_u f^{(u_\theta)}(x_{i\Delta t}+u\Delta t,(i+1)\Delta t)|_{u=u_\theta(x_{i\Delta t},i\Delta t)}+\nabla_\theta f^{(u_\theta)}(x_1 ,1)
        \\& =\sum_{i=0}^{N-1}\nabla_\theta\ u_\theta(x_{i\Delta t},i\Delta t)\nabla_u f^{(u_\theta)}(x_{i\Delta t}+u\Delta t,(i+1)\Delta t)|_{u=u_\theta(x_{i\Delta t},i\Delta t)}. \qquad \text{(by property (b))}
    \end{aligned}$}
\end{equation}

Under the stated assumptions, by applying the dominated convergence theorem, one can exchange the
integration and differentiation,
\begin{equation}
\resizebox{0.92\textwidth}{!}{$
\begin{aligned}
    \nabla_\theta \mathcal{F}_{u_\theta}(\mathcal{C}) &= \nabla_\theta\int p_0(x_0) f^{(u_\theta)}(x_0,0) dx_0
    \\&=\int p_0(x_0) \nabla_\theta f^{(u_\theta)}(x_0,0) dx_0\\&=\int p_0(x_0) \sum_{i=0}^{N-1}\nabla_\theta\ u_\theta(x_{i\Delta t},i\Delta t)\nabla_u f^{(u_\theta)}(x_{i\Delta t}+u\Delta t,(i+1)\Delta t)|_{u=u_\theta(x_{i\Delta t},i\Delta t)}  dx_0\\
    &=\mathbb{E}_{ x_{i\Delta t}}\left[\sum_{i=0}^{N-1}\nabla_\theta\ u_\theta(x_{i\Delta t},i\Delta t)\nabla_u f^{(u_\theta)}(x_{i\Delta t}+u\Delta t,(i+1)\Delta t)|_{u=u_\theta(x_{i\Delta t},i\Delta t)} \right]\\
    & = N\mathbb{E}_{i,x_{i\Delta t}}\left[\nabla_\theta\ u_\theta(x_{i\Delta t},i\Delta t)\nabla_u f^{(u_\theta)}(x_{i\Delta t}+u\Delta t,(i+1)\Delta t)|_{u=u_\theta(x_{i\Delta t},i\Delta t)}\right].
\end{aligned}$}
\end{equation}

\end{proof}
\vspace{-12pt}
\section{Proofs for Constraint guarantees on PDFM }\label{sec:proofpart2}
To assess the TV gap between the distribution ${P_{u'_{\theta}}^{(\lambda(i))}}$  induced by $u'^{(\lambda(i))}_\theta$ solved by Alg.~\ref{alg:pd_practical} at the $i$-th iteration and the optimal solution $P^{(\lambda^*)}$ to \eqref{eq:cvxoriginal}, 
we decompose this gap as

\begin{equation}
    \left\|{P_{u'_{\theta}}^{(\lambda(i))}}-P^{(\lambda^*)} \right\|_{\mathrm{TV}}\leq \left\|{P_{u'_{\theta}}^{(\lambda(i))}}-P^{(\lambda(i))} \right\|_{\mathrm{TV}}+\left\|{P^{(\lambda(i))}}-P^{(\lambda^*)} \right\|_{\mathrm{TV}}.
\end{equation}
Similar to the decomposition in Appendix~\ref{sec:proofpart1},
the first term on the RHS captures the approximation error introduced by both restricting the solution to the FM-parameterized class $U_\Theta$ and solving the minimizer with the approach in Sec.~\ref{sec:gradimp}, while the second term captures the distributional error caused
by the dual-variable mismatch between $\lambda(i)$ and $\lambda^*$.

\subsection{Bound $\left\|{P_{u'_\theta}^{(\lambda)}}-P^{(\lambda)} \right\|_{\mathrm{TV}}$}

\begin{lemma} \label{lemma:TVgapthetai}
    Under Assumptions~\ref{ass:feasibility}, \ref{ass:esterror}, \ref{ass:dual_upper_bound}, \ref{ass:densityrange}, \ref{ass:lambdamax} and \ref{ass:esterrorGD}, for all $\lambda\in [0,\lambda_{\max}]$,
\begin{equation}
    \left\|{P_{u'_\theta}^{(\lambda)}}-P^{(\lambda)} \right\|_{\mathrm{TV}}\leq e(\lambda)\leq e_{\max},
\end{equation}
in which
$
e(\lambda):=\sqrt{\frac{\varepsilon+\mu(\lambda) (A_1\nu + A_2\nu^2)}{2(\mu(\lambda) - \lambda)}}$, $\mu(\lambda) = \frac{1}{2}\left(\lambda+1+\sqrt{(\lambda+1)^2-4\lambda(1-Q(\mathcal{C}))}\right)$, $e_{\max}: =\max_{\lambda\in[0,\lambda_{\max}]} e(\lambda).$
\end{lemma}
\begin{proof}
From Assumption~\ref{ass:esterrorGD} and Lemma~\ref{lemma:dualfuncgap},
\begin{equation}\label{eq:epsgd}\resizebox{0.92\textwidth}{!}{$
\begin{aligned}
     L(p_{u'_{\theta}}^{(\lambda)}, \lambda) - g(\lambda)&=L(p_{u'_{\theta}}^{(\lambda)}, \lambda) - \bar{g}(\lambda)+ \bar{g}(\lambda)-g(\lambda)\leq \varepsilon+\mu(\lambda) (A_1\nu + A_2\nu^2).
     \end{aligned}$}
\end{equation}

 From \eqref{eq:sol1} and \eqref{eq:musol}, we can obtain
 \begin{equation}
     \frac{q(x)}{p^{(\lambda)}(x)} = {\mu-\lambda \bm{1}_{\mathcal{C}}(x)}\geq \mu(\lambda)-\lambda,
 \end{equation}
    in which $\mu(\lambda) = \frac{1}{2}\left(\lambda+1+\sqrt{(\lambda+1)^2-4\lambda(1-Q(\mathcal{C}))}\right)$.
Therefore,
\begin{equation}\label{eq:epsgd2}
    \begin{aligned}
        & L(p_{u'_{\theta}}^{(\lambda)}, \lambda) - g(\lambda)\\
                  \\& = \int p^{(\lambda)} (x)\frac{q(x)}{p^{(\lambda)}(x)}\left[\log \frac{p^{(\lambda)}(x)}{p_{u'_\theta}^{(\lambda)}(x)}dx+\left(\frac{{p_{u'_\theta}^{(\lambda)}(x)}}{p^{(\lambda)}(x)}\right) -1 \right]dx \quad\text{(With the same procedure as \eqref{eq:gbarg})}
\\& \geq(\mu(\lambda) - \lambda) \int p^{(\lambda)} (x)\left[\log \frac{p^{(\lambda)}(x)}{p_{u'_\theta}^{(\lambda)}(x)}dx+\left(\frac{{p_{u'_\theta}^{(\lambda)}(x)}}{p^{(\lambda)}(x)}\right) -1 \right]dx
\\& =(\mu(\lambda) - \lambda) \int\left[ p^{(\lambda)} (x)\log \frac{p^{(\lambda)}(x)}{p_{u'_\theta}^{(\lambda)}(x)}dx+{{p_{u'_\theta}^{(\lambda)}(x)}} -p^{(\lambda)} (x)\right]dx
\\& =(\mu(\lambda) - \lambda)\left( \DKL(P^{(\lambda)}\|P_{u'_\theta}^{(\lambda)})+1-1\right)
\\&\geq (\mu(\lambda) - \lambda) \cdot 2\left\|{P_{u'_\theta}^{(\lambda)}}-P^{(\lambda)} \right\|_{\mathrm{TV}}^2  
\\&= 2(\mu(\lambda) - \lambda) \left\|{P_{u'_\theta}^{(\lambda)}}-P^{(\lambda)} \right\|_{\mathrm{TV}}^2 .
\end{aligned}
\end{equation}
Since $\mu(\lambda)-\lambda>0$ and $e(\lambda)$ is continuous on the
compact interval $[0,\lambda_{\max}]$, its finite maximum
$e_{\max}$ is attained.
\end{proof}

\subsection{Bound $\|P^{(\lambda(i))}- P^{(\lambda^*)}\|_{\mathrm{TV}}$}

\begin{lemma} \label{lemma:dualgap2}
    Under Assumption~\ref{ass:esterrorGD} and follow Lemma~\ref{lemma:TVgapthetai}, 
    \begin{equation}
\begin{aligned}
|\nabla_{\lambda}L(p_{u'_\theta}^{(\lambda)}, \lambda)-\nabla_\lambda g(\lambda)|&\leq e(\lambda).
\end{aligned}
\end{equation}

\end{lemma}
\begin{proof}
    \begin{equation}
\begin{aligned}
|\nabla_{\lambda}L(p_{u'_\theta}^{(\lambda)}, \lambda)-\nabla_\lambda g(\lambda)|&=\left|P_{u'_\theta}^{(\lambda)}(\mathcal{C})-P^{(\lambda)}(\mathcal{C})\right|\leq \left\|P_{u'_\theta}^{(\lambda)}-P^{(\lambda)}\right\|_{\mathrm{TV}}\leq e(\lambda)
\end{aligned}
\end{equation}

\end{proof}

\begin{lemma}\label{lemma:pracricalPDconvergence}
Under Assumption~\ref{ass:exactdualgrad}, with $\forall i, 0\leq \eta_i\leq\frac{2}{kw}$, $\sum_{i=1}^\infty\eta_i=\infty, \sum_{i=1}^\infty\eta_i^2<\infty $,
    \begin{equation}  
    \limsup_{i \rightarrow \infty}|\lambda(i)-{\lambda}^*|^2\leq \frac{4e_{\max}^2}{k^2},
    \end{equation}
\end{lemma}

\begin{proof} 
Note that $\lambda(0)=\lambda'(0)=0$, and in Alg.~\ref{alg:pd_practical}, $\lambda'=\lambda/w$ is updated as $\lambda'(i+1)\gets[\lambda'(i) +\eta_{i}(b-P^{(\lambda(i))}_{u'_{\theta}}(\mathcal{C}))]_{+}$. The update in $\lambda$ is equivalent to  
\begin{equation}
    \lambda(i+1) = \left[\lambda(i)+w\eta_i\left(b-P^{(\lambda(i))}_{u'_\theta}(\mathcal{C})\right)\right]_+ = \left[\lambda(i)+w\eta_{i}\nabla_{\lambda}L(p^{(\lambda(i))}_{u'_\theta}, \lambda(i))\right]_+.
\end{equation}
Next we analyze the gap between $\lambda(i)$ and $\lambda^*$,
\begin{equation}\label{eq:exactlambdagap}
\resizebox{1\textwidth}{!}{$
    \begin{alignedat}{2}
        &|\lambda(i+1)-{\lambda}^*|^2\\&=|[\lambda (i)+w\eta_{i}\nabla_{\lambda}L(p^{(\lambda(i))}_{u'_\theta}, \lambda(i))]_{+}-{\lambda}^*|^2\\
        &\leq|\lambda (i)+w\eta_{i}\nabla_{\lambda}L(p^{(\lambda(i))}_{u'_\theta}, \lambda(i))-{\lambda}^*|^2~~\text{(Nonexpansiveness of projection)}
        \\&=|\lambda(i)-{\lambda}^*|^2+2w\eta_{i}(\lambda(i)-{\lambda}^*)\nabla_{\lambda}L(p^{(\lambda(i))}_{u'_\theta}, \lambda(i))+{(w\eta_i)^2}\left[\nabla_{\lambda}L(p^{(\lambda(i))}_{u'_\theta}, \lambda(i))\right]^2
        \\&\leq |\lambda(i)-{\lambda}^*|^2+2w\eta_i(\lambda(i)-{\lambda}^*)\nabla_{\lambda}L(p^{(\lambda(i))}_{u'_\theta}, \lambda(i))+{(w\eta_i)^2}\\&\qquad\left(\text{ Note that $\left[\nabla_{\lambda}L(p^{(\lambda(i))}_{u'_\theta}, \lambda(i))\right]^2\leq 1$ since $\nabla_{\lambda}L(p^{(\lambda(i))}_{u'_\theta}, \lambda(i)) = b-P^{(\lambda(i))}_{u'_\theta}(\mathcal{C})\in[b-1,b]$}\right)
        \\&= |\lambda(i)-{\lambda}^*|^2+2w\eta_i(\lambda(i)-{\lambda}^*)\left(\nabla_{\lambda}g(\lambda(i))+\nabla_{\lambda}L(p^{(\lambda(i))}_{u'_\theta}, \lambda(i))-\nabla_{\lambda}g(\lambda(i))\right)+{(w\eta_i)^2}
        \\& = |\lambda(i)-{\lambda}^*|^2+2w\eta_i(\lambda(i)-{\lambda}^*)\nabla_{\lambda}g(\lambda(i))+2w\eta_i(\lambda(i)-{\lambda}^*)\left(\nabla_{\lambda}L(p^{(\lambda(i))}_{u'_\theta}, \lambda(i))-\nabla_{\lambda}g(\lambda(i))\right)+{(w\eta_i)^2}
        \\& = |\lambda(i)-{\lambda}^*|^2+2w\eta_i(\lambda(i)-{\lambda}^*)(\nabla_{\lambda}g(\lambda(i))-\nabla_{\lambda}g(\lambda^*))\\&~~~~+2w\eta_i(\lambda(i)-{\lambda}^*)\left(\nabla_{\lambda}L(p^{(\lambda(i))}_{u'_\theta}, \lambda(i))-\nabla_{\lambda}g(\lambda(i))\right)+{(w\eta_i)^2}
        \quad(\text{Since } \nabla_{\lambda}g(\lambda^*)=0)\\
        &\leq |\lambda(i)-{\lambda}^*|^2-kw\eta_i|\lambda(i)-{\lambda}^*|^2+2w\eta_i(\lambda(i)-{\lambda}^*)\left(\nabla_{\lambda}L(p^{(\lambda(i))}_{u'_\theta}, \lambda(i))-\nabla_{\lambda}g(\lambda(i))\right)+{(w\eta_i)^2}\\&\qquad\text{($k$-strong concavity, Lemma~\ref{lemma:kconcave})}\\&
        \leq  |\lambda(i)-{\lambda}^*|^2-kw\eta_i|\lambda(i)-{\lambda}^*|^2+2w\eta_i|\lambda(i)-{\lambda}^*|\left|\nabla_{\lambda}L(p^{(\lambda(i))}_{u'_\theta}, \lambda(i))-\nabla_{\lambda}g(\lambda(i))\right|+{(w\eta_i)^2}
        \\&
        \leq  |\lambda(i)-{\lambda}^*|^2-kw\eta_i|\lambda(i)-{\lambda}^*|^2+2w\eta_i|\lambda(i)-{\lambda}^*|e(\lambda(i))+{(w\eta_i)^2}\quad\text{(Lemma~\ref{lemma:dualgap2})}
        \\&\leq  (1-kw\eta_i)|\lambda(i)-{\lambda}^*|^2+2w\eta_i|\lambda(i)-{\lambda}^*|e_{\max}+{(w\eta_i)^2}\quad\text{(Lemma~\ref{lemma:TVgapthetai})}
        \\&\leq (1-kw\eta_i)\left|\lambda(i)-{\lambda}^*\right|^2+\frac{k}{2}w\eta_i|\lambda(i)-{\lambda}^*|^2+\frac{2}{k}w\eta_i e_{\max}^2+(w\eta_i)^2
        \\&\leq  (1-\frac{k}{2}w\eta_i)\left|\lambda(i)-{\lambda}^*\right|^2+w\eta_i\frac{2}{k}e_{\max}^2+(w\eta_i)^2 .
    \end{alignedat}$}
\end{equation}

Rearranging the terms,
\begin{equation}
    |\lambda(i+1)-{\lambda}^*|^2-\frac{4e_{\max}^2}{k^2}\leq (1-\frac{k}{2}w\eta_i)(\left|\lambda(i)-{\lambda}^*\right|^2-\frac{4e_{\max}^2}{k^2})+(w\eta_i)^2.
\end{equation}
Unrolling this recursion yields
\begin{equation} \label{eq:recur1}
\resizebox{0.98\textwidth}{!}{$\displaystyle
     |\lambda(i+1)-{\lambda}^*|^2-\frac{4e_{\max}^2}{k^2}\leq \left(|\lambda(0)-\lambda^*|^2-\frac{4e_{\max}^2}{k^2}\right)\prod_{j=0}^i (1-\frac{k}{2}w\eta_j)+\sum_{n=0}^i\left(\prod_{j=n+1}^i(1-\frac{k}{2}w\eta_j)\right)(w\eta_n)^2,$}
\end{equation}
and we define the empty product $\prod_{j=i+1}^i(\cdot) = 1$ for the last term.

With $\forall i, 0\leq \eta_i\leq\frac{2}{kw}$, $\sum_{i=1}^\infty\eta_i=\infty, \sum_{i=1}^\infty\eta_i^2<\infty $, \eqref{eq:recur1} implies that 
\begin{equation}
    \displaystyle \limsup_{i\rightarrow \infty}|\lambda(i)-{\lambda}^*|^2\leq \frac{4e_{\max}^2}{k^2} .
\end{equation}
 \end{proof}

\subsection{ Proof of Theorem~\ref{theorem:bound2}}

 \begin{proof}
By Lemma~\ref{lemma:TVgapthetai} and Lemma~\ref{lemma:lambdabound},
\begin{equation}
    \begin{alignedat}{2}
              \left\|{P_{u'_{\theta}}^{(\lambda(i))}}-P^{(\lambda^*)} \right\|_{\mathrm{TV}}&\leq \left\|{P_{u'_{\theta}}^{(\lambda(i))}}-P^{(\lambda(i))} \right\|_{\mathrm{TV}}+\left\|{P^{(\lambda(i))}}-P^{(\lambda^*)} \right\|_{\mathrm{TV}}\\
          &\leq e(\lambda(i))+\frac{1}{\left((1-Q({\mathcal{C}}))^{1/3}+Q({\mathcal{C}})^{1/3}\right)^3}|{\lambda(i)}-{\lambda^*}|.
    \end{alignedat}
\end{equation}
By Lemma~\ref{lemma:pracricalPDconvergence},
\begin{equation}
    \limsup_{i\to\infty}  \left\|{P_{u'_{\theta}}^{(\lambda(i))}}-P^{(\lambda^*)} \right\|_{\mathrm{TV}}\leq  \mathcal{E}_{\mathrm{PDFM}},\quad \mathcal{E}_{\mathrm{PDFM}} =e_{\max}+\frac{2e_{\max}}{k\left((1-Q({\mathcal{C}}))^{1/3}+Q({\mathcal{C}})^{1/3}\right)^3} .
\end{equation}
and by Lemma~\ref{lemma:TVandconstraint} the constraint satisfaction rate satisfies
     \begin{equation}
        \liminf_{i\to\infty}P_{u'_{\theta}}^{(\lambda(i))}(\mathcal{C})\geq b-\mathcal{E}_{\mathrm{PDFM}}.
     \end{equation}

\end{proof}

\section{Convergence under Unbiased Constraint Estimation}
\label{app:unbiased_constraint_estimation}

Assumption~\ref{ass:exactdualgrad} simplifies the analysis by treating the empirical estimate of the constraint satisfaction gap as exact at every dual update. In practice, however, $P_{u_\theta}(\mathcal{C})$ is estimated using a finite number of generated samples and is therefore stochastic. We relax Assumption~\ref{ass:exactdualgrad} by requiring the Monte Carlo estimator to be conditionally unbiased given the current dual variable.

\begin{assumption}[Unbiased Monte Carlo estimation]
\label{ass:unbiasedMC}
For each fixed dual variable $\lambda$, let $u_{\theta}'^{(\lambda)}$ denote the corresponding velocity field. The dual gradient $
    \nabla_\lambda L\!\left(
        p_{u_{\theta}'}^{(\lambda)},\lambda
    \right)
    =
    b-P_{u_{\theta}'}^{(\lambda)}(\mathcal{C})
$
can be estimated via Monte Carlo approximation. Let $\mathrm{grad}(\lambda)$ denote the resulting estimator. At iteration $i$, define $\mathcal{F}_i=(\lambda(i),u_{\theta}'^{(\lambda(i))})$, which consists of the current dual variable and the corresponding velocity field. We assume that the estimator is conditionally unbiased given $\mathcal{F}_i$, i.e.,
\begin{equation}
    \mathbb{E}\left[
        \mathrm{grad}(\lambda(i))
        \,\middle|\,
        \mathcal{F}_i
    \right]
    =
    \nabla_\lambda L(
        p_{u_{\theta}'}^{(\lambda(i))},\lambda(i)
    )
    =
    b-P_{u_{\theta}'}^{(\lambda(i))}(\mathcal{C}).
\end{equation}
\end{assumption}

\begin{lemma}\label{lemma:pracricalPDconvergence2}
    Under Assumption~\ref{ass:unbiasedMC} instead of Assumption~\ref{ass:exactdualgrad}, with $\forall i, 0\leq \eta_i\leq\frac{2}{kw}$, $\sum_{i=1}^\infty\eta_i=\infty, \sum_{i=1}^\infty\eta_i^2<\infty $,
\begin{equation}
    \displaystyle \limsup_{i\rightarrow \infty}\mathbb{E}[|\lambda(i)-{\lambda}^*|^2]\leq \frac{4e_{\max}^2}{k^2} .
\end{equation}
\end{lemma}

\begin{proof}
Note that $\lambda(0)=\lambda'(0)=0$, and in Alg.~\ref{alg:pd_practical}, $\lambda'=\lambda/w$ is updated in the way $\lambda'(i+1)\gets[\lambda'(i) +\eta_{i}\mathrm{grad}(\lambda(i))]_{+}$. The update in $\lambda$ is equivalent to  
\begin{equation}
    \lambda(i+1) = \left[\lambda(i)+w\eta_i\mathrm{grad}(\lambda(i))\right]_+.
\end{equation}
Therefore, 
    \begin{equation}
    \begin{alignedat}{2}
        |\lambda(i+1)-{\lambda}^*|^2&=|[\lambda (i)+w\eta_{i}\mathrm{grad}(\lambda(i))]_{+}-{\lambda}^*|^2\\
        &\leq|\lambda(i) +w\eta_{i}\text{grad}(\lambda(i))-{\lambda}^*|^2
        \\&=|\lambda(i)-{\lambda}^*|^2+2w\eta_i(\lambda(i)-{\lambda}^*)\mathrm{grad}(\lambda(i))+(w\eta_i)^2\mathrm{grad}(\lambda(i))^2
    \end{alignedat}
\end{equation}

Take the expectation conditioning on $\mathcal{F}_i$. After applying the bounds from the proof of Lemma~F.3, we condition the resulting inequality on $\lambda(i)$ using the tower property to obtain the recurrence below.
\begin{equation}\label{eq:expectlambdagap}
    \begin{alignedat}{2}
        & \mathbb{E}[|\lambda(i+1)-{\lambda}^*|^2|\mathcal{F}_i]\\
       &\leq\mathbb{E}\left[|\lambda(i)-{\lambda}^*|^2+2w\eta_i(\lambda(i)-{\lambda}^*)\text{grad}(\lambda(i))+(w\eta_i)^2\text{grad}(\lambda(i))^2|\mathcal{F}_i\right]\\
       & = |\lambda(i)-{\lambda}^*|^2+2w\eta_i(\lambda(i)-{\lambda}^*)\mathbb{E}[\text{grad}(\lambda(i))|\mathcal{F}_i]+(w\eta_i)^2\mathbb{E}[\text{grad}(\lambda(i))^2|\mathcal{F}_i]\\
   & = |\lambda(i)-{\lambda}^*|^2+2w\eta_i(\lambda(i)-{\lambda}^*)\nabla_{\lambda}L(p^{(\lambda(i))}_{u'_\theta}, \lambda(i))+(w\eta_i)^2\mathbb{E}[\text{grad}(\lambda(i))^2|\mathcal{F}_i] \\&\quad\text{(According to Assumption~\ref{ass:unbiasedMC})}
   \\& \leq |\lambda(i)-{\lambda}^*|^2+2w\eta_i(\lambda(i)-{\lambda}^*)\nabla_{\lambda}L(p^{(\lambda(i))}_{u'_\theta}, \lambda(i))+(w\eta_i)^2
   \\&\quad\text{($\text{grad}(\lambda(i))^2\leq 1$ since it is a Monte Carlo estimate of $\left(b-P_{u_{\theta}'}^{(\lambda)}(\mathcal{C})\right)^2\leq \max\{b^2,(b-1)^2\}$)}
    \end{alignedat}
\end{equation}

Note that the last line of \eqref{eq:expectlambdagap} arrives at the fifth line of \eqref{eq:exactlambdagap}. With the same procedure as in the proof of Lemma~\ref{lemma:pracricalPDconvergence}, we first obtain the recursion conditioned on $\mathcal{F}_i$. Since $\lambda(i)$ is a component of $\mathcal{F}_i$ and the right-hand side depends only on $\lambda(i)$, conditioning this recursion on $\lambda(i)$ and applying the tower property gives
\begin{equation}\label{eq:conditional_recursion}
    \mathbb{E}\left[|\lambda(i+1)-{\lambda}^*|^2|\lambda(i)\right]-\frac{4e_{\max}^2}{k^2}\leq (1-\frac{k}{2}w\eta_i)(\left|\lambda(i)-{\lambda}^*\right|^2-\frac{4e_{\max}^2}{k^2})+(w\eta_i)^2.
\end{equation}
Taking expectation on both sides of \eqref{eq:conditional_recursion} and applying the tower property gives
\begin{equation}
\mathbb{E}\left[
    \mathbb{E}\left[
        |\lambda(i+1)-\lambda^*|^2
        \,\middle|\,
        \lambda(i)
    \right]
\right]
-\frac{4e_{\max}^2}{k^2}
\leq
\left(1-\frac{k}{2}w\eta_i\right)
\left(
    \mathbb{E}\left[
        |\lambda(i)-\lambda^*|^2
    \right]
    -\frac{4e_{\max}^2}{k^2}
\right)
+(w\eta_i)^2.    
\end{equation}
By the law of total expectation,
\begin{equation}
    \mathbb{E}\left[
        \mathbb{E}\left[
            |\lambda(i+1)-\lambda^*|^2
            \,\middle|\,
            \lambda(i)
        \right]
    \right]
    =
    \mathbb{E}\left[
        |\lambda(i+1)-\lambda^*|^2
    \right].
\end{equation}
Therefore,
\begin{equation}
\mathbb{E}\left[
    |\lambda(i+1)-\lambda^*|^2
\right]
-\frac{4e_{\max}^2}{k^2}
\leq
\left(1-\frac{k}{2}w\eta_i\right)
\left(
    \mathbb{E}\left[
        |\lambda(i)-\lambda^*|^2
    \right]
    -\frac{4e_{\max}^2}{k^2}
\right)
+(w\eta_i)^2.
\label{eq:unconditional_recursion}
\end{equation}

Unrolling this recursion yields
\begin{equation} \label{eq:recur}
\resizebox{0.98\textwidth}{!}{$\displaystyle
     \mathbb{E}[|\lambda(i+1)-{\lambda}^*|^2]-\frac{4e_{\max}^2}{k^2}\leq \left(\mathbb{E}[|\lambda(0)-\lambda^*|^2]-\frac{4e_{\max}^2}{k^2}\right)\prod_{j=0}^i (1-\frac{k}{2}w\eta_j)+\sum_{n=0}^i\left(\prod_{j=n+1}^i(1-\frac{k}{2}w\eta_j)\right)(w\eta_n)^2,$}
\end{equation}
and we define the empty product $\prod_{j=i+1}^i(\cdot) = 1$ for the last term.

With $\forall i, 0\leq \eta_i\leq\frac{2}{kw}$, $\sum_{i=1}^\infty\eta_i=\infty, \sum_{i=1}^\infty\eta_i^2<\infty $, \eqref{eq:recur} implies that 
\begin{equation}
    \displaystyle \limsup_{i\rightarrow \infty}\mathbb{E}[|\lambda(i)-{\lambda}^*|^2]\leq \frac{4e_{\max}^2}{k^2} .
\end{equation}
\end{proof}

\begin{theorem} \label{theorem:bound3} (Theorem~\ref{theorem:bound2} under unbiased Monte Carlo assumption instead of exact Monte Carlo assumption) Under  Assumptions~\ref{ass:feasibility}, \ref{ass:esterror}- \ref{ass:densityrange}, \ref{ass:lambdamax}, \ref{ass:esterrorGD}, and \ref{ass:unbiasedMC}, with $\forall i, 0\leq \eta_i\leq\frac{2}{kw}$, $\sum_{i=1}^\infty\eta_i=\infty, \sum_{i=1}^\infty\eta_i^2<\infty $,
\begin{equation}
    \limsup_{i\to\infty} \mathbb{E}\left[ \left\|{P_{u'_{\theta}}^{(\lambda(i))}}-P^{(\lambda^*)} \right\|_{\mathrm{TV}}\right] \leq \mathcal{E}_{\mathrm{PDFM}}, \quad\liminf_{i\to\infty}\mathbb{E}\left[P_{u'_{\theta}}^{(\lambda(i))}(\mathcal{C})\right]\geq b-\mathcal{E}_{\mathrm{PDFM}},
    \end{equation}
    \begin{equation}
    \mathcal{E}_{\mathrm{PDFM}} =  e_{\max}+\frac{2e_{\max}}{k\left((1-Q({\mathcal{C}}))^{1/3}+Q({\mathcal{C}})^{1/3}\right)^3} ,\quad e(\lambda):=\sqrt{\frac{\varepsilon+\mu(\lambda) (A_1\nu + A_2\nu^2)}{2(\mu(\lambda) - \lambda)}}, 
\end{equation}
     in which $\mu(\lambda) = \frac{1}{2}\left(\lambda+1+\sqrt{(\lambda+1)^2-4\lambda(1-Q(\mathcal{C}))}\right)$ and $e_{\max} =\max_{\lambda\in[0,\lambda_{\max}]} e(\lambda)$.
\end{theorem}

 \begin{proof}
By Lemma~\ref{lemma:TVgapthetai} and Lemma~\ref{lemma:lambdabound},
\begin{equation}
    \begin{alignedat}{2}
              \left\|{P_{u'_{\theta}}^{(\lambda(i))}}-P^{(\lambda^*)} \right\|_{\mathrm{TV}}&\leq \left\|{P_{u'_{\theta}}^{(\lambda(i))}}-P^{(\lambda(i))} \right\|_{\mathrm{TV}}+\left\|{P^{(\lambda(i))}}-P^{(\lambda^*)} \right\|_{\mathrm{TV}}\\
          &\leq e(\lambda(i))+\frac{1}{\left((1-Q({\mathcal{C}}))^{1/3}+Q({\mathcal{C}})^{1/3}\right)^3}|{\lambda(i)}-{\lambda^*}|.
    \end{alignedat}
\end{equation}
Therefore,
\begin{equation}
\begin{aligned}
    \mathbb{E}\left[\left\|{P_{u'_{\theta}}^{(\lambda(i))}}-P^{(\lambda^*)} \right\|_{\mathrm{TV}}\right]&\leq \mathbb{E}[e(\lambda(i))]+\frac{1}{\left((1-Q({\mathcal{C}}))^{1/3}+Q({\mathcal{C}})^{1/3}\right)^3}\mathbb{E}\left[|{\lambda(i)}-{\lambda^*}|\right]\\&
    \leq \mathbb{E}[e(\lambda(i))]+\frac{\sqrt{\mathbb{E}\left[|{\lambda(i)}-{\lambda^*}|^2\right]}}{\left((1-Q({\mathcal{C}}))^{1/3}+Q({\mathcal{C}})^{1/3}\right)^3}
    \end{aligned}
\end{equation}

By $\mathbb{E}[e(\lambda(i))]\leq e_{\max}$ and Lemma~\ref{lemma:pracricalPDconvergence2},
\begin{equation}
    \limsup_{i\to\infty} \mathbb{E}\left[ \left\|{P_{u'_{\theta}}^{(\lambda(i))}}-P^{(\lambda^*)} \right\|_{\mathrm{TV}}\right]\leq  \mathcal{E}_{\mathrm{PDFM}},\quad \mathcal{E}_{\mathrm{PDFM}} =e_{\max}+\frac{2e_{\max}}{k\left((1-Q({\mathcal{C}}))^{1/3}+Q({\mathcal{C}})^{1/3}\right)^3} .
\end{equation}
and by Lemma~\ref{lemma:TVandconstraint} the expected constraint satisfaction rate satisfies
     \begin{equation}
        \liminf_{i\to\infty}\mathbb{E}\left[P_{u'_{\theta}}^{(\lambda(i))}(\mathcal{C})\right]\geq b-\mathcal{E}_{\mathrm{PDFM}}.
     \end{equation}

\end{proof}

\section{Experiment Details }\label{app:exp}
All experiments are run on an NVIDIA A100 GPU with $80$ GB of memory.

We implement standard FM using independent random coupling and linear interpolants, following Algorithm 2 (Simplified Conditional Flow Matching) in \cite{tong2024improving}. The network architecture for each experiment is specified in the corresponding subsection. During FM training, the learning rate decays from $10^{-4}$ to $10^{-5}$ during training. During subsequent primal dual training, we use the settings specified below.

\subsection{Illustrative $2$-D Constrained Generation ($Q(\mathcal{C})< b$)}\label{app:toy2d}

\paragraph{Network architectures:}
The velocity field $u_\theta(x_t,t)$ is parameterized by a fully connected network taking $(x_t,t)\in\mathbb{R}^3$ as input. It contains five hidden layers of width $64$ with SELU activations and a two-dimensional linear output representing $u_\theta(x_t,t)\in\mathbb{R}^2$. The feasibility predictor $f_\phi(x_t,t)$ uses the same hidden architecture with a scalar output. A sigmoid function is applied to the output logit to obtain the predicted probability of terminal constraint satisfaction.

\paragraph{Hyperparameters:} learning rate for the velocity field $\eta^{(\theta)} = 1\times 10^{-5}$, learning rate for the classifier $\eta^{(\phi)} = 1\times 10^{-4}$, dual variable learning rate $\eta_i = 500$. Large $\eta_i$ is applied because a high constraint satisfaction target corresponds to a large optimal dual variable. A larger update step therefore accelerates convergence.

The visualizations of the clean distribution, $Q$, $P^{(\lambda^*)}$, the distributions induced by baselines (1), (2), and PDFM are shown in Fig.~\ref{fig:2Dtoy}.

\begin{figure}[ht]
    \centering
   \resizebox{0.8\textwidth}{!}{
    \begin{tikzpicture}

    \def\imgw{0.3\textwidth}
    \def\imgh{0.3\textwidth}
    \def\xsep{0.35cm}
    \def\ysep{0.55cm}
    \def\xlabsep{-0.08cm}
    \def\ylabsep{-0.15cm}

    \pgfmathsetlengthmacro{\xstep}{\imgw + \xsep}
    \pgfmathsetlengthmacro{\ystep}{\imgh + \ysep}

    \node at (0*\xstep,0) {\includegraphics[width=\imgw,height=\imgh]{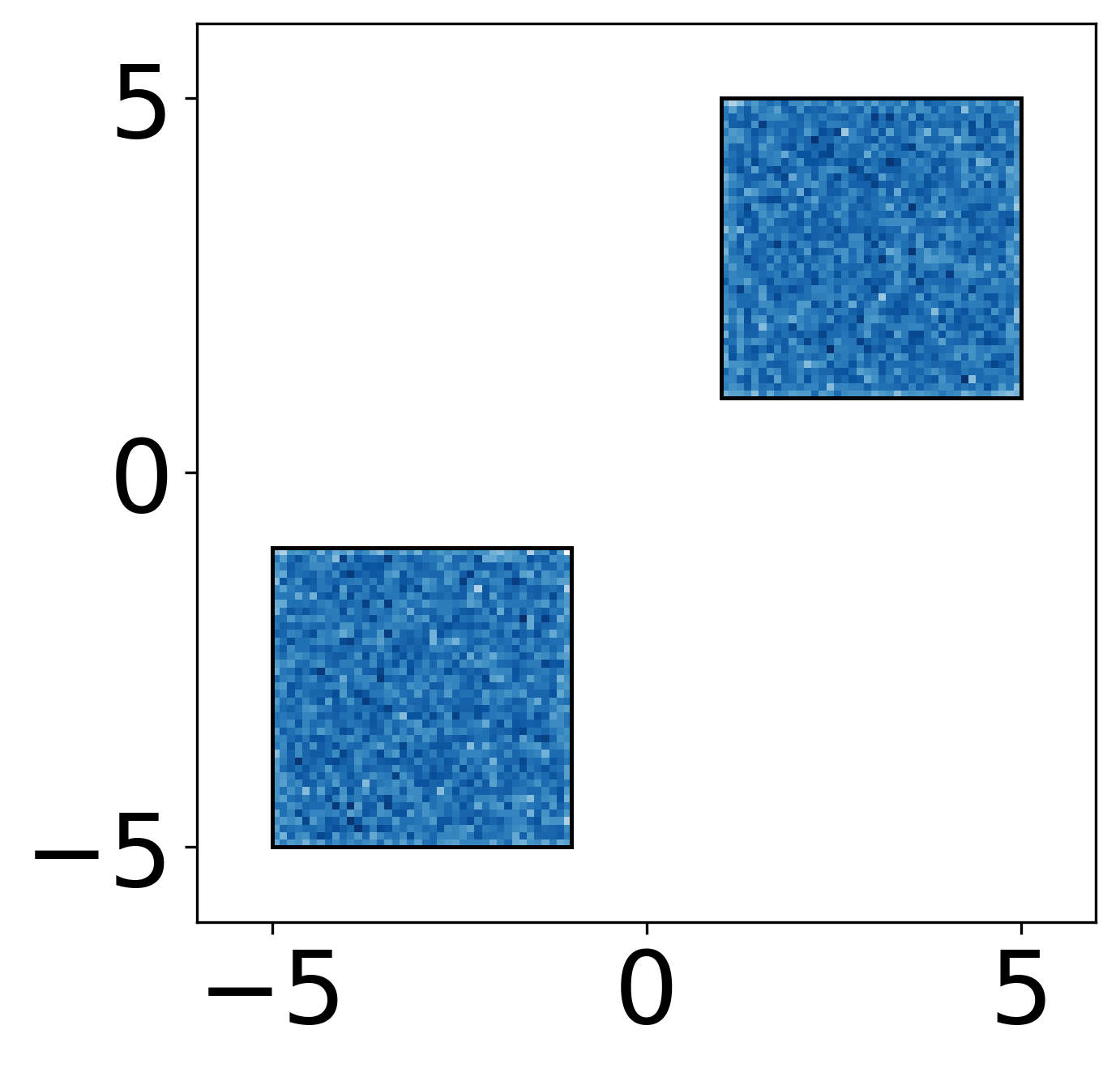}};
    \node[anchor=north, font=\small] at (0*\xstep,{0 - 0.5*\imgh - \xlabsep}) {(a) Clean distribution};

    \node at (1*\xstep,0) {\includegraphics[width=\imgw,height=\imgh]{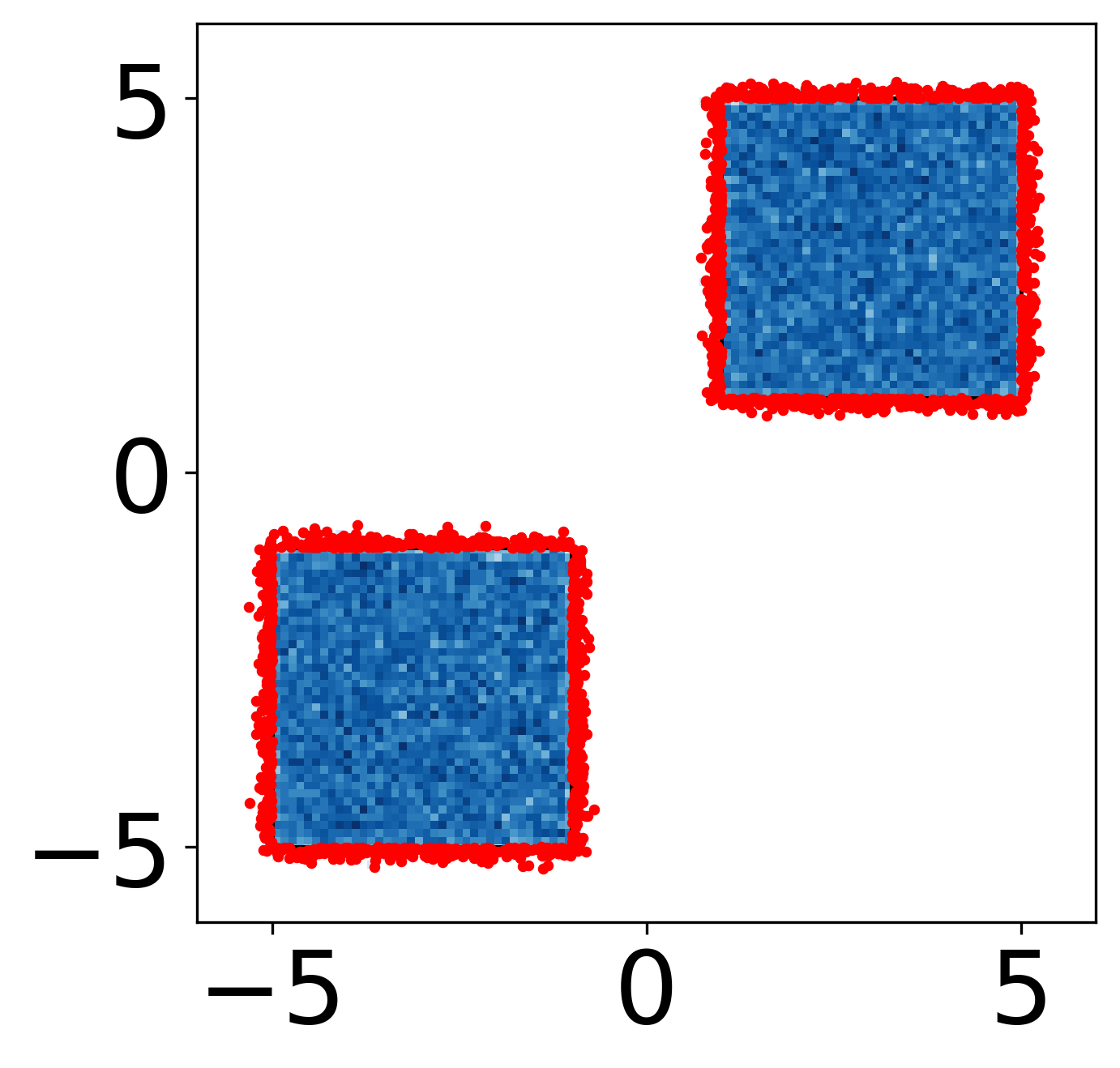}};
    \node[anchor=north, font=\small] at (1*\xstep,{0 - 0.5*\imgh - \xlabsep}) {(b) Training distribution $Q$};
        \node[anchor=north, font=\small] at (1*\xstep,{0 - 0.6*\imgh - \xlabsep}) {$Q(\mathcal{C})\approx96\%$};

    \node at (2*\xstep,0) {\includegraphics[width=\imgw,height=\imgh]{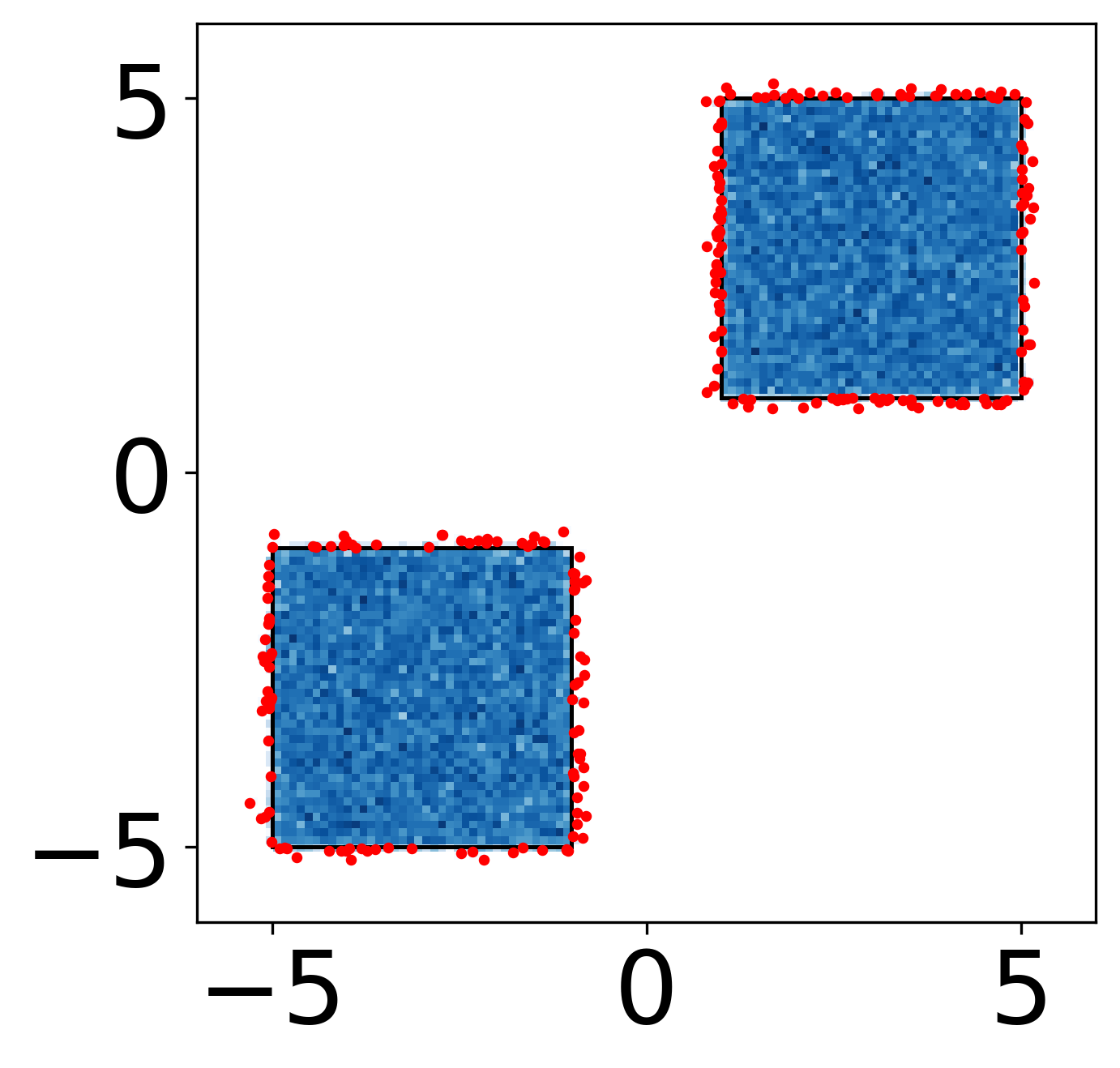}};
    \node[anchor=north, font=\small] at (2*\xstep,{0 - 0.5*\imgh - \xlabsep}) {(c) Optimal solution $P^{(\lambda^*)}$,};
    \node[anchor=north, font=\small] at (2*\xstep,{0 - 0.6*\imgh - \xlabsep}) {$ P^{(\lambda^*)}(\mathcal{C}) = 99.5\%$};

    \node at (0*\xstep,-1.1*\ystep) {\includegraphics[width=\imgw,height=\imgh]{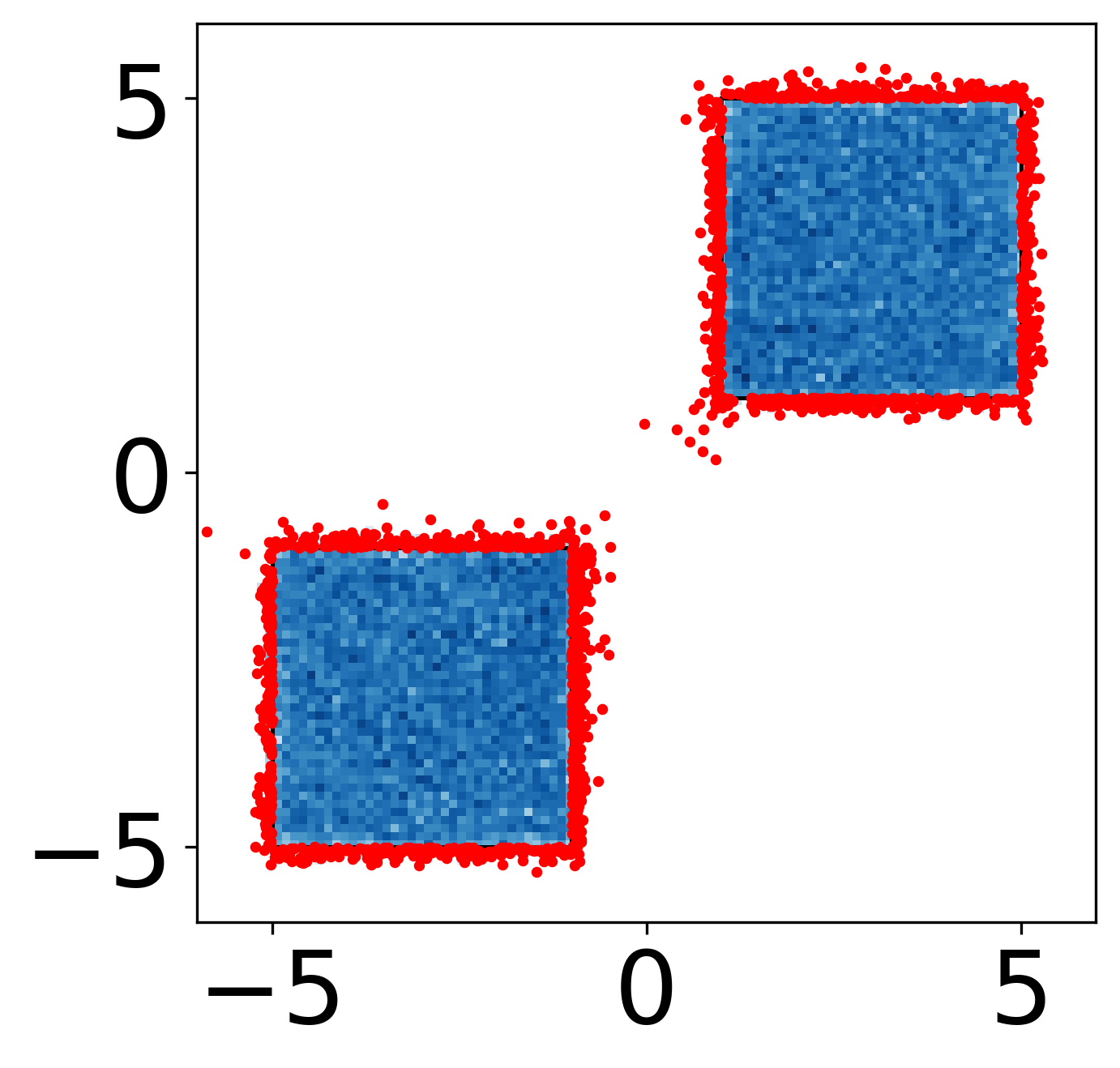}};
    \node[anchor=north, font=\small] at (0*\xstep,{-\ystep - 0.6*\imgh - \xlabsep}) {(d) FM trained on $Q$,};
        \node[anchor=north, font=\small] at (0*\xstep,{-\ystep - 0.7*\imgh - \xlabsep}) {$P(\mathcal{C}) =96.37\%$,};
        \node[anchor=north, font=\small] at (0*\xstep,{-\ystep - 0.8*\imgh - \xlabsep}) {$\mathrm{SWD} = 0.1384$};

    \node at (1*\xstep,-1.1*\ystep) {\includegraphics[width=\imgw,height=\imgh]{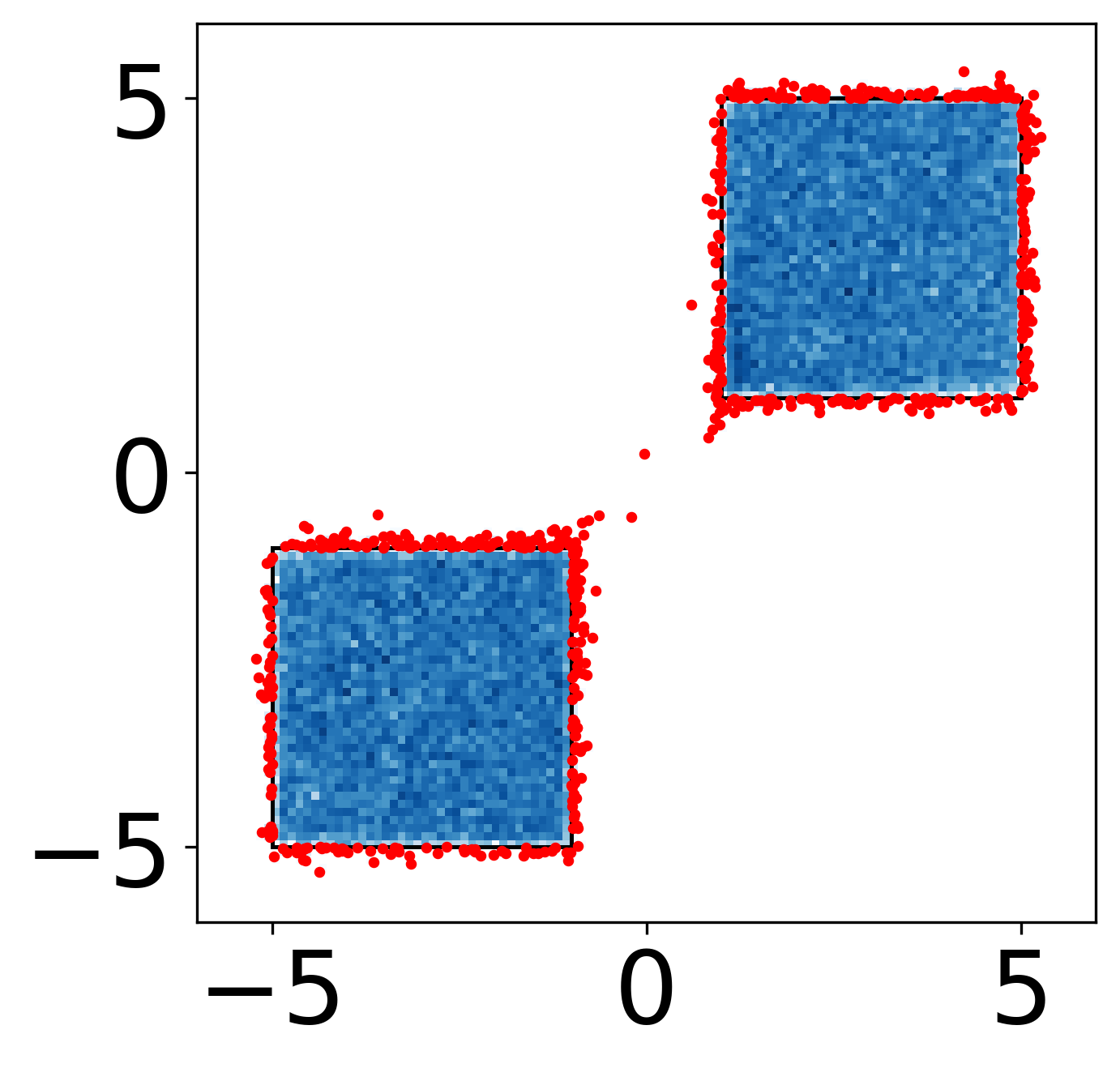}};
    \node[anchor=north, font=\small] at (1*\xstep,{-\ystep - 0.6*\imgh - \xlabsep}) {(e) FM trained on $Q'$,};
    \node[anchor=north, font=\small] at (1*\xstep,{-\ystep - 0.7*\imgh - \xlabsep}) {$P(\mathcal{C}) =98.74\%$,};
     \node[anchor=north, font=\small] at (1*\xstep,{-\ystep - 0.8*\imgh - \xlabsep}) {$\mathrm{SWD} = 0.1156$};

    \node at (2*\xstep,-1.1*\ystep) {\includegraphics[width=\imgw,height=\imgh]{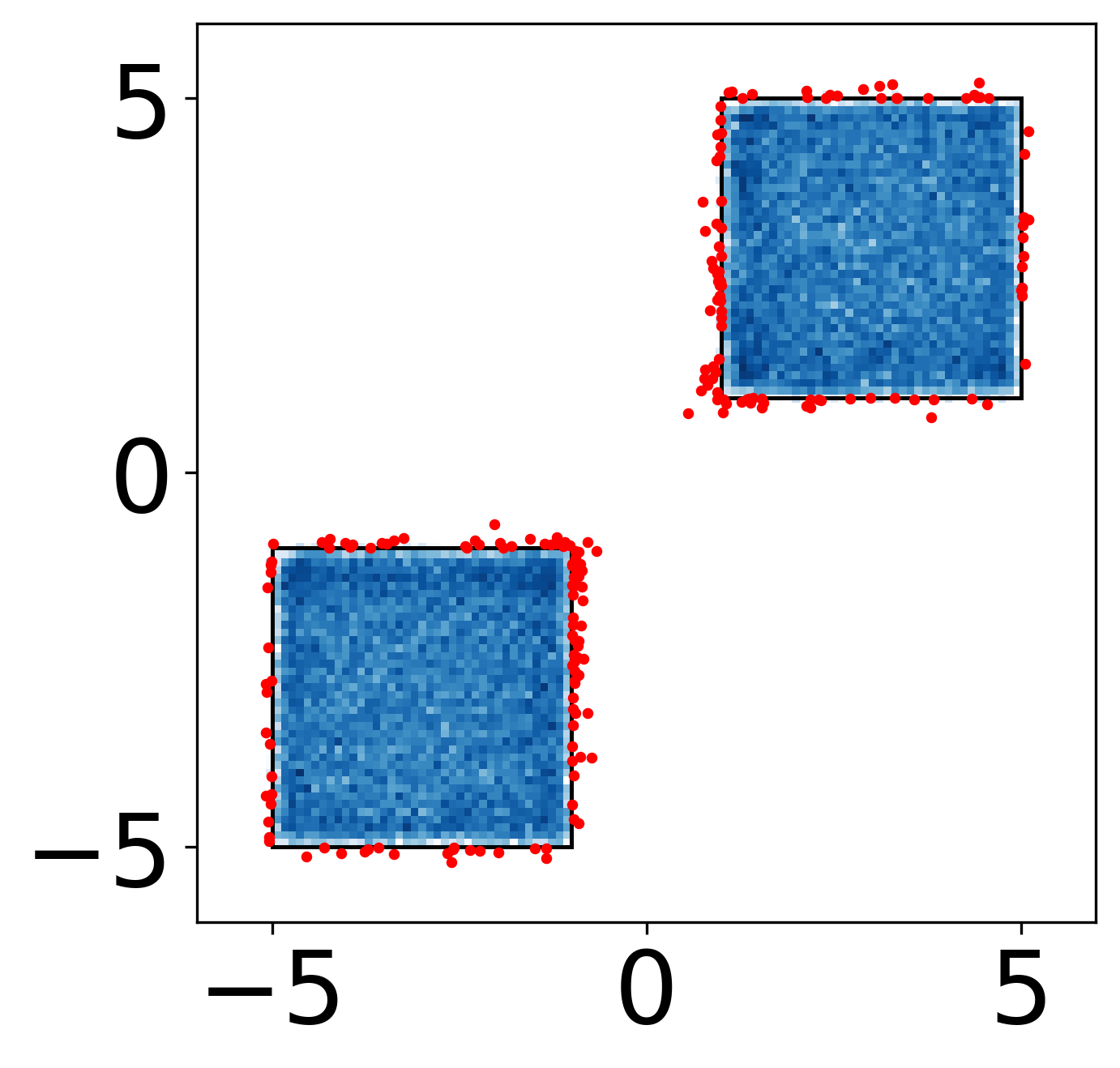}};
    \node[anchor=north, font=\small] at (2*\xstep,{-\ystep - 0.6*\imgh - \xlabsep}) {(f) PDFM trained on $Q$ ($b=99.5\%$),};
    \node[anchor=north, font=\small] at (2*\xstep,{-\ystep - 0.7*\imgh - \xlabsep}){$P(\mathcal{C})=99.51\%$,};
        \node[anchor=north, font=\small] at (2*\xstep,{-\ystep - 0.8*\imgh - \xlabsep}){$\mathrm{SWD} = 0.1096$};

    \end{tikzpicture}}
    \caption{Histogram plots of $50000$ samples from the compared distributions.}
    \label{fig:2Dtoy}
\end{figure}

\subsection{Constrained Generation with $\ell_2$-Ball Constraints ($Q(\mathcal{C})\geq b$)}

\paragraph{Network architectures:}
For the $\ell_2$-ball experiments, where $d\in\{8,20\}$, the velocity network separately embeds $x_t\in\mathbb{R}^d$ and $t\in\mathbb{R}$. The time encoder contains two linear layers, while the state encoder contains three linear layers; both produce $64$-dimensional representations. These representations are added and passed through a SELU activation, followed by a three-layer prediction network with hidden width $64$ and SELU activations. Its output is $u_\theta(x_t,t)\in\mathbb{R}^d$. The feasibility predictor $f_\phi(x_t,t)$ uses the same architecture with separate parameters and a scalar output logit. A sigmoid function converts this logit into the predicted probability of terminal constraint satisfaction.

\paragraph{Hyperparameters:} learning rate for the velocity field $\eta^{(\theta)} = 1\times 10^{-5}$, learning rate for the classifier $\eta^{(\phi)} = 1\times 10^{-4}$, dual variable learning rate $\eta_i = 500$. Large $\eta_i$ is applied because a high constraint satisfaction target corresponds to a large optimal dual variable. A larger update step therefore accelerates convergence.

\begin{table}[t]
\centering
\caption{Results on $8$-D and $20$-D constrained generation tasks with standard deviations. $P(\mathcal{C})$ denotes the mean constrained-satisfaction rate of $10000$ samples, averaged over $100$ generation trials. SWD is computed with $10000$ generated samples and samples from $Q$, averaged over $100$ generation trials.  }
\newcommand{\meanstd}[2]{%
    \shortstack{$#1$\\[-1pt]{\scriptsize $\pm\,#2$}}%
}
\scalebox{0.9}{
\begin{tabular}{cc|ccccccc}
\hline
& & FM & MDM & NAMM & PDM & PCFM & FM-RE & PDFM \\
\hline
\multirow{2}{*}{$8$-D $\ell_2$ ball}
& $P(\mathcal C)$ $(\uparrow)$
& \meanstd{97.633\%}{0.143\%}
& \meanstd{100\%}{0}
& \meanstd{99.986\%}{0.008\%}
& \meanstd{100\%}{0}
& \meanstd{100\%}{0}
& \meanstd{99.923\%}{0.025\%}
& \meanstd{99.906\%}{0.026\%} \\
& SWD $(\downarrow)$
& \meanstd{0.0193}{0.0029}
& \meanstd{0.0217}{0.0030}
& \meanstd{0.0228}{0.0037}
& \meanstd{0.0200}{0.0031}
& \meanstd{0.0284}{0.0042}
& \meanstd{0.0194}{0.0029}
& \meanstd{0.0172}{0.0028} \\
\hline
\multirow{2}{*}{$20$-D $\ell_2$ ball}
& $P(\mathcal C)$ $(\uparrow)$
& \meanstd{90.918\%}{0.318\%}
& \meanstd{100\%}{0}
& \meanstd{99.711\%}{0.053\%}
& \meanstd{100\%}{0}
& \meanstd{100\%}{0}
& \meanstd{99.748\%}{0.052\%}
& \meanstd{99.755\%}{0.0051\%} \\
& SWD $(\downarrow)$
& \meanstd{0.0087}{0.0008}
& \meanstd{0.0144}{0.0012}
& \meanstd{0.0205}{0.0008}
& \meanstd{0.0257}{0.0018}
& \meanstd{0.0148}{0.0009}
& \meanstd{0.0132}{0.0010}
& \meanstd{0.0101}{0.0007} \\
\hline
\end{tabular}}
\label{tb:820l2ballwstd}
\end{table}

\subsection{Unconditional PDE-Constrained Generation of Spatiotemporal Dynamics}\label{app:PDE}

\paragraph{Dataset construction and residual computation:}
We construct a dataset of Kuramoto--Sivashinsky dynamics by numerically solving the conservative PDE with fixed solver parameters and randomly sampled truncated-Fourier initial conditions, using Exponax \citep{koehler2024apebench}. From each long trajectory, we extract non-overlapping space--time patches $x \in \mathbb{R}^{d \times d}$, where $x[:,i]$ denotes the spatial state at time index $i$. Let $F(\cdot)$ denote one application of the numerical stepper. Let $\hat{x}$ denote the generated sample, and for each $i=0,\dots,d-2$, we regard $\hat{x}[:,i]$ as the true current state, compute the next true state ${x}[:,{i+1}]=F(\hat{x}[:,i])$, and compare it with the next estimated state $\hat{x}[:,i+1]$, yielding the residual $r_i = x[:,i+1] - \hat{x}[:,{i+1}]$.
 The resulting residual sequence $\{r_i\}_{i=0}^{d-2}$ defines the size-normalized residual
 $\delta_{L_{n}}:=\left(\frac{1}{(d-1)d}\sum_{i=0}^{d-2}\|r_i\|_n^{n}\right)^{1/n} $, with the usual extension for $n=\infty$. A generated sample is defined as satisfying the constraint if $\delta_{L_{n}}< \tau$, in which 
 $\tau$ is a prescribed tolerance. In this task, we set $d=64, n=1$.

 \paragraph{Network architectures:}
The velocity network $u_\theta$ takes $x_t\in\mathbb{R}^{1\times64\times64}$ and the time $t\in\mathbb{R}$ as inputs. A $128$-dimensional sinusoidal time embedding is projected to dimension $256$, while $x_t$ is projected to $64$ channels. The network then applies five residual space-time blocks, each combining $5\times1$ spatial, $1\times5$ physical-time, and $3\times3$ joint convolutions with group normalization, SiLU activations, and time-conditioned transformations. A final $3\times3$ convolution produces $u_\theta(x_t,t)\in\mathbb{R}^{1\times64\times64}$. The feasibility predictor uses the same backbone with separate parameters, followed by global average pooling, a scalar linear output, and a sigmoid function producing $f_\phi(x_t,t)\in(0,1)$.

\paragraph{Hyperparameters:} for both cases, learning rate for the velocity field $\eta^{(\theta)} = 1\times 10^{-5}$, learning rate for the classifier $\eta^{(\phi)} = 1\times 10^{-4}$, dual variable learning rate $\eta_i = 2$. 

 \paragraph{Significance for chaotic PDE emulation.}
The Kuramoto–Sivashinsky setting is representative of generative emulation tasks for chaotic systems, where samples are drawn from the emergent chaotic attractor. High-quality emulation requires both local dynamical accuracy and statistical fidelity to that attractor. The two terms in the PDFM objective parallel this decomposition: the FM loss matches the attractor distribution, while the chance constraint enforces local consistency with the dynamics, providing a principled probabilistic alternative to the combination of next-step prediction error and statistical-fidelity losses common in deterministic emulators of chaotic systems \citep{jiang2023training, melo2026learning}.

The visualization of the generated space-time patch samples in Case 1 ($\tau = 0.033, b = 98\%$) is provided in Fig.~\ref{fig:FMsamples} and Fig.~\ref{fig:PDFMsamples}. To quantify the discrepancy between generated samples and the simulator predictions, we extract the first column of each generated sample as the initial state and use the simulator to predict the subsequent $63$ states. The comparisons between the generated samples and the simulator-predicted truth, along with the histograms of the summary statistic, are provided in Fig.~\ref{fig:case1_individual} and Fig.~\ref{fig:case2_individual}. The visualizations of the mean histogram over $2000$ samples are provided in Fig.~\ref{fig:case1_batch} and Fig.~\ref{fig:case2_batch}.

\begin{table}[t!]
\vspace{-2mm}
\centering
\caption{Comparison across the considered methods for Case 1 and Case 2, evaluated using $2000$ generated samples per method in each of $10$ trials. Results are reported as the mean and standard deviation.}\label{tab:PDEwstd}
\newcommand{\meanstdpde}[2]{%
    \shortstack{$#1$\\[-1pt]{\scriptsize $\pm\,#2$}}%
}
\resizebox{\textwidth}{!}{
\begin{tabular}{c|ccc|cccc}
\hline
\multirow{2}{*}{Method}
& \multicolumn{3}{c|}{Case 1 $(Q(\mathcal C)=1)$}
& \multicolumn{4}{c}{Case 2 $(Q(\mathcal C)\approx 39.47\%)$} \\
\cline{2-8}
& $P(\mathcal C)$ $(\uparrow)$
& Avg $\delta_{L_n}$ $(\downarrow)$
& HE $(\downarrow)$
& $P(\mathcal C)$ $(\uparrow)$
& Avg $\delta_{L_n}$ $(\downarrow)$
& $\mathrm{HE}_{\mathrm{clean}}$ $(\downarrow)$
& $\mathrm{HE}_{\mathrm{deg}}$ $(\downarrow)$ \\
\hline
FM ($Q$)
& \meanstdpde{14.19\%}{0.732\%}
& \meanstdpde{0.0409}{0.0002}
& \meanstdpde{(0.0087,0.0453)}{(0.0002,0.0005)}
& \meanstdpde{9.82\%}{0.767\%}
& \meanstdpde{0.0622}{0.0005}
& \meanstdpde{(0.0094,0.0465)}{(0.0002, 0.0004)}
& \meanstdpde{(0.0086,0.0565)}{(0.0002, 0.0005)} \\
FM ($Q'$)
& N/A
& N/A
& N/A
& \meanstdpde{33.24\%}{0.470\%}
& \meanstdpde{0.0539}{0.0001}
& \meanstdpde{(0.0135,0.0477)}{(0.0002,0.0004)}
& \meanstdpde{(0.0114,0.0489)}{(0.0002,0.0004)} \\
PCFM
& \meanstdpde{6.03\%}{0.383\%}
& \meanstdpde{0.0484}{0.0001}
& \meanstdpde{(0.0156,0.1547)}{(0.0002,0.0008)}
& \meanstdpde{0.06\%}{0.054\%}
& \meanstdpde{0.1022}{0.0010}
& \meanstdpde{(0.0173,0.0729)}{(0.0002,0.0005)}
& \meanstdpde{(0.0186,0.0823)}{(0.0002,0.0005)} \\
FM-RE
& \meanstdpde{43.17\%}{1.496\%}
& \meanstdpde{0.0339}{0.0002}
& \meanstdpde{(0.0117,0.0329)}{(0.0002,0.0003)}
& \meanstdpde{46.33\%}{1.007\%}
& \meanstdpde{0.0513}{0.0003}
& \meanstdpde{(0.0118,0.0558)}{(0.0002, 0.0004)}
& \meanstdpde{(0.0097,0.0570)}{(0.0002, 0.0004)} \\
PDFM
& \meanstdpde{97.74\%}{0.531\%}
& \meanstdpde{0.0235}{0.0001}
& \meanstdpde{(0.0107,0.0269)}{(0.0003,0.0003)}
& \meanstdpde{90.48\%}{0.525\%}
& \meanstdpde{0.0406}{0.0002}
& \meanstdpde{(0.0118,0.0435)}{(0.0002, 0.0004)}
& \meanstdpde{(0.0103,0.0457)}{(0.0002, 0.0004)} \\
\hline
\end{tabular}}
\vspace{-4mm}
\end{table}

\subsection{SOCOFing Fingerprint Generation with  Ridge-Connectivity Constraint}\label{app:fing}

We consider the task of generating fingerprint images with a ridge-connectivity constraint, using the Sokoto Coventry Fingerprint Dataset (SOCOFing) \citep{shehu2018detection} for training. SOCOFing is a public grayscale fingerprint image dataset containing real fingerprint samples and synthetically altered variants. It includes fingerprints collected from multiple subjects and provides altered versions with varying levels of distortion.

Collected fingerprint images and the altered variants can be corrupted by sensor noise, partial occlusion, pressure variation, or preprocessing artifacts, which may break continuous ridge structures into disconnected fragments. We therefore consider a ridge-connectivity constraint, which requires the main fingerprint ridges to remain sufficiently continuous; satisfying this constraint indicates that the image preserves coherent ridge structure.

Concretely, we evaluate this constraint using a non-differentiable image-processing oracle. For each grayscale fingerprint image, we first apply Otsu thresholding to separate ridge pixels from the background, treating ridges as the foreground. We then skeletonize the foreground ridge map into a one-pixel-wide representation and count the number of connected components in the skeletonized image. A sample is considered feasible if the number of connected ridge components is below the prescribed threshold $50$. We use the Real and Altered-Easy partitions of the SOCOFing as the training set. Among the $23931$ training samples, $16235$ of them satisfy this ridge-connectivity constraint, indicating $Q(\mathcal{C})\approx 67.84\%$. The target constraint satisfaction rate is set to $b=99\%$.

\paragraph{Network architectures:}
The velocity network takes $x_t\in\mathbb{R}^{1\times96\times96}$ and $t\in\mathbb{R}$ as inputs. It uses a four-level U-Net with channel dimensions $64$, $128$, $256$, and $512$, three downsampling and three upsampling layers, skip connections, and two residual blocks at the bottleneck. A $256$-dimensional sinusoidal embedding of $t$ is projected and added to the features in each residual block. A final $3\times3$ convolution produces $u_\theta(x_t,t)\in\mathbb{R}^{1\times96\times96}$. The feasibility predictor uses a three-level encoder with channel dimensions $64$, $128$, and $256$, followed by two residual blocks, global average pooling, and a sigmoid output $f_\phi(x_t,t)\in(0,1)$.

\paragraph{Hyperparameters:} learning rate for the velocity field $\eta^{(\theta)} = 1\times 10^{-5}$, learning rate for the classifier $\eta^{(\phi)} = 1\times 10^{-4}$, dual variable learning rate $\eta_i = 5$. 

For this oracle-constrained task, only constraint-aware methods compatible with membership oracle feedback are applicable. Therefore, we compare PDFM against FM-RE and the vanilla FM baseline. The performances are evaluated in two aspects: constraint satisfaction rate and negative log-likelihood (NLL). NLL is computed by treating the finite-step sampler as a sequence of invertible transformations from a base noise distribution to the data space. For each evaluation sample, the discrete updates are inverted to recover the corresponding latent noise, the log density of this latent under the base distribution is evaluated, and the log-determinant correction associated with each discrete transformation is accumulated to obtain the sample log likelihood; the NLL is then the negative average log likelihood over the evaluation set. Lower NLL indicates that the model assigns a higher probability density to the evaluated samples under this discrete likelihood. In this setting, we estimate NLL using $100$ Euler discretization steps on two evaluation sets: the full training set, denoted by $\mathcal{D_{\mathrm{train}}}$, which includes both feasible and infeasible samples, and the feasible subset $\{x\in D_{\mathrm{train}}:x\in \mathcal{C}\}$, obtained by discarding infeasible samples. We report these as $\mathrm{NLL}_{\mathrm{train}}$ and $\mathrm{NLL}_{\mathrm{feas}}$, respectively.

The statistical results in Table~\ref{tab:nll_constraint} show that while FM-RE improves the feasible rate over FM, a degraded induced distribution is observed.
PDFM successfully achieves the target constraint satisfaction rate of $99\%$, while maintaining NLLs comparable to those of FM. The visualizations of the training samples and the generated samples are provided in Fig.~\ref{fig:fingerprint}. The generated samples from FM-RE display a slightly perturbed background, indicating a degraded induced distribution. PDFM and FM generate samples that are visually close to the training samples. This further validates that PDFM is able to achieve the desired feasibility control while preserving distributional quality and visual fidelity.

\begin{table}[t]
\centering
\renewcommand{\arraystretch}{1.25}
\caption{Constraint satisfaction is evaluated using 1000 generated samples per trial. $\mathrm{NLL}_{\mathrm{train}}$ and $\mathrm{NLL}_{\mathrm{feas}}$ are evaluated using 1000 samples randomly drawn from the full training set and its feasible subset, respectively, per trial. Results are reported as the mean and standard deviation over 10 trials.}
\begin{tabular}{l|ccc}
\hline
\textbf{Method} & $P(\mathcal{C})(\uparrow)$ & $\mathrm{NLL}_{\mathrm{train}}(\downarrow)$ & $\mathrm{NLL}_{\mathrm{feas}}(\downarrow)$ \\
\hline
FM    & $72.74\%\pm 1.57\%$ & $-4.0851\pm 0.0063$  & $-4.1525\pm 0.0061$ \\
\hline
FM-RE & $88.14\% \pm 1.37\%$ & $-0.9254\pm 0.0039$ &  $-0.9368\pm 0.0057$\\
\hline
PDFM  & $99.28\%\pm 1.33\%$ & $-3.8904\pm 0.0122$ & $-4.0768\pm 0.0130$ \\
\hline
\end{tabular}
\label{tab:nll_constraint}
\end{table}

\subsection{QM9 $3$-D Molecule Generation}
\label{App:qm9}
We evaluate PDFM on $3$-D molecule generation using the QM9 dataset \citep{Ramakrishnan2014}, a standard benchmark of small organic molecules with atom types and $3$-D coordinates. This task is important for the discovery of new drug molecules. The molecules are represented as $(x,h)$, in which $x$ represents the $3$-D atom coordinates and $h$ contains atom features such as atom types and charges.
Here, the sample-wise constraint is molecule stability: a generated molecule is considered feasible if its atoms satisfy valid valency rules after bond inference, so that the generated structure is chemically stable at the atom level. In this dataset, $95.2\%$ of the samples are stable, indicating $Q(\mathcal{C})\approx 95.2\%$. Note that the stability constraint is a membership oracle constraint, without differentiable distances to the feasible set or available projections.

Existing works such as~\citep{song2023equivariant} introduce molecule-specific
architectural and training designs, including hybrid probability paths for
coordinates $x$ and atom features $h$, equivariant optimal transport for
coordinates, Hungarian--Kabsch alignment, and information-aligned paths for atom
features, achieving superior stability rates. In this section, we investigate
whether PDFM can improve constraint satisfaction using only binary
feasibility feedback, without relying on such domain-specific knowledge.

\paragraph{Network architectures.}
Following the EGNN backbone of \citet{satorras21a}, each molecular state $x_t\in\mathbb{R}^{N\times8}$ is represented as a graph with three coordinates and five atom features per node. For both FM and PDFM, the velocity network uses four time-conditioned EGNN blocks with hidden dimension $64$, where pairwise relative coordinates and distances are used to update node features and coordinates. A $64$-dimensional sinusoidal embedding of $t$ conditions each block, and the network outputs $u_\theta(x_t,t)\in\mathbb{R}^{N\times8}$. The feasibility predictor uses a separate pairwise graph network with hidden dimension $128$, a $32$-dimensional time embedding, and two interaction blocks. Pairwise messages are constructed from node features, relative coordinates, and squared distances and are aggregated by averaging. The pooled graph representation and summary statistics of pairwise distances are passed through an MLP and a sigmoid function to produce $f_\phi(x_t,t)\in(0,1)$.

\paragraph{Hyperparameters:} learning rate for the velocity field $\eta^{(\theta)} = 1\times 10^{-5}$, learning rate for the classifier $\eta^{(\phi)} = 1\times 10^{-4}$, dual variable learning rate $\eta_i = 10$. 

Following EDM~\citep{hoogeboom2022equivariant}, we use NLL as a likelihood-based measure of distributional alignment $\mathbb{E}_{x\sim Q}[-\log p_\theta(x)]$, together with
molecule stability to evaluate sample-wise chemical feasibility. Lower NLL indicates that the model assigns a larger probability density to the observed data and therefore achieves a better distributional match.
In practice, NLL is approximated with a finite-step Euler sampler by reversing the discrete updates to recover the latent noise, evaluating its prior probability, and accumulating each step’s density/volume-change correction. In this task, NLL is evaluated using $100$ discrete steps.

The visualizations of the generated samples are provided in Fig.~\ref{fig:FM_mole} and Fig.~\ref{fig:PDFM_mole}. 
Over $10$ trials, constraint satisfaction is evaluated using $1000$ generated samples per trial, while NLL is evaluated using $100$ randomly selected samples from the training set per trial. FM achieves $P(\mathcal{C})=(32.77\pm1.22)\%$ and $\mathrm{NLL}=-19.74\pm0.44$. PDFM, trained with the target $b=60\%$, achieves $P(\mathcal{C})=(48.75\pm1.56)\%$ and $\mathrm{NLL}=-20.70\pm0.25$. Results are reported as the mean and standard deviation. While PDFM is unable to achieve the target constraint satisfaction rate due to the design limitations, including the generic backbone and probability path without domain-specific knowledge, it substantially improves constraint satisfaction over the FM baseline. Moreover, the improvement is achieved without sacrificing likelihood quality: PDFM obtains a lower NLL, indicating higher distributional alignment with the training distribution. This suggests that PDFM is able to improve constraint satisfaction while preserving, or even improving, distributional fidelity without task-specific designs and with only membership oracle feedback.

\section{Discussion}\label{app:discussion}
\subsection{Why Flow Matching?}

The procedure of starting from \ref{eq:cvxoriginal} then turning to the parameterized $\bar{L}$ and solving ${L'}$ practically can also be applied to other types of generative models, such as DDPM and score-based models. For example, in DDPM with a parameterized score denoted by $s_\theta$, one can generally replace Assumption~\ref{ass:esterror} with a score estimation error in the form $\mathbb{E}_{t,x_t}\left[\|s^{{(\lambda)}}_{\theta}(x_t,t)-\nabla\log p^{(\lambda)}(x_t,t)\|_2^2\right]\leq \nu^2$, leading to a similar conclusion of bounded TV as in Lemma~\ref{lemma:TVgap1} shown in \citet{li2025odt}. As for practically computing the gradient for $P_{s_\theta}{(\mathcal{C})}$, one can apply FM-RE's \citep{FMRE} optimization framework due to DDPM's inherent stochasticity.

However, DDPM's fixed stochasticity, especially the large noise steps at the start of the sampling trajectory, together with hard binary feedback in the FM-RE framework, make the optimization hard to perform in practice. Empirical evidence shows that this approach is able to deal with simple constraints, e.g., $\ell_2$-ball constraints in Sec.~\ref{sec:l2ball}, but may induce degraded distribution when the complexity of the constraint increases. In contrast, the FM framework provides more stable sampling trajectories, thereby being able to adapt to more complex constraints while preserving the distributional fidelity with the optimization strategy in Sec.~\ref{sec:gradimp}.

\subsection{Implementation Tips and Relationship with FM-RE}

From an implementation perspective, a single PDFM iteration with fixed $\lambda' = \lambda/w$ is closely related to FM-RE, if the dual variable in PDFM is interpreted as the weight of the constraint-aware term in FM-RE. Although the procedures for estimating $\nabla_{\theta}P_{u_\theta}(\mathcal C)$
differ, both methods share the same underlying idea: they provide feedback to
$u_\theta(x_t,t)$ that is explicitly or implicitly derived from the terminal binary
feasibility feedback. Since FM-RE requires sampling complete trajectories to obtain terminal oracle feedback, a natural question is whether PDFM also requires complete trajectory
sampling during its updates.

Only partially. PDFM maintains both $u_\theta$ and $f_\phi$ during the
primal update. Training $f_\phi$ requires complete trajectory sampling, since the
terminal oracle value $y=\mathbf 1_{\mathcal C}(x_1)$ serves as the training label
for the intermediate predictions $f_\phi(x_t,t)$. However, complete trajectory
sampling is not strictly required to compute $\mathrm{grad}_\theta$ in
\eqref{eq:gradtheta}, because the feedback is provided directly by the fixed
estimator $f_\phi$. 

One subtlety is that evaluating \eqref{eq:gradtheta} still requires samples of
the intermediate $x_{i\Delta t}$. From an implementation perspective,
these states can be approximated using random couplings and linear interpolants, i.e.,
$    x_{i\Delta t}=(1-i\Delta t)x_0+i\Delta t x_1$, 
provided that the probability path does not deviate significantly, as in the PDE-constrained and molecule
generation tasks.

Note that $f_\phi$, which is fundamentally trained as a classifier, is generally
cheaper to train than the more complex velocity field $u_\theta$ and does not
require highly accurate prediction to provide useful feedback. Since PDFM does
not require complete trajectory sampling for the $\theta$-update, it has a
computational advantage over FM-RE. In single-iteration evaluations for
$i> 0$, PDFM training is approximately $2$--$3\times$ faster than
FM-RE; for example, the runtime is about $1.5$ minutes vs. $4$ minutes for
the $8$-D $\ell_2$ ball constraint, and $25$ minutes vs. $70$ minutes for
Case~1 in Sec.~\ref{sec:PDE}.
Users seeking faster training can repeat multiple $\theta$-updates between
the $\phi$-updates, provided that these $\theta$-updates do not
substantially change the constraint satisfaction behavior.
\subsection{Limitations}
The limitations of PDFM are summarized as follows:
\begin{itemize}
    \item [1. ] Since $f_\phi$ is fundamentally trained as a classifier, it requires both positive
and negative feedback. Highly imbalanced feedback, especially when the feedback is
overwhelmingly negative, meaning that most samples generated by the current
$u_\theta$ violate the constraint, can prevent $f_\phi$ from learning reliable
decision boundaries and lead to inaccurate feedback. Therefore, at the start of
iteration $i=1$ in Alg.~\ref{alg:pd_practical}, PDFM requires the current
velocity field to generate at least some feasible samples.
    \item[2. ] PDFM requires the feasible region to be sufficiently represented or reachable
under the data and model distributions. This means one cannot generally
expect PDFM to learn a near-optimal constrained distribution when $Q$ deviates significantly from  $\mathcal C$. In such cases, the FM objective can become
severely misaligned with the constraint satisfaction objective, leading to a
large optimization error in Assumption~\ref{ass:esterrorGD}.
\end{itemize}


\begin{figure}[h]
    \centering
    \includegraphics[width=0.9\linewidth]{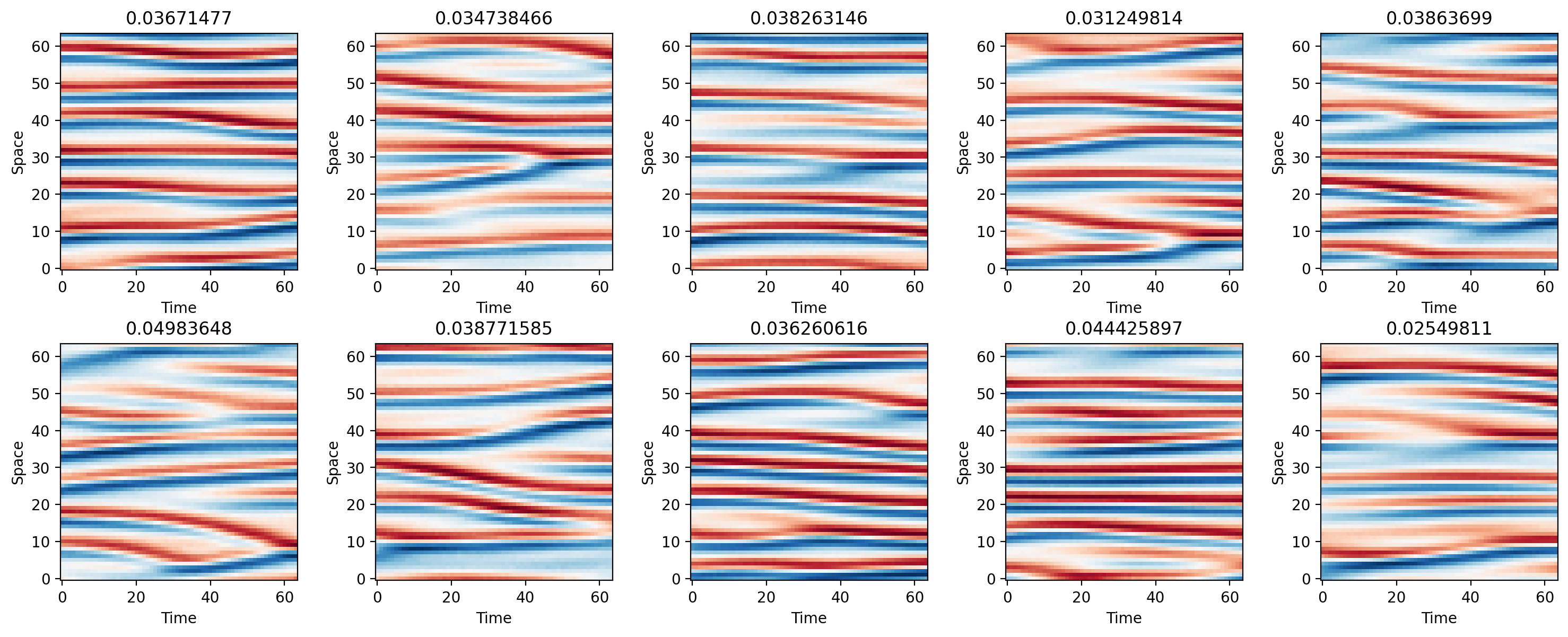}
    \caption{Samples generated by FM in Case 1 ($\tau = 0.033, b = 98\%$). The titles denote the residual $\delta_{L_n}$. The constraint satisfaction rate is $14.19\%$.}
    \label{fig:FMsamples} 
\end{figure}
\begin{figure}[h]
    \centering
    \includegraphics[width=0.9\linewidth]{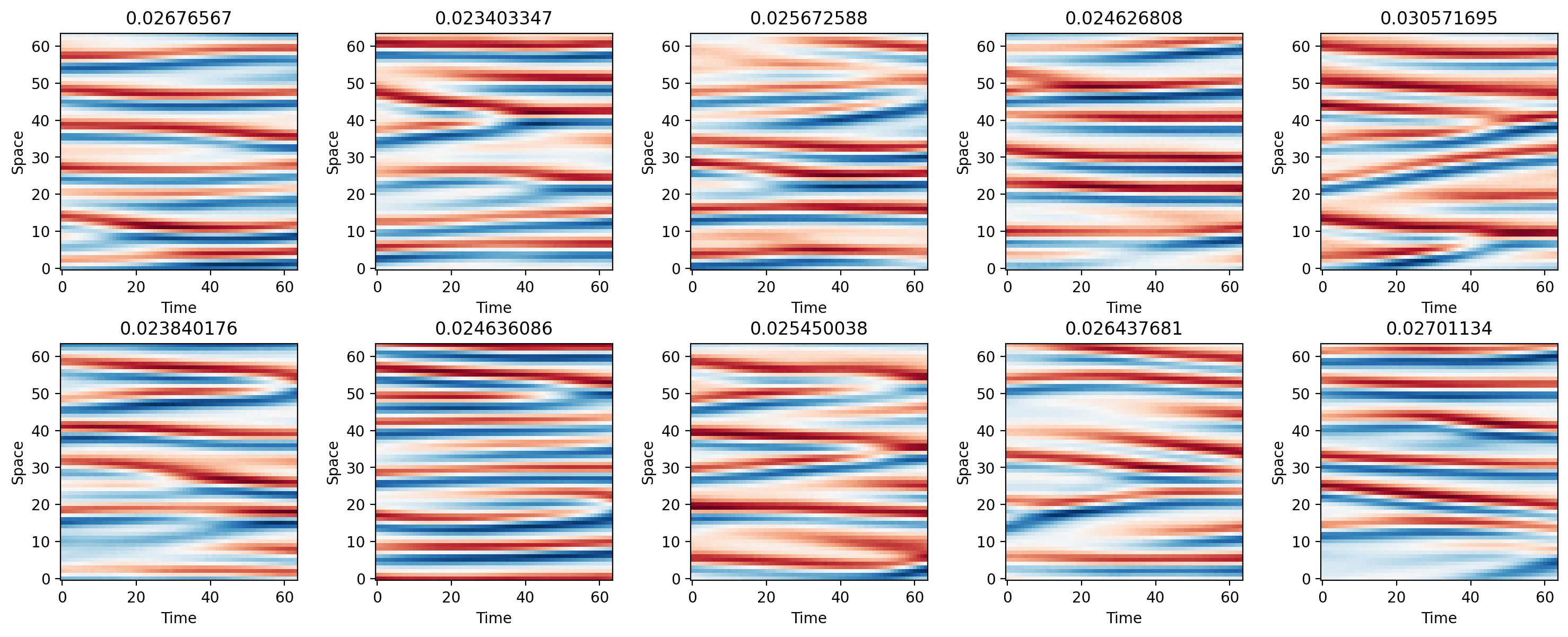}
    \caption{Samples generated by PDFM in Case 1 ($\tau = 0.033, b = 98\%$). The titles denote the residual $\delta_{L_n}$.  The constraint satisfaction rate is $97.74\%$.}
    \label{fig:PDFMsamples} 
\end{figure}

\begin{figure}[ht]
    \centering
    \resizebox{\textwidth}{!}{$
    \begin{tikzpicture}

    \def\imgw{0.16\textwidth}
    \def\imgh{2.2cm}
    \def\xsep{0.2cm}
    \def\ysep{0.45cm}
    \def\xlabsep{-0.08cm}
    \def\ylabsep{-0.15cm}

    \def\gpadx{0.12cm}
    \def\gpady{0.12cm}

    \pgfmathsetlengthmacro{\xstep}{\imgw + \xsep}
    \pgfmathsetlengthmacro{\ystep}{\imgh + \ysep}


    \node at (0*\xstep,0) {\includegraphics[width=\imgw,height=\imgh]{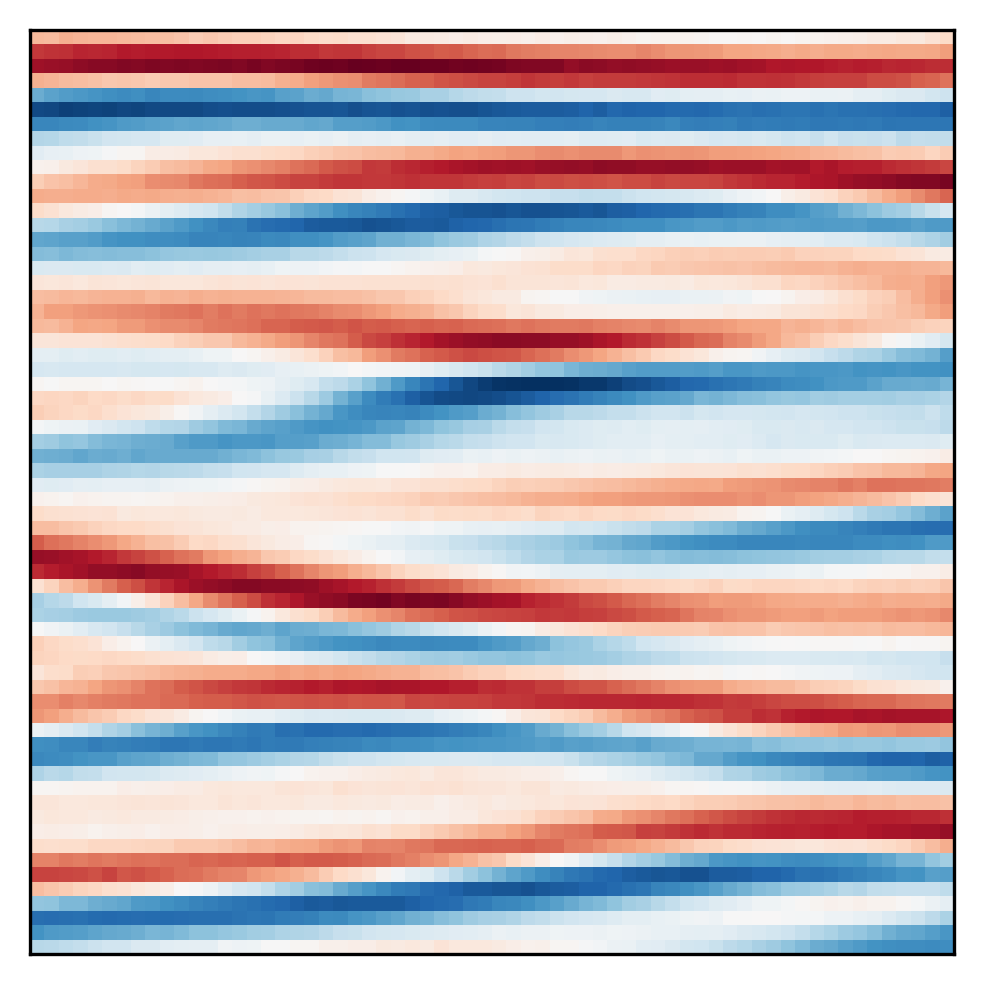}};
    \node[anchor=north, font=\tiny] at (0*\xstep,{0 - 0.5*\imgh - \xlabsep}) {$t$};
    \node[rotate=90, anchor=south, font=\tiny] at ({0*\xstep - 0.5*\imgw - \ylabsep},0) {$x$};

        \node[rotate=90, anchor=south, font=\small] at ({0*\xstep - 0.6*\imgw - \ylabsep},0){Generated};

    \node at (1*\xstep,0) {\includegraphics[width=\imgw,height=\imgh]{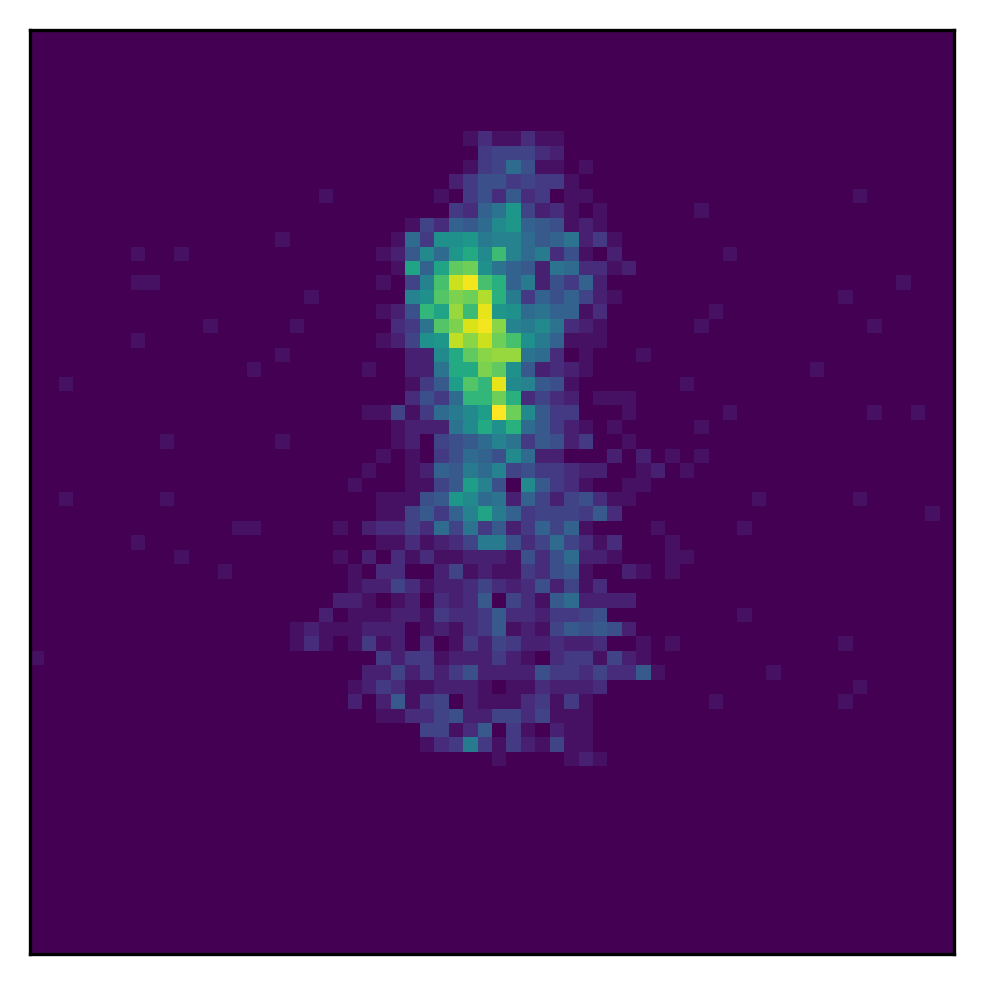}};
    \node[anchor=north, font=\tiny] at (1*\xstep,{0 - 0.5*\imgh - \xlabsep}) {$\partial u/\partial x$};

    \node[anchor=north, font=\small] at (1*\xstep,{0 + 0.7*\imgh - \xlabsep}) {FM};
    
    \node[rotate=90, anchor=south, font=\tiny] at ({1*\xstep - 0.5*\imgw - \ylabsep},0) {$\partial u/\partial t$};

    \node at (2*\xstep,0) {\includegraphics[width=\imgw,height=\imgh]{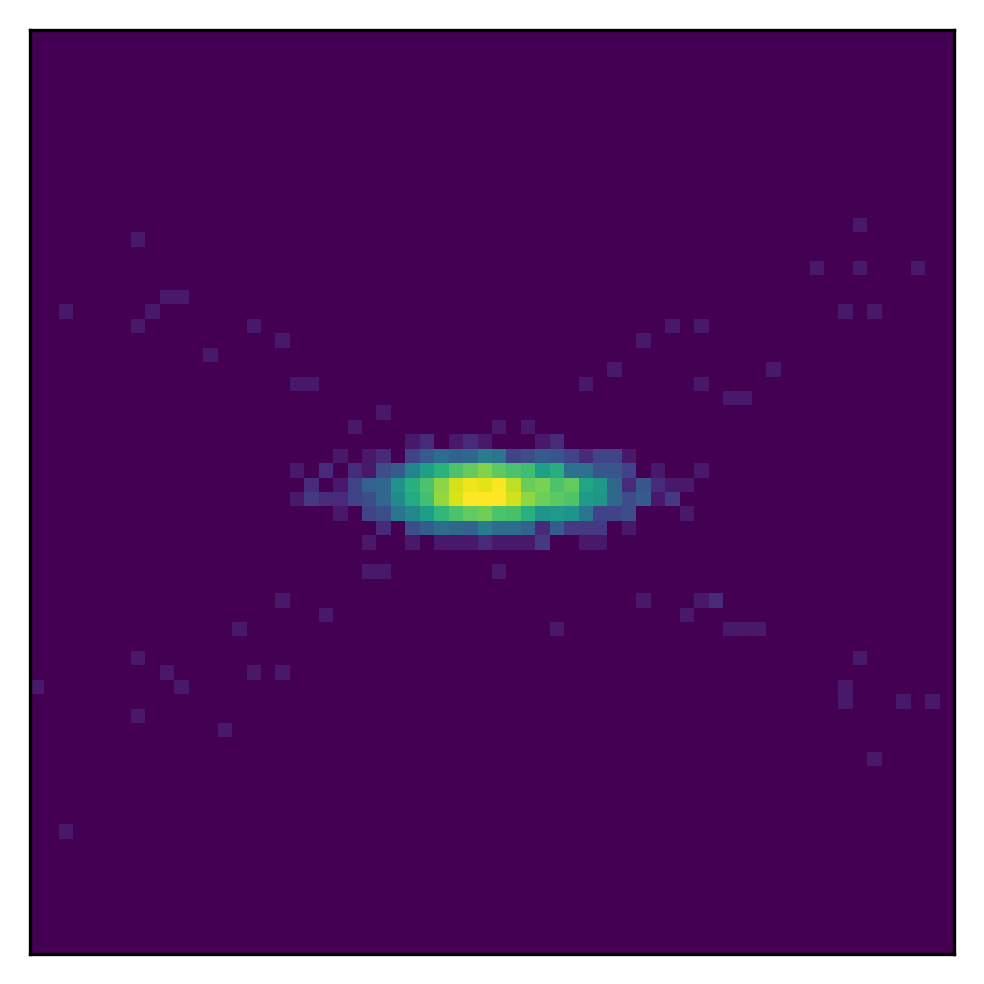}};
    \node[anchor=north, font=\tiny] at (2*\xstep,{0 - 0.5*\imgh - \xlabsep}) {$\partial u/\partial x$};
    \node[rotate=90, anchor=south, font=\tiny] at ({2*\xstep - 0.5*\imgw - \ylabsep},0) {$\partial^2 u/\partial x^2$};

    \node at (3*\xstep,0) {\includegraphics[width=\imgw,height=\imgh]{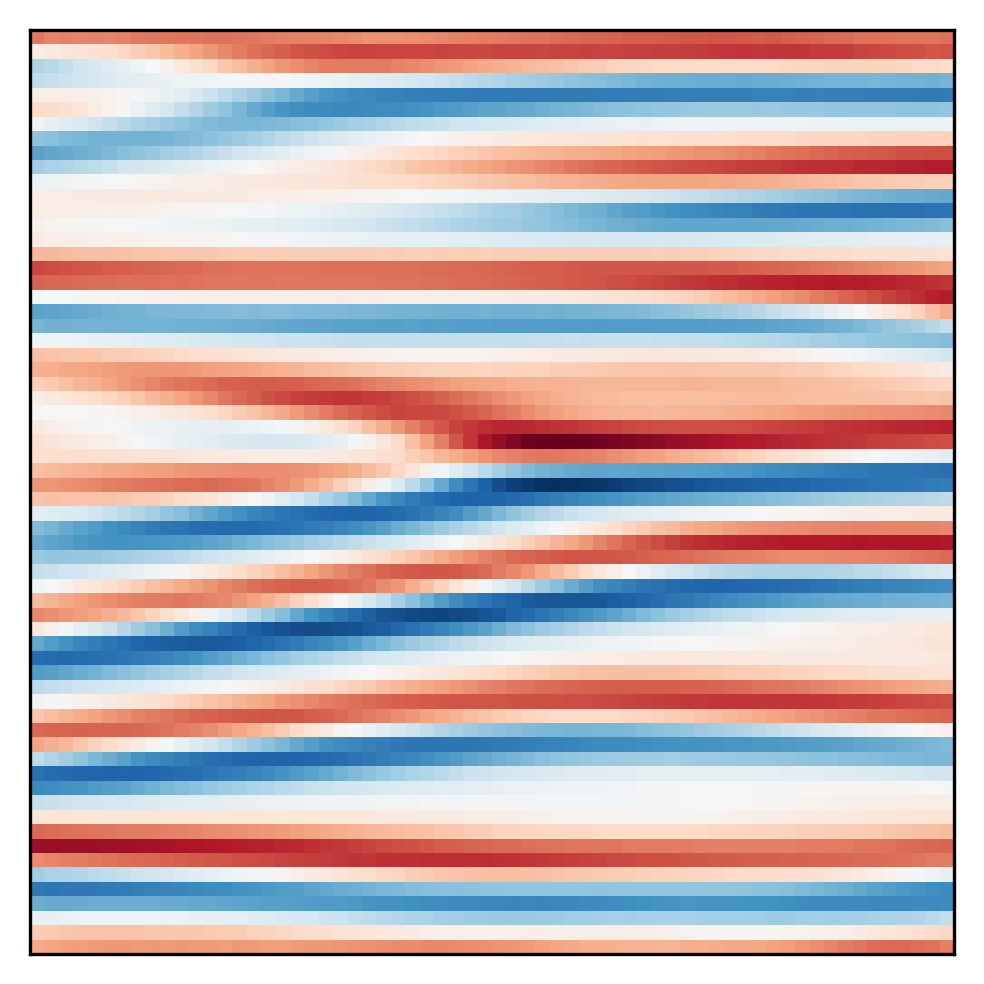}};
    \node[anchor=north, font=\tiny] at (3*\xstep,{0 - 0.5*\imgh - \xlabsep}) {$t$};
    \node[rotate=90, anchor=south, font=\tiny] at ({3*\xstep - 0.5*\imgw - \ylabsep},0) {$x$};

    \node at (4*\xstep,0) {\includegraphics[width=\imgw,height=\imgh]{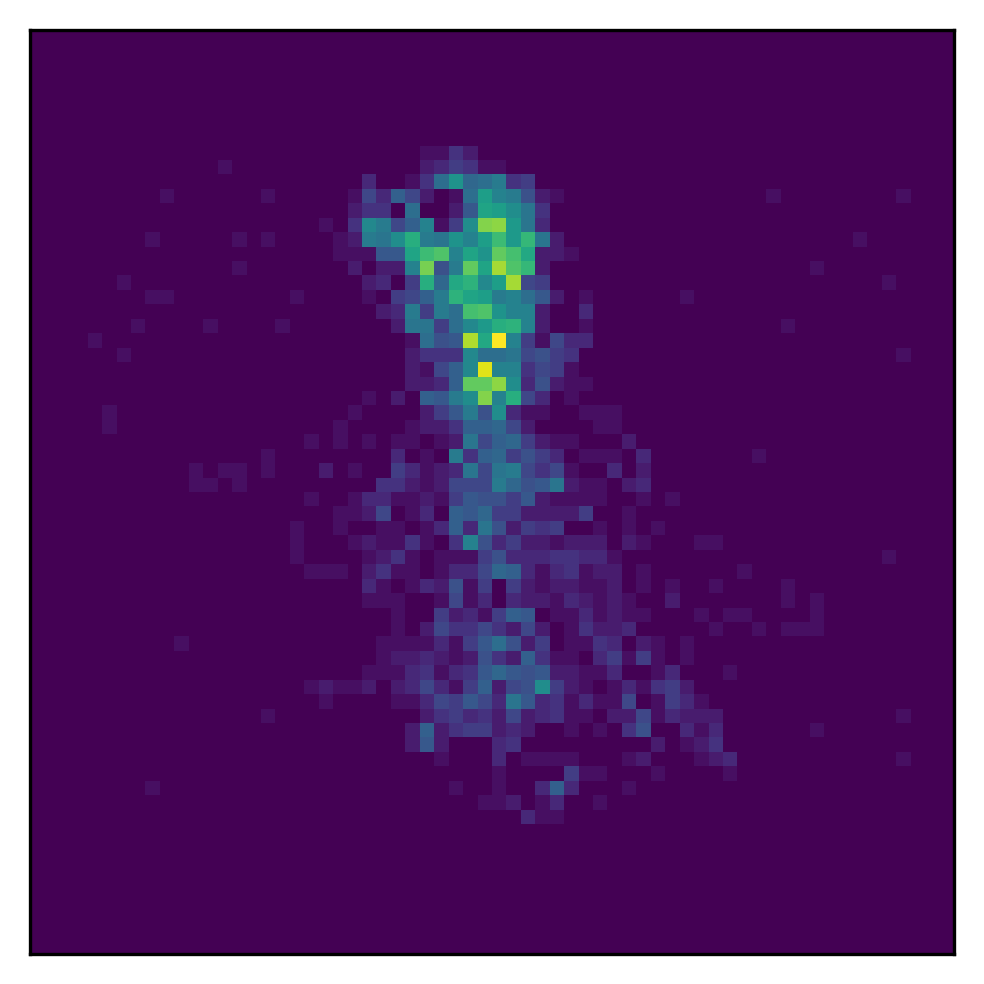}};
    \node[anchor=north, font=\tiny] at (4*\xstep,{0 - 0.5*\imgh - \xlabsep}) {$\partial u/\partial x$};

       \node[anchor=north, font=\small] at (4*\xstep,{0 + 0.7*\imgh - \xlabsep}) {PDFM};
       
    \node[rotate=90, anchor=south, font=\tiny] at ({4*\xstep - 0.5*\imgw - \ylabsep},0) {$\partial u/\partial t$};

    \node at (5*\xstep,0) {\includegraphics[width=\imgw,height=\imgh]{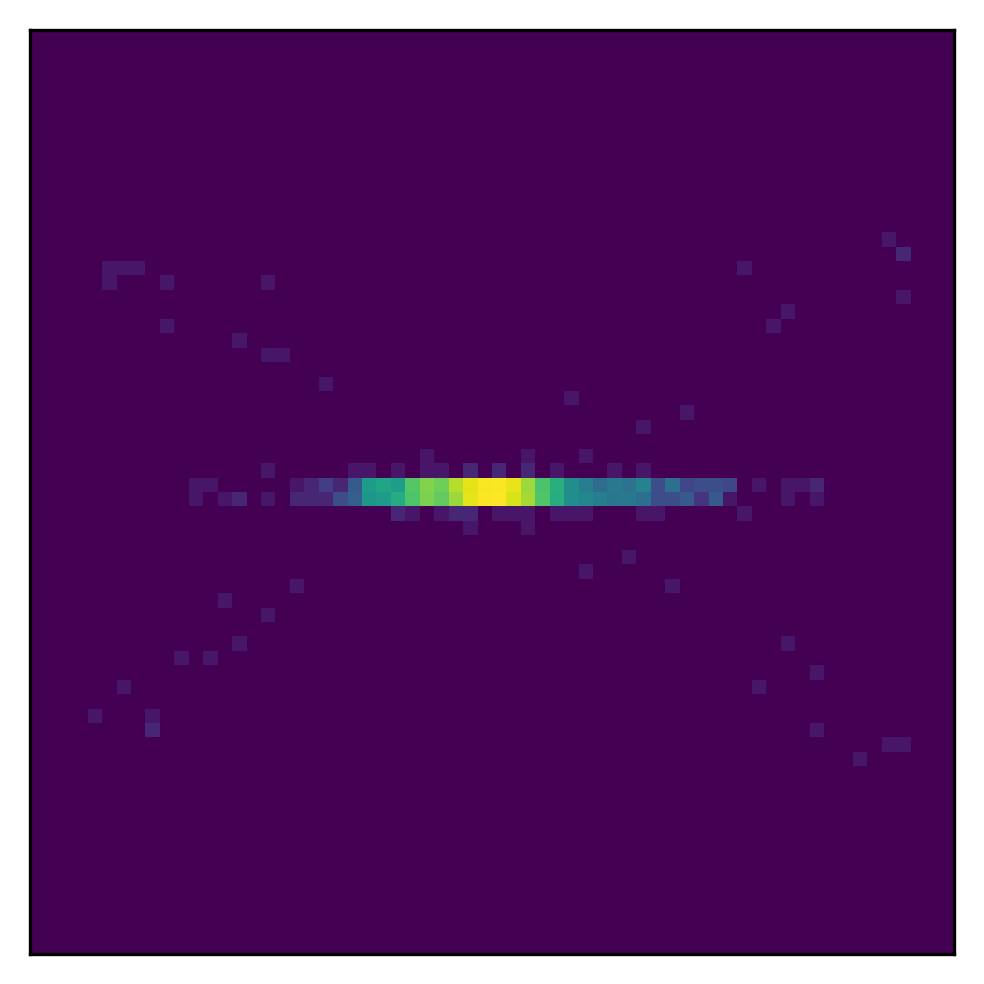}};
    \node[anchor=north, font=\tiny] at (5*\xstep,{0 - 0.5*\imgh - \xlabsep}) {$\partial u/\partial x$};
    \node[rotate=90, anchor=south, font=\tiny] at ({5*\xstep - 0.5*\imgw - \ylabsep},0) {$\partial^2 u/\partial x^2$};


    \node[rotate=90, anchor=south, font=\small] at ({0*\xstep - 0.6*\imgw - \ylabsep},-\ystep) {Truth};
    
    \node at (0*\xstep,-\ystep) {\includegraphics[width=\imgw,height=\imgh]{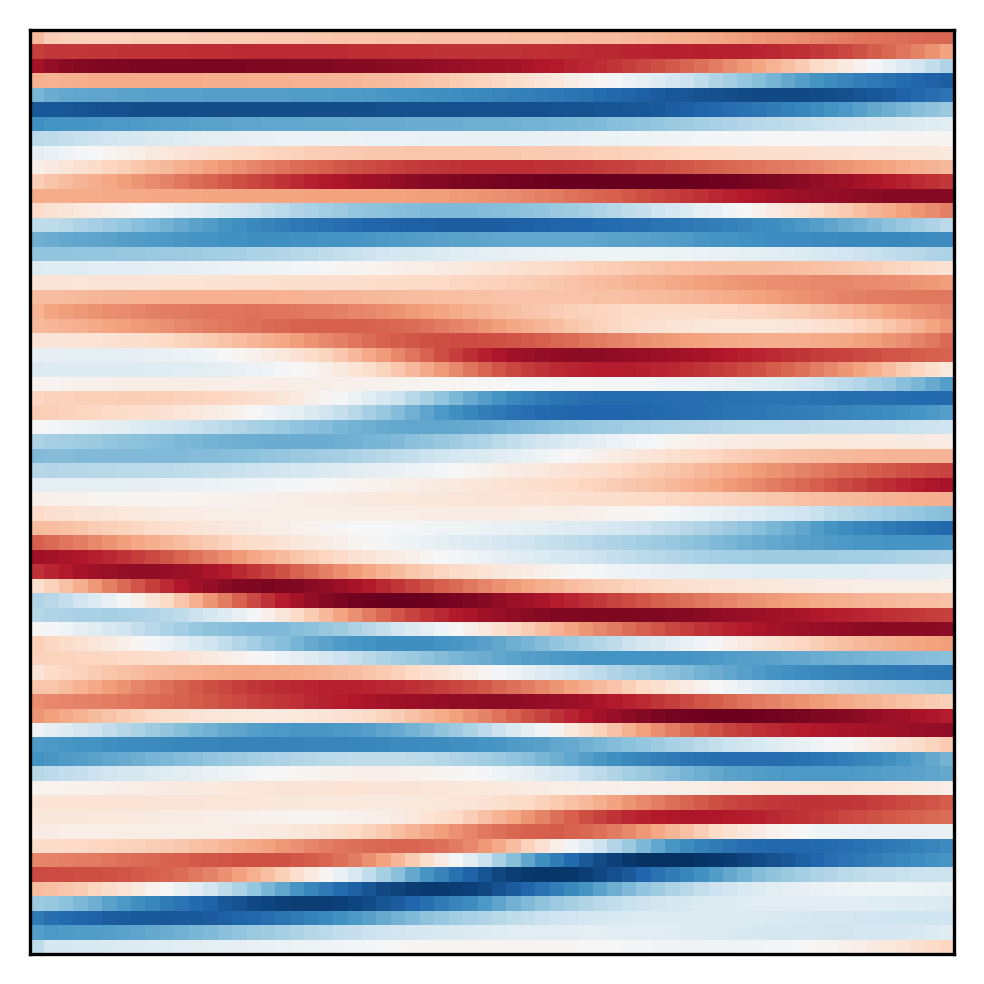}};
    \node[anchor=north, font=\tiny] at (0*\xstep,{-\ystep - 0.5*\imgh - \xlabsep}) {$t$};
    \node[rotate=90, anchor=south, font=\tiny] at ({0*\xstep - 0.5*\imgw - \ylabsep},-\ystep) {$x$};

    \node at (1*\xstep,-\ystep) {\includegraphics[width=\imgw,height=\imgh]{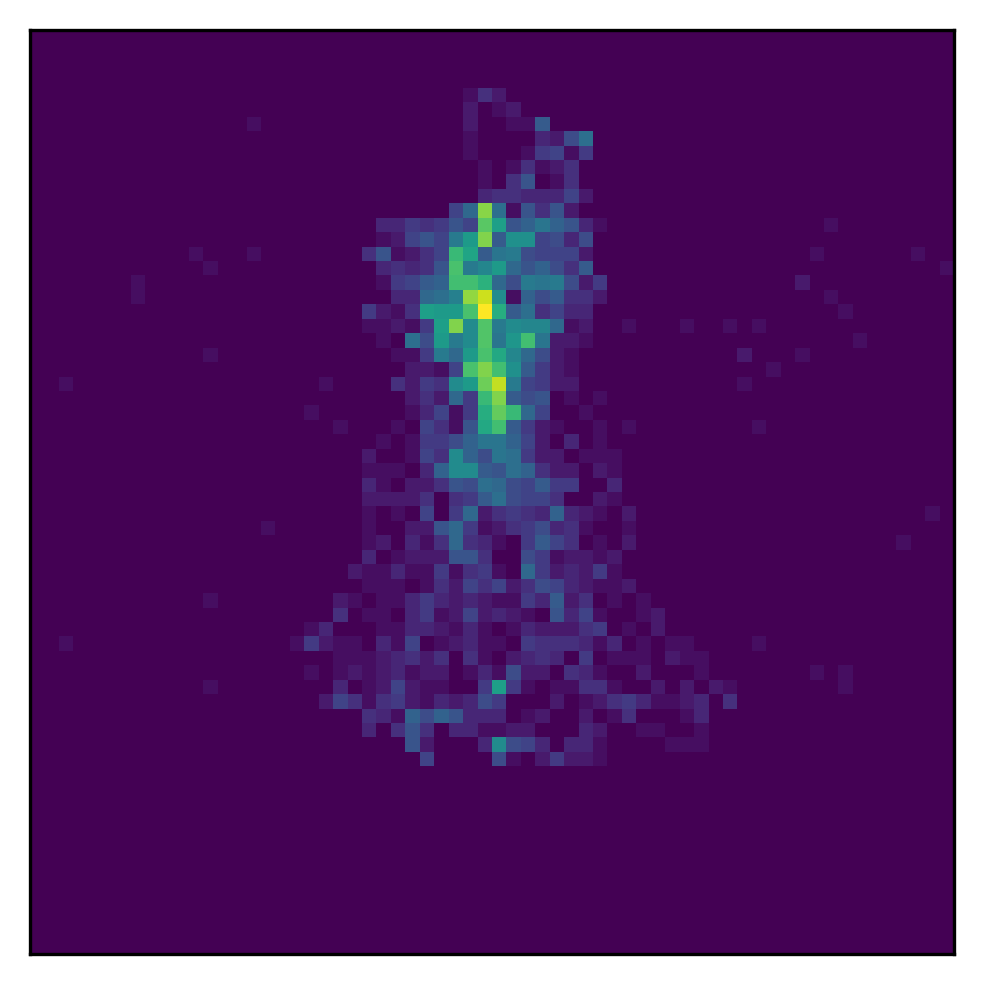}};
    \node[anchor=north, font=\tiny] at (1*\xstep,{-\ystep - 0.5*\imgh - \xlabsep}) {$\partial u/\partial x$};
    \node[rotate=90, anchor=south, font=\tiny] at ({1*\xstep - 0.5*\imgw - \ylabsep},-\ystep) {$\partial u/\partial t$};

    \node at (2*\xstep,-\ystep) {\includegraphics[width=\imgw,height=\imgh]{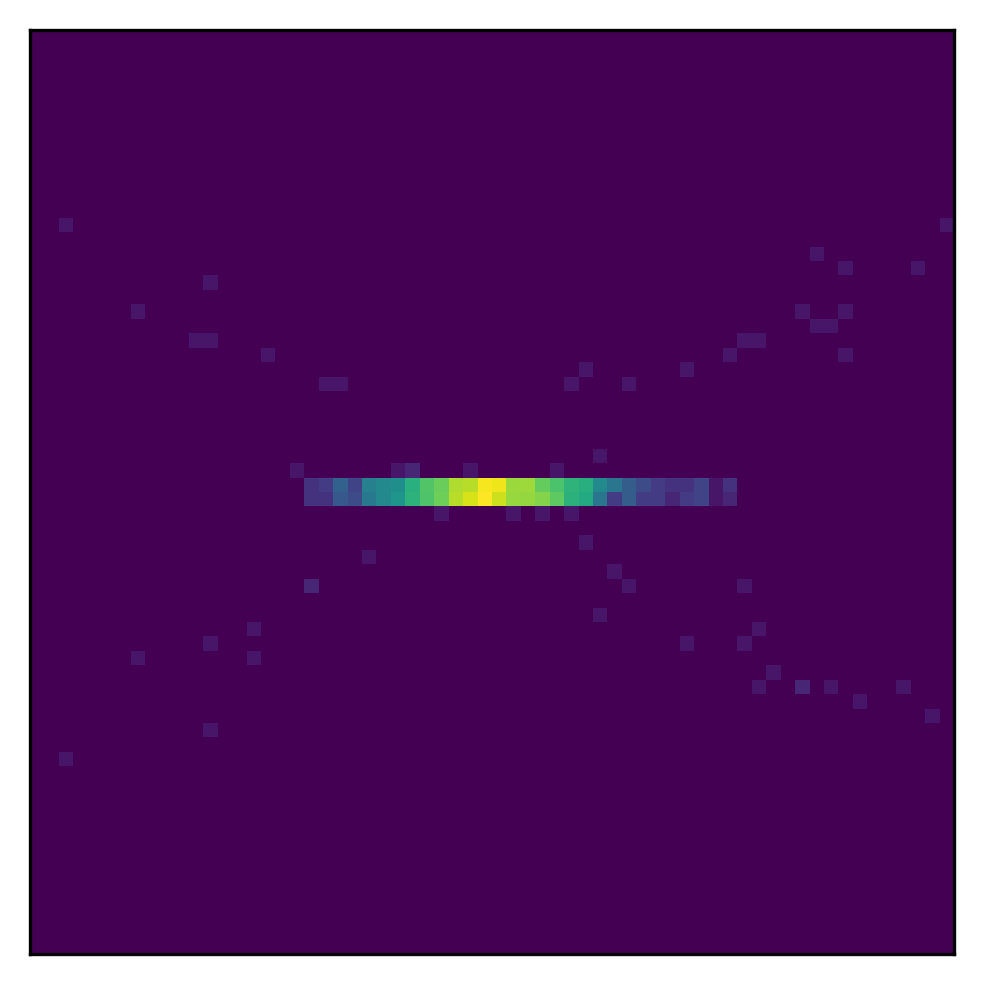}};
    \node[anchor=north, font=\tiny] at (2*\xstep,{-\ystep - 0.5*\imgh - \xlabsep}) {$\partial u/\partial x$};
    \node[rotate=90, anchor=south, font=\tiny] at ({2*\xstep - 0.5*\imgw - \ylabsep},-\ystep) {$\partial^2 u/\partial x^2$};

    \node at (3*\xstep,-\ystep) {\includegraphics[width=\imgw,height=\imgh]{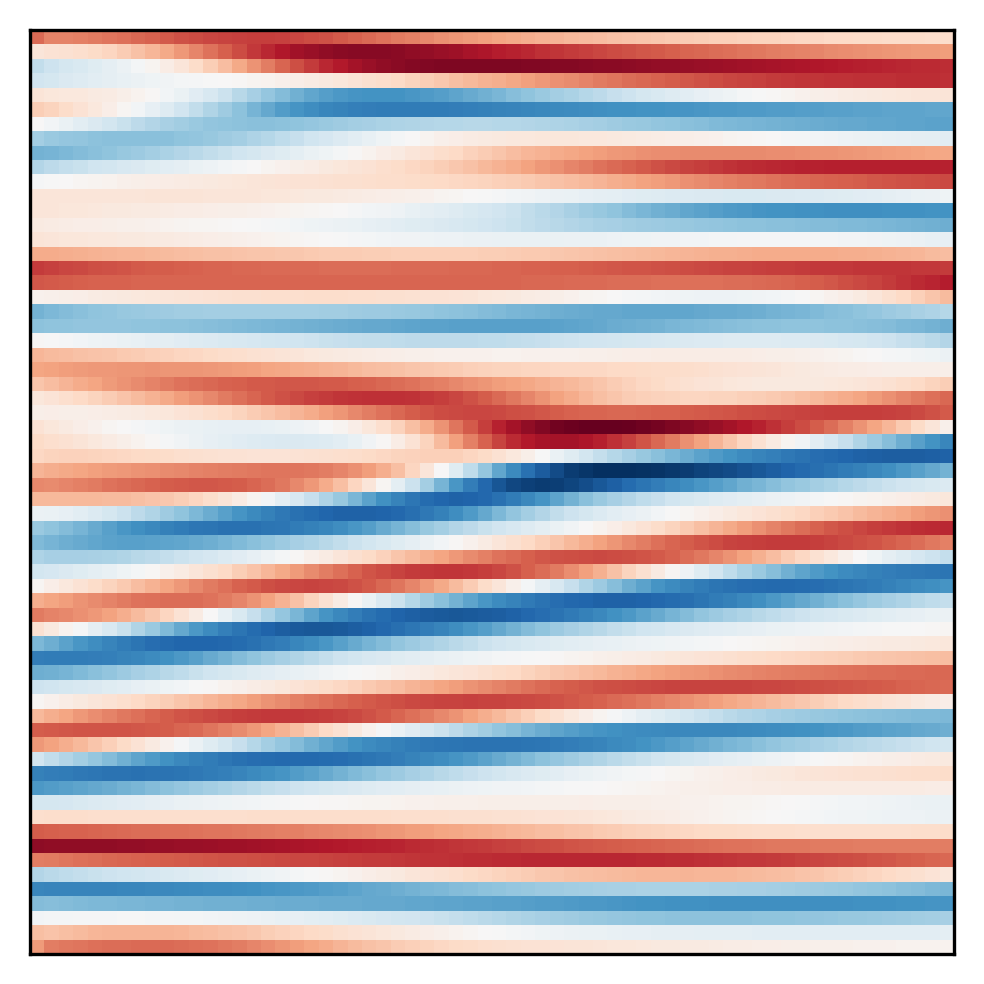}};
    \node[anchor=north, font=\tiny] at (3*\xstep,{-\ystep - 0.5*\imgh - \xlabsep}) {$t$};
    \node[rotate=90, anchor=south, font=\tiny] at ({3*\xstep - 0.5*\imgw - \ylabsep},-\ystep) {$x$};

    \node at (4*\xstep,-\ystep) {\includegraphics[width=\imgw,height=\imgh]{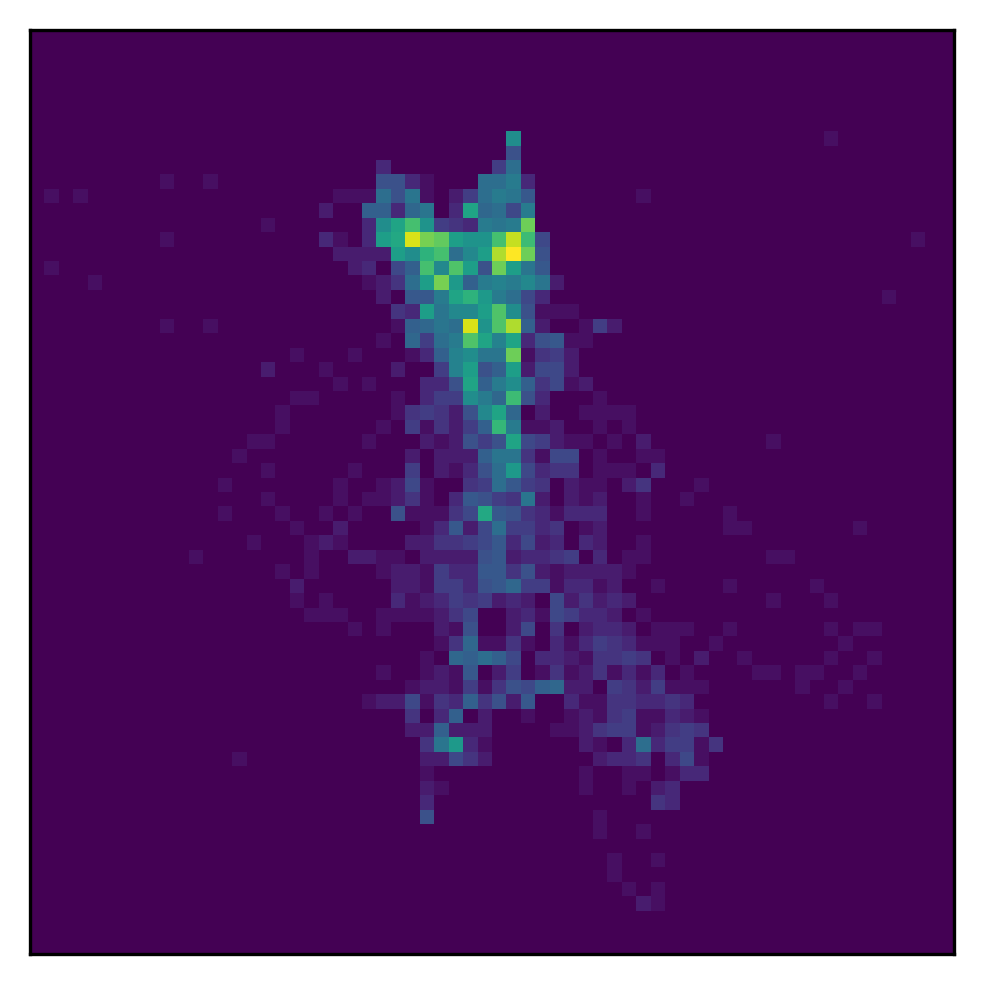}};
    \node[anchor=north, font=\tiny] at (4*\xstep,{-\ystep - 0.5*\imgh - \xlabsep}) {$\partial u/\partial x$};
    \node[rotate=90, anchor=south, font=\tiny] at ({4*\xstep - 0.5*\imgw - \ylabsep},-\ystep) {$\partial u/\partial t$};

    \node at (5*\xstep,-\ystep) {\includegraphics[width=\imgw,height=\imgh]{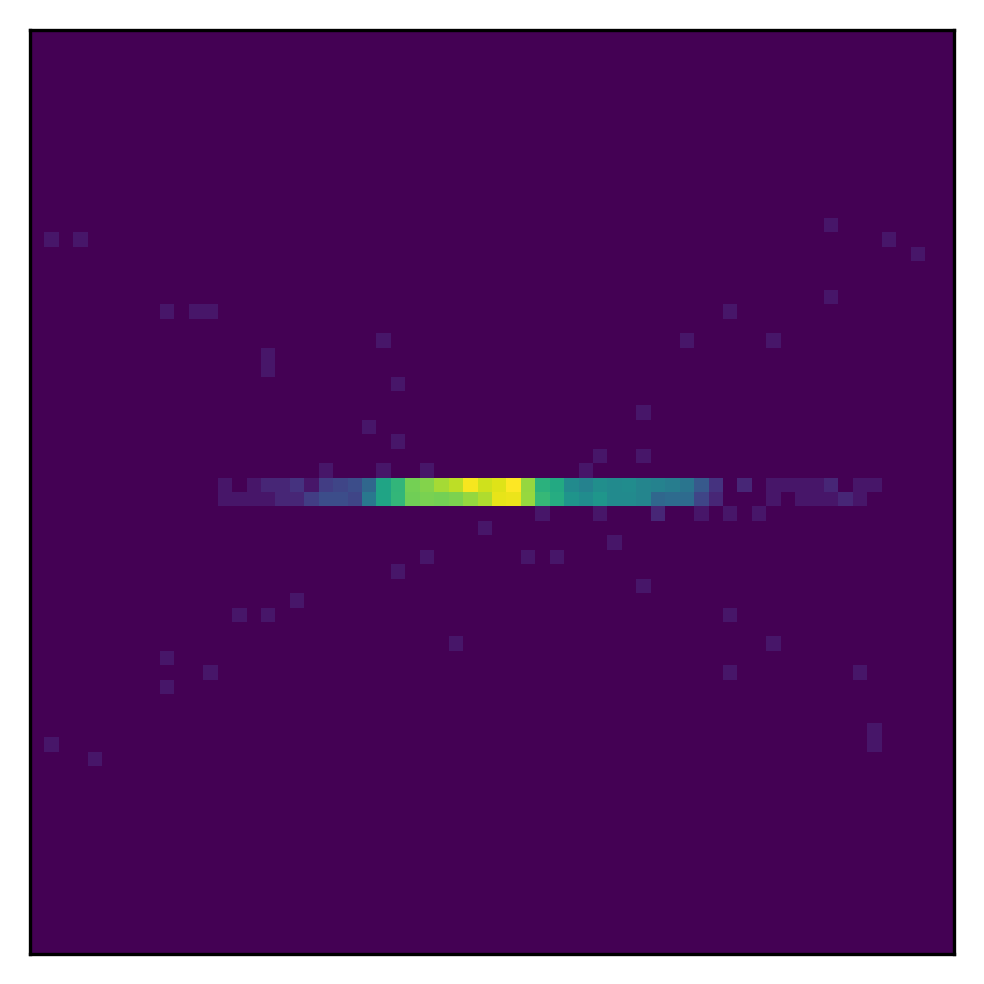}};
    \node[anchor=north, font=\tiny] at (5*\xstep,{-\ystep - 0.5*\imgh - \xlabsep}) {$\partial u/\partial x$};
    \node[rotate=90, anchor=south, font=\tiny] at ({5*\xstep - 0.5*\imgw - \ylabsep},-\ystep) {$\partial^2 u/\partial x^2$};

    \draw[black, dashed, rounded corners=2pt, line width=1pt]
      ({-0.5*\imgw - \gpadx}, {0.5*\imgh + 0.2*\gpady})
      rectangle
      ({2*\xstep + 0.5*\imgw + 0*\gpadx}, {-0.5*\imgh - 2.2*\gpady});

    \draw[black, dashed, rounded corners=2pt, line width=1pt]
      ({3*\xstep - 0.5*\imgw - \gpadx}, {0.5*\imgh +  0.2*\gpady})
      rectangle
      ({5*\xstep + 0.5*\imgw + 0*\gpadx}, {-0.5*\imgh - 2.2*\gpady});

    \draw[black, dashed, rounded corners=2pt, line width=1pt]
      ({-0.5*\imgw - \gpadx}, {-\ystep + 0.5*\imgh +  0.2*\gpady})
      rectangle
      ({2*\xstep + 0.5*\imgw + 0*\gpadx}, {-\ystep - 0.5*\imgh - 2.2*\gpady});

    \draw[black, dashed, rounded corners=2pt, line width=1pt]
      ({3*\xstep - 0.5*\imgw - \gpadx}, {-\ystep + 0.5*\imgh +  0.2*\gpady})
      rectangle
      ({5*\xstep + 0.5*\imgw +0* \gpadx}, {-\ystep - 0.5*\imgh - 2.2*\gpady});
    \end{tikzpicture}$}
    \caption{Single generated sample and its summary statistic in Case 1. For each three-image set, the first image shows sampled emulator dynamics, and the last two show histograms of summary statistic distributions. First row: samples generated by the generative models. Left: FM, $\delta_{L_n} = 0.0421>\tau$. Right: PDFM, $\delta_{L_n} = 0.0256<\tau$. Second row: simulator (truth) samples initialized from the same initial states in the first row. }
    \label{fig:case1_individual}
\end{figure}

\begin{figure}[ht]
    \centering
    \resizebox{\textwidth}{!}{$
    \begin{tikzpicture}

    \def\imgw{0.16\textwidth}
    \def\imgh{2.2cm}
    \def\xsep{0.2cm}
    \def\ysep{0.45cm}
    \def\xlabsep{-0.08cm}
    \def\ylabsep{-0.15cm}

    \def\gpadx{0.12cm}
    \def\gpady{0.12cm}

    \pgfmathsetlengthmacro{\xstep}{\imgw + \xsep}
    \pgfmathsetlengthmacro{\ystep}{\imgh + \ysep}


    \node at (0*\xstep,0) {\includegraphics[width=\imgw,height=\imgh]{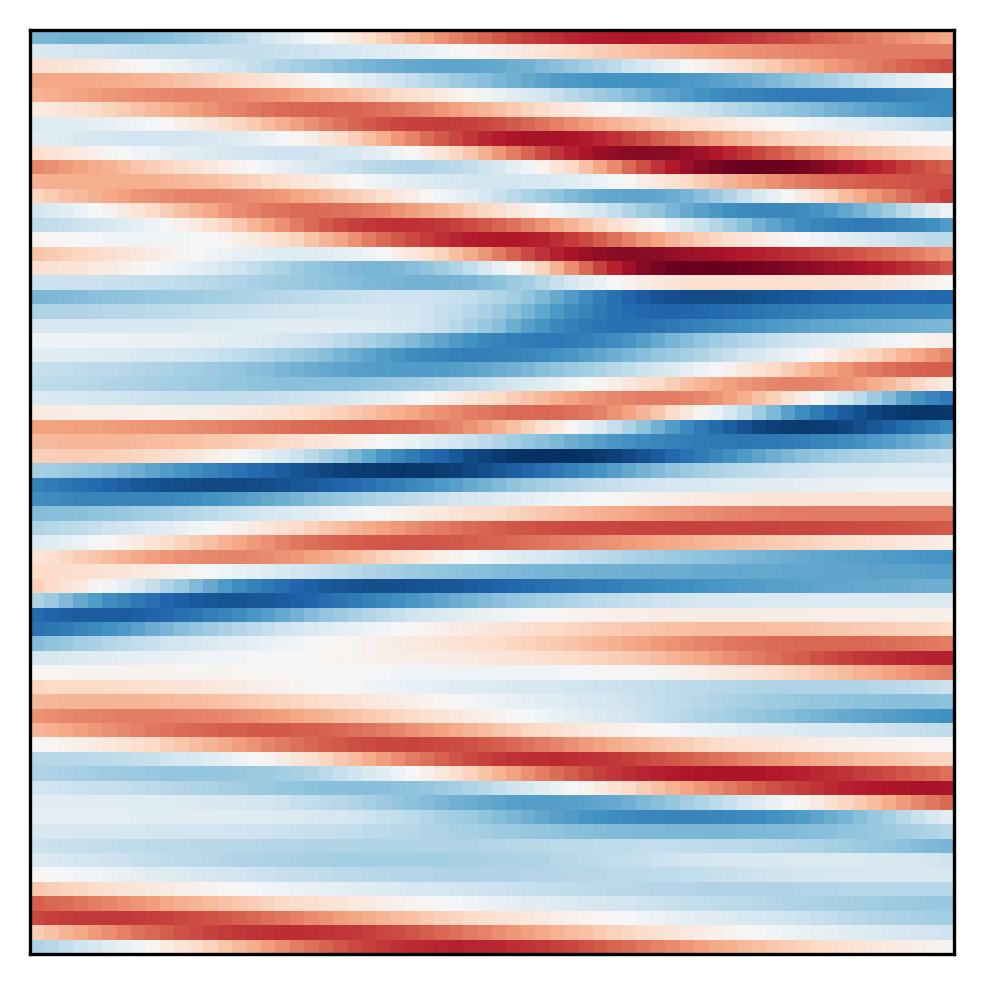}};
    \node[anchor=north, font=\tiny] at (0*\xstep,{0 - 0.5*\imgh - \xlabsep}) {$t$};
    \node[rotate=90, anchor=south, font=\tiny] at ({0*\xstep - 0.5*\imgw - \ylabsep},0) {$x$};

    \node at (1*\xstep,0) {\includegraphics[width=\imgw,height=\imgh]{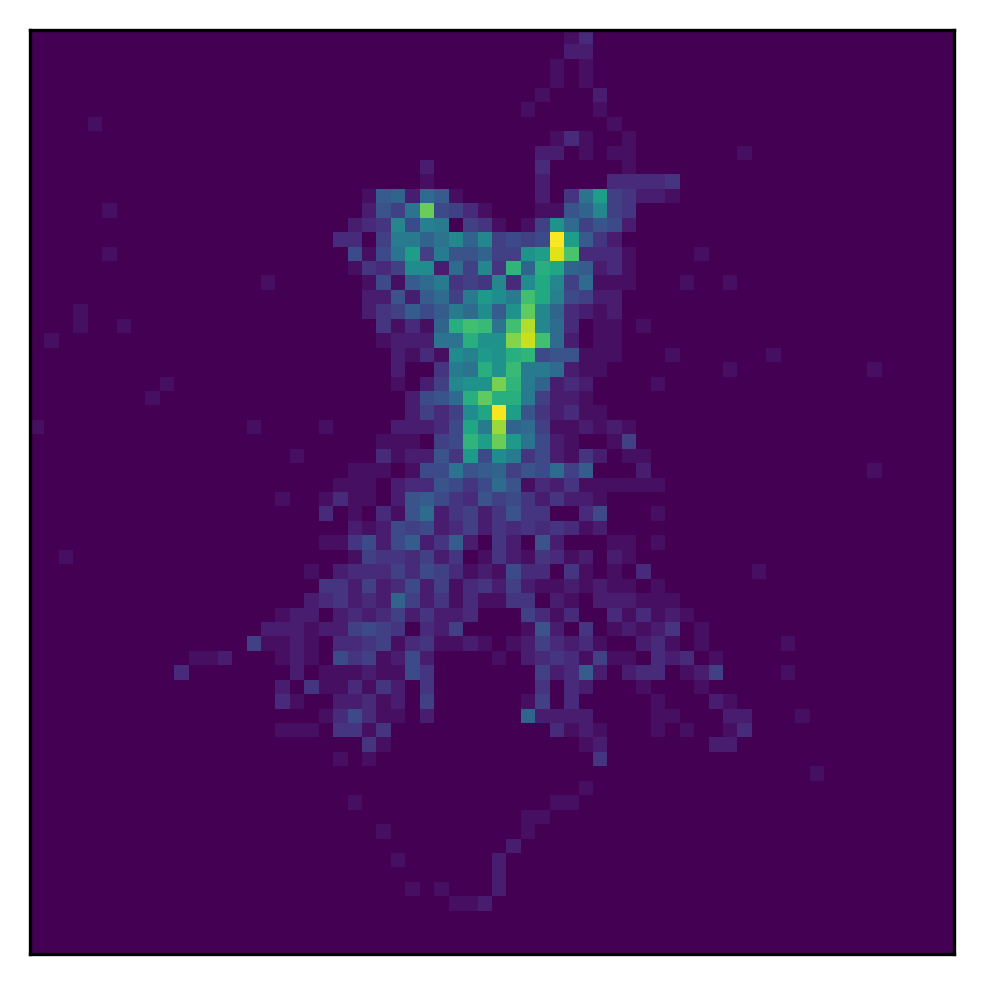}};
    \node[anchor=north, font=\tiny] at (1*\xstep,{0 - 0.5*\imgh - \xlabsep}) {$\partial u/\partial x$};
        \node[anchor=north, font=\small] at (1*\xstep,{0 + 0.7*\imgh - \xlabsep}) {FM};
    \node[rotate=90, anchor=south, font=\tiny] at ({1*\xstep - 0.5*\imgw - \ylabsep},0) {$\partial u/\partial t$};

    \node at (2*\xstep,0) {\includegraphics[width=\imgw,height=\imgh]{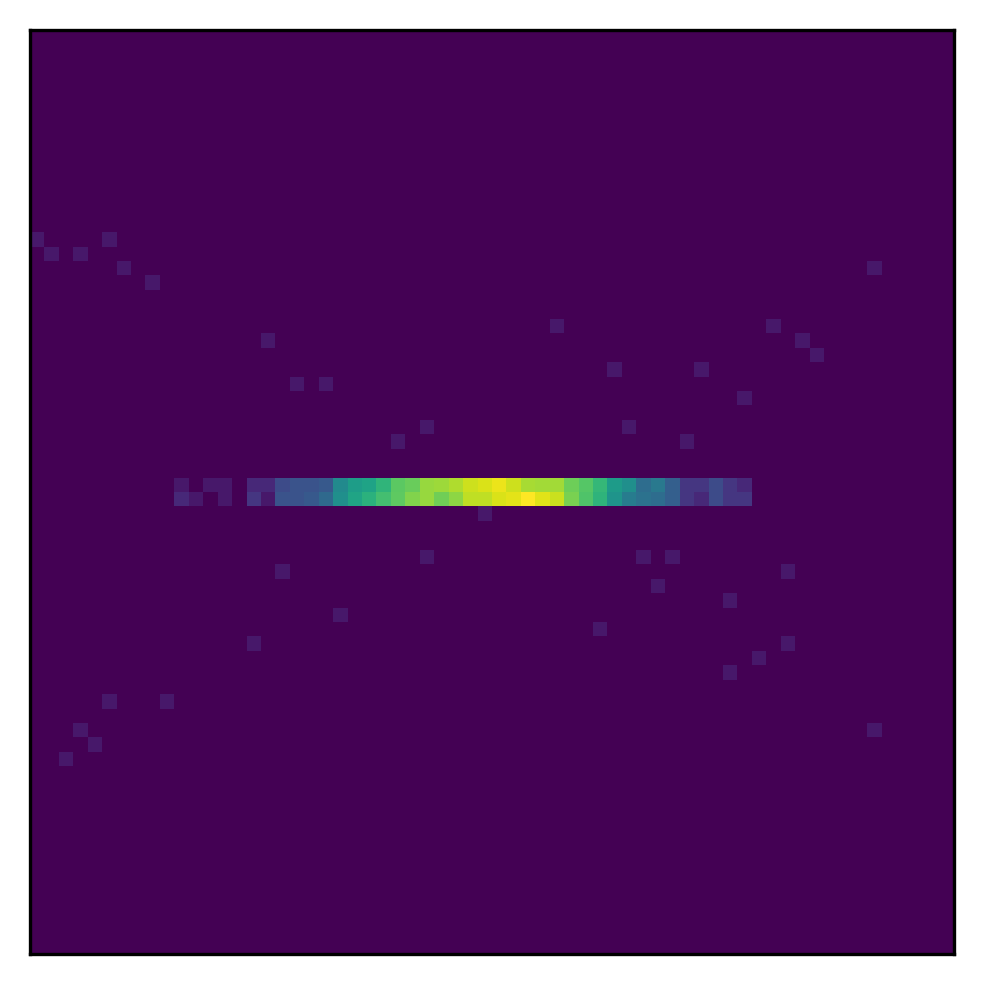}};
    \node[anchor=north, font=\tiny] at (2*\xstep,{0 - 0.5*\imgh - \xlabsep}) {$\partial u/\partial x$};
    \node[rotate=90, anchor=south, font=\tiny] at ({2*\xstep - 0.5*\imgw - \ylabsep},0) {$\partial^2 u/\partial x^2$};

    \node at (3*\xstep,0) {\includegraphics[width=\imgw,height=\imgh]{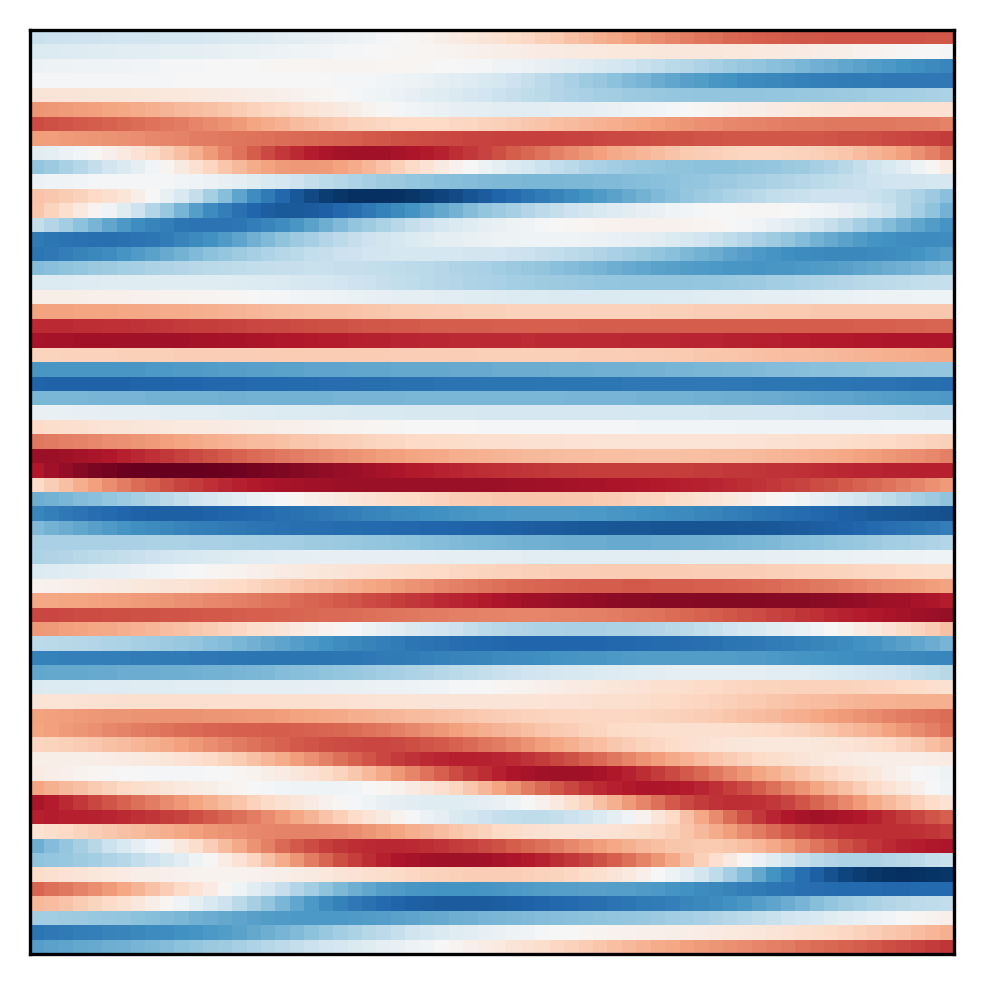}};
    \node[anchor=north, font=\tiny] at (3*\xstep,{0 - 0.5*\imgh - \xlabsep}) {$t$};
    \node[rotate=90, anchor=south, font=\tiny] at ({3*\xstep - 0.5*\imgw - \ylabsep},0) {$x$};

    \node at (4*\xstep,0) {\includegraphics[width=\imgw,height=\imgh]{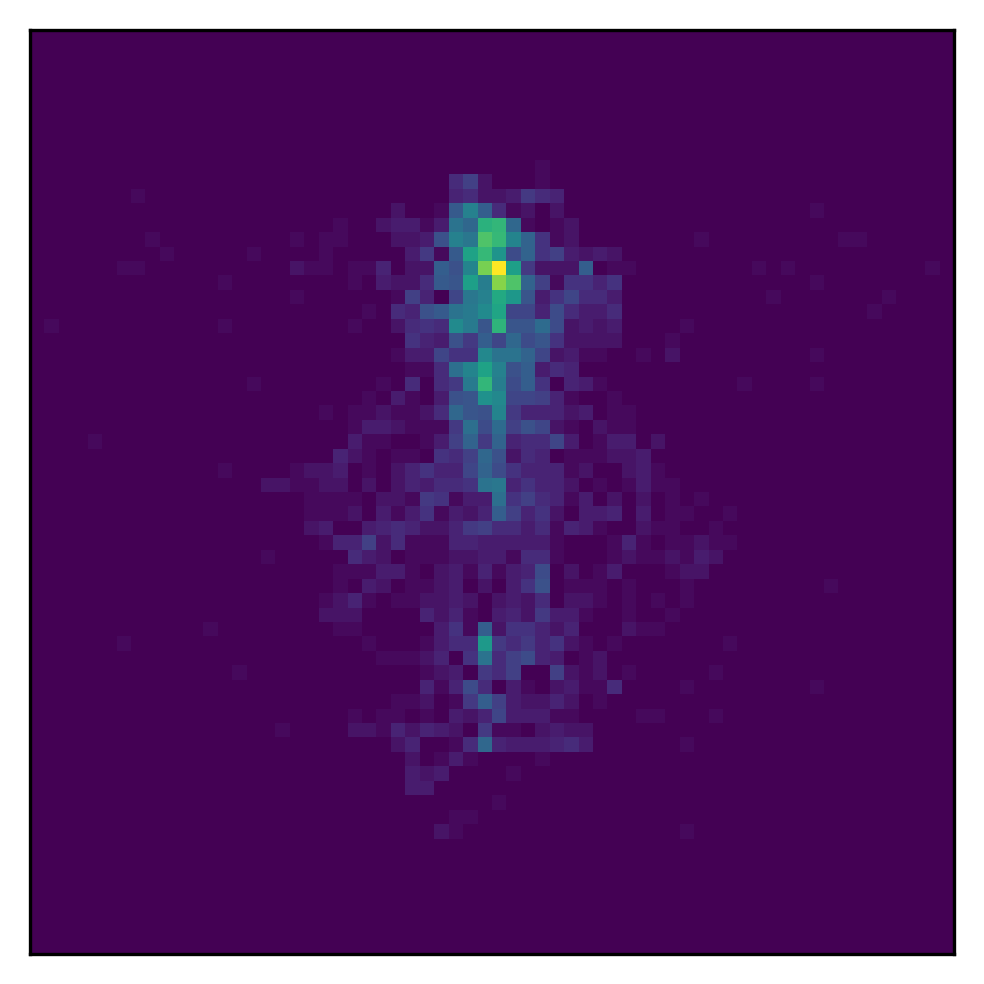}};
    \node[anchor=north, font=\tiny] at (4*\xstep,{0 - 0.5*\imgh - \xlabsep}) {$\partial u/\partial x$};
        \node[anchor=north, font=\small] at (4*\xstep,{0 + 0.7*\imgh - \xlabsep}) {PDFM};
    \node[rotate=90, anchor=south, font=\tiny] at ({4*\xstep - 0.5*\imgw - \ylabsep},0) {$\partial u/\partial t$};

    \node at (5*\xstep,0) {\includegraphics[width=\imgw,height=\imgh]{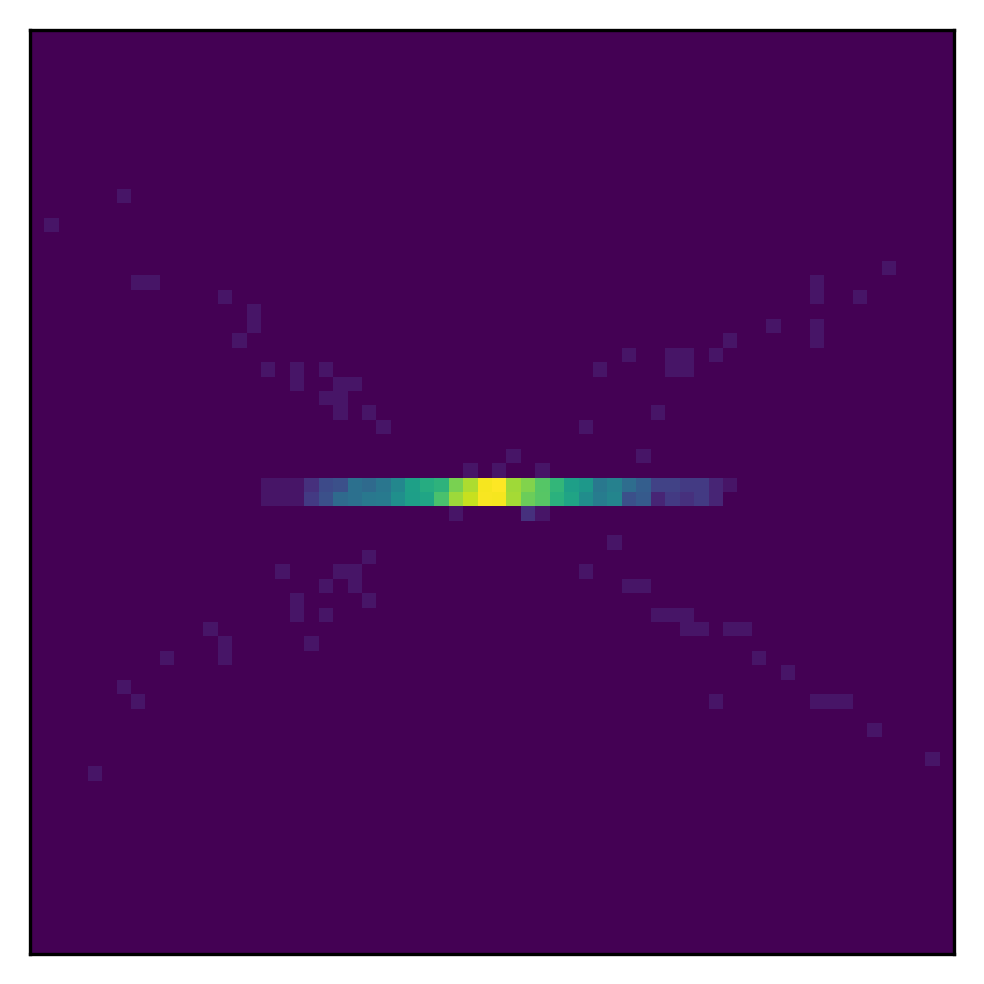}};
    \node[anchor=north, font=\tiny] at (5*\xstep,{0 - 0.5*\imgh - \xlabsep}) {$\partial u/\partial x$};
    \node[rotate=90, anchor=south, font=\tiny] at ({5*\xstep - 0.5*\imgw - \ylabsep},0) {$\partial^2 u/\partial x^2$};

    \node at (0*\xstep,-\ystep) {\includegraphics[width=\imgw,height=\imgh]{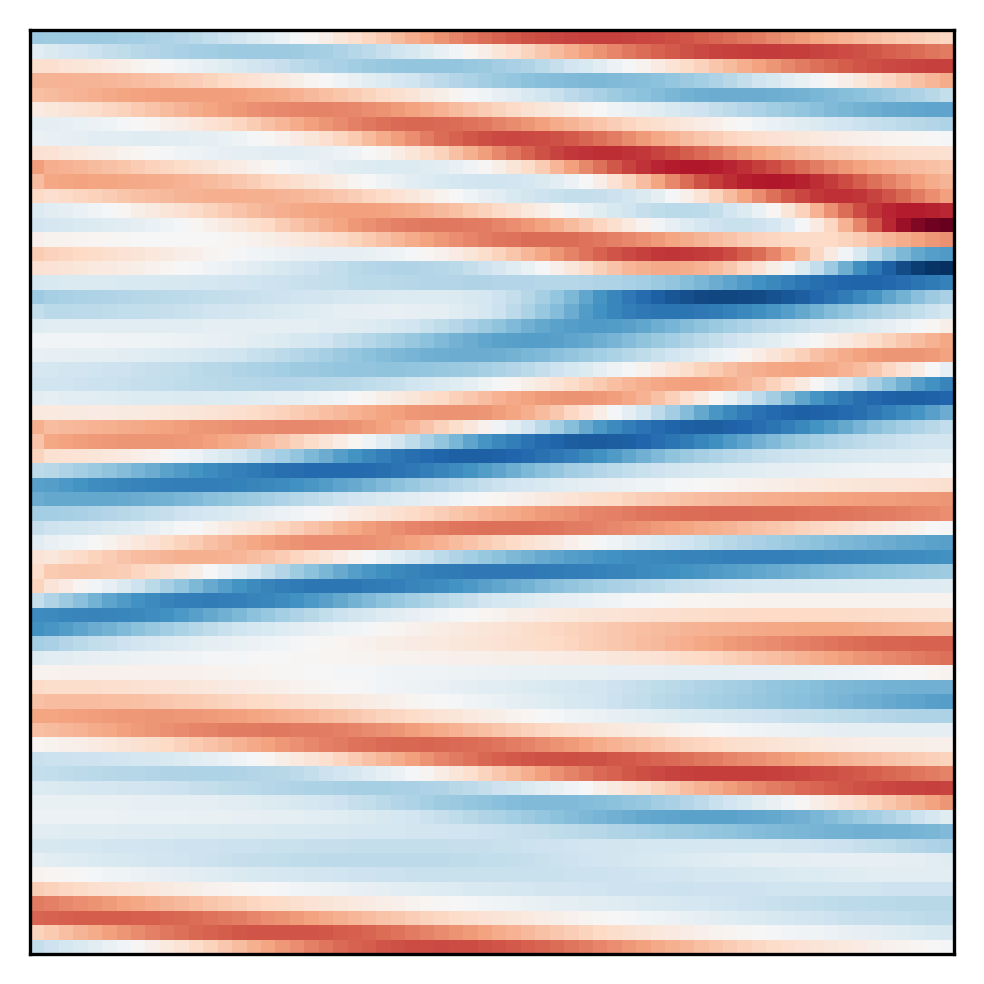}};
    \node[anchor=north, font=\tiny] at (0*\xstep,{-\ystep - 0.5*\imgh - \xlabsep}) {$t$};
    \node[rotate=90, anchor=south, font=\tiny] at ({0*\xstep - 0.5*\imgw - \ylabsep},-\ystep) {$x$};

    \node[rotate=90, anchor=south, font=\small] at ({0*\xstep - 0.6*\imgw - \ylabsep},-\ystep) {Truth};
    
    \node[rotate=90, anchor=south, font=\small] at ({0*\xstep - 0.6*\imgw - \ylabsep},0) {Generated};
    
    \node at (1*\xstep,-\ystep) {\includegraphics[width=\imgw,height=\imgh]{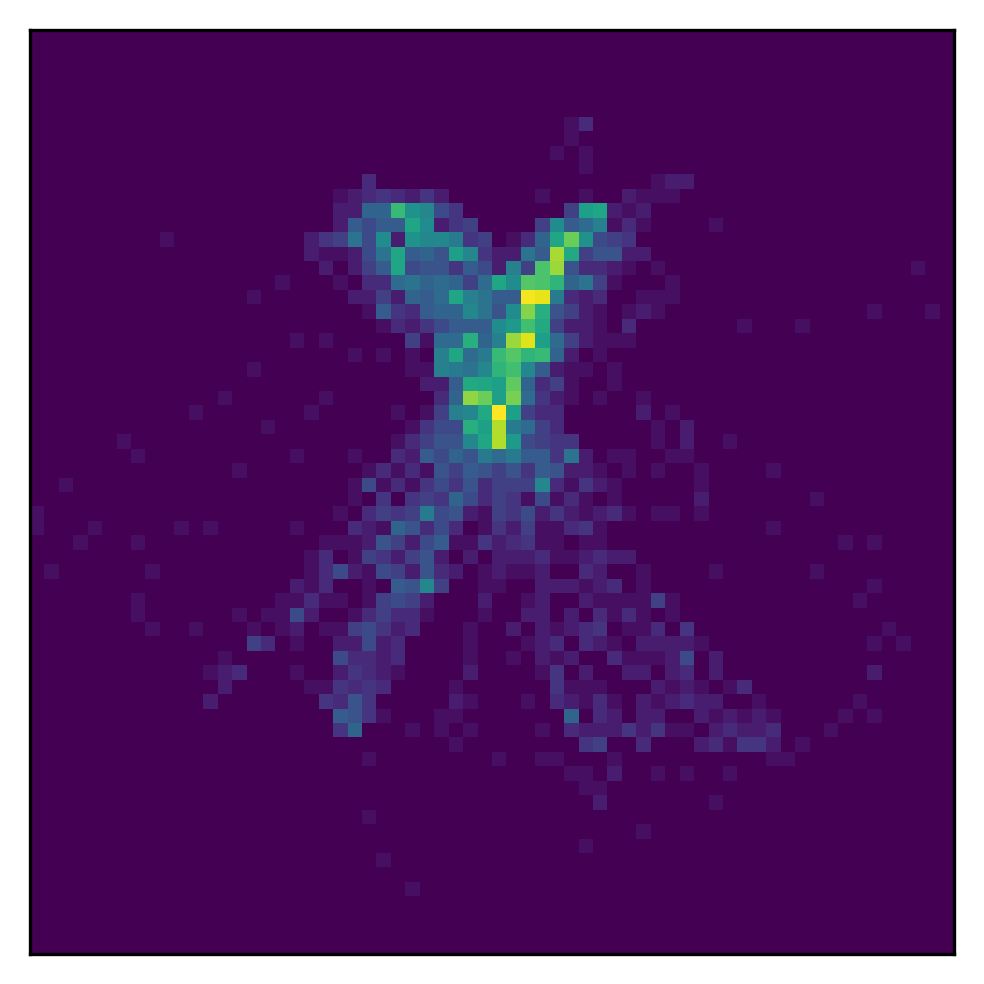}};
    \node[anchor=north, font=\tiny] at (1*\xstep,{-\ystep - 0.5*\imgh - \xlabsep}) {$\partial u/\partial x$};
    \node[rotate=90, anchor=south, font=\tiny] at ({1*\xstep - 0.5*\imgw - \ylabsep},-\ystep) {$\partial u/\partial t$};

    \node at (2*\xstep,-\ystep) {\includegraphics[width=\imgw,height=\imgh]{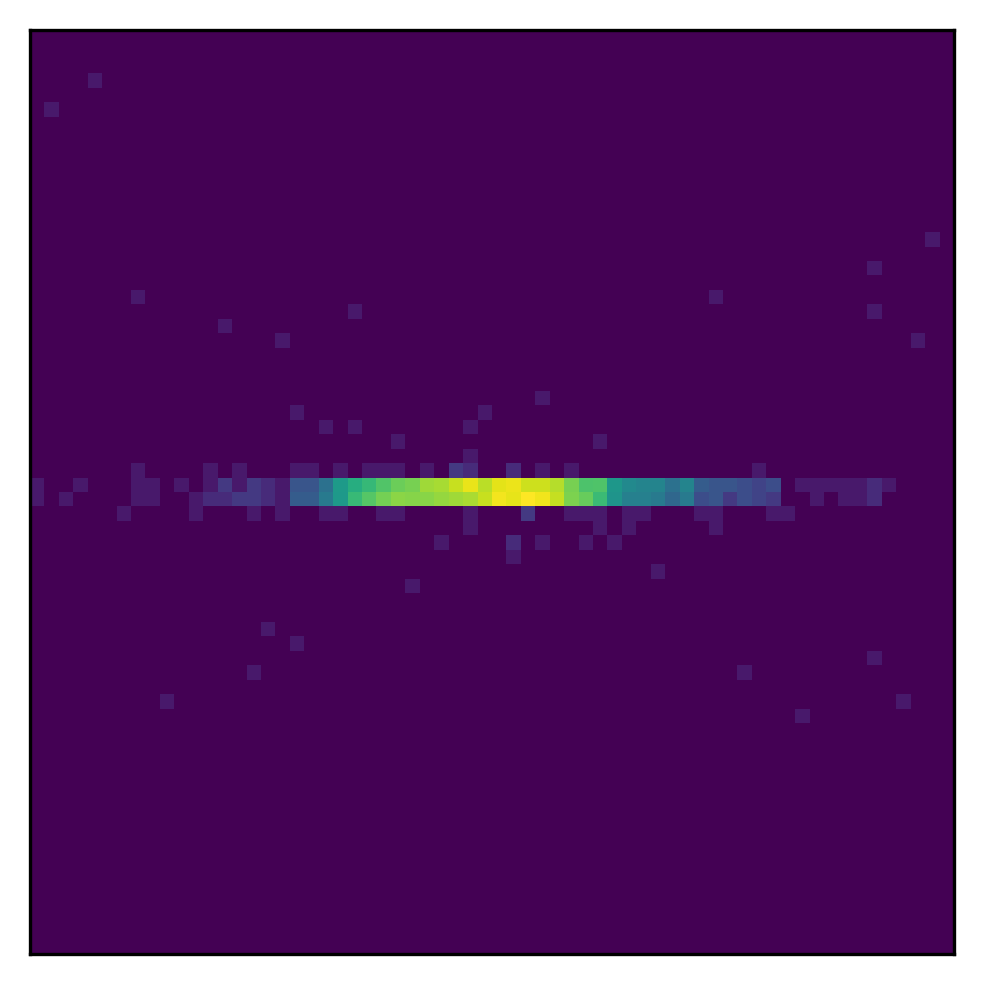}};
    \node[anchor=north, font=\tiny] at (2*\xstep,{-\ystep - 0.5*\imgh - \xlabsep}) {$\partial u/\partial x$};
    \node[rotate=90, anchor=south, font=\tiny] at ({2*\xstep - 0.5*\imgw - \ylabsep},-\ystep) {$\partial^2 u/\partial x^2$};

    \node at (3*\xstep,-\ystep) {\includegraphics[width=\imgw,height=\imgh]{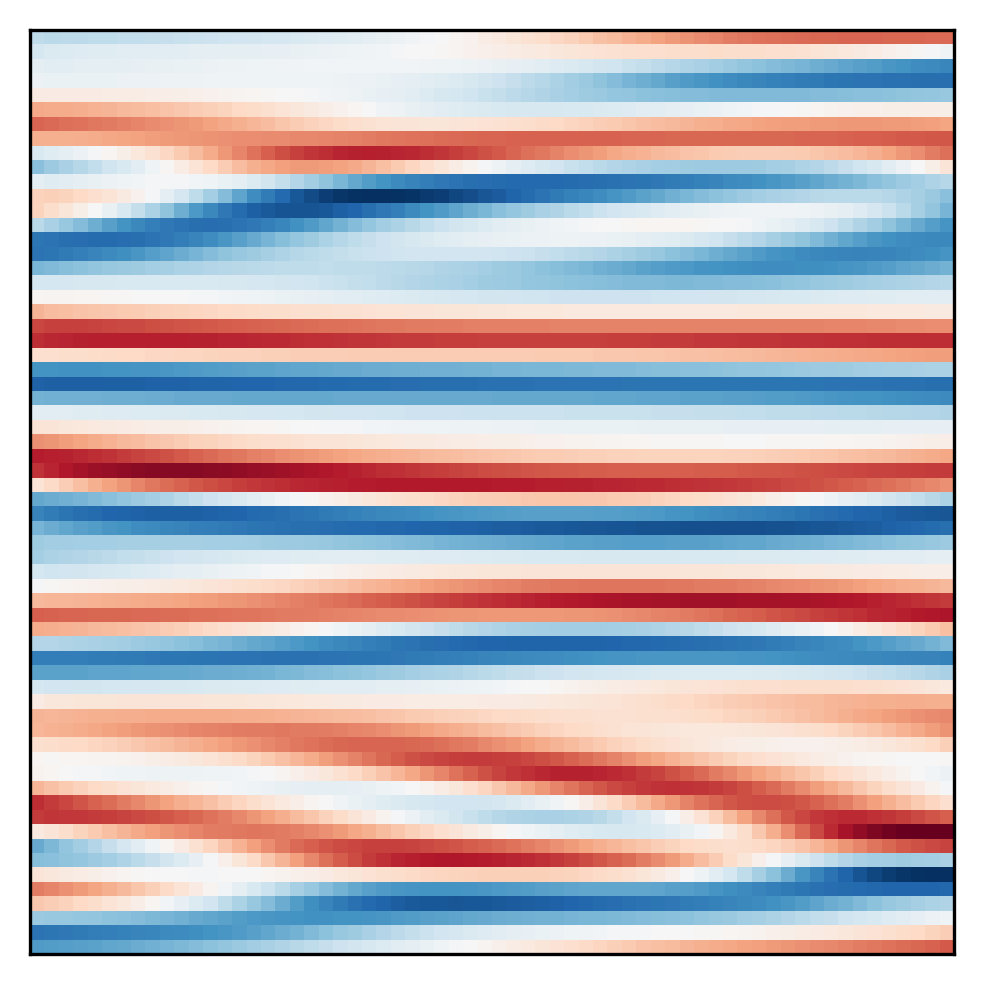}};
    \node[anchor=north, font=\tiny] at (3*\xstep,{-\ystep - 0.5*\imgh - \xlabsep}) {$t$};
    \node[rotate=90, anchor=south, font=\tiny] at ({3*\xstep - 0.5*\imgw - \ylabsep},-\ystep) {$x$};

    \node at (4*\xstep,-\ystep) {\includegraphics[width=\imgw,height=\imgh]{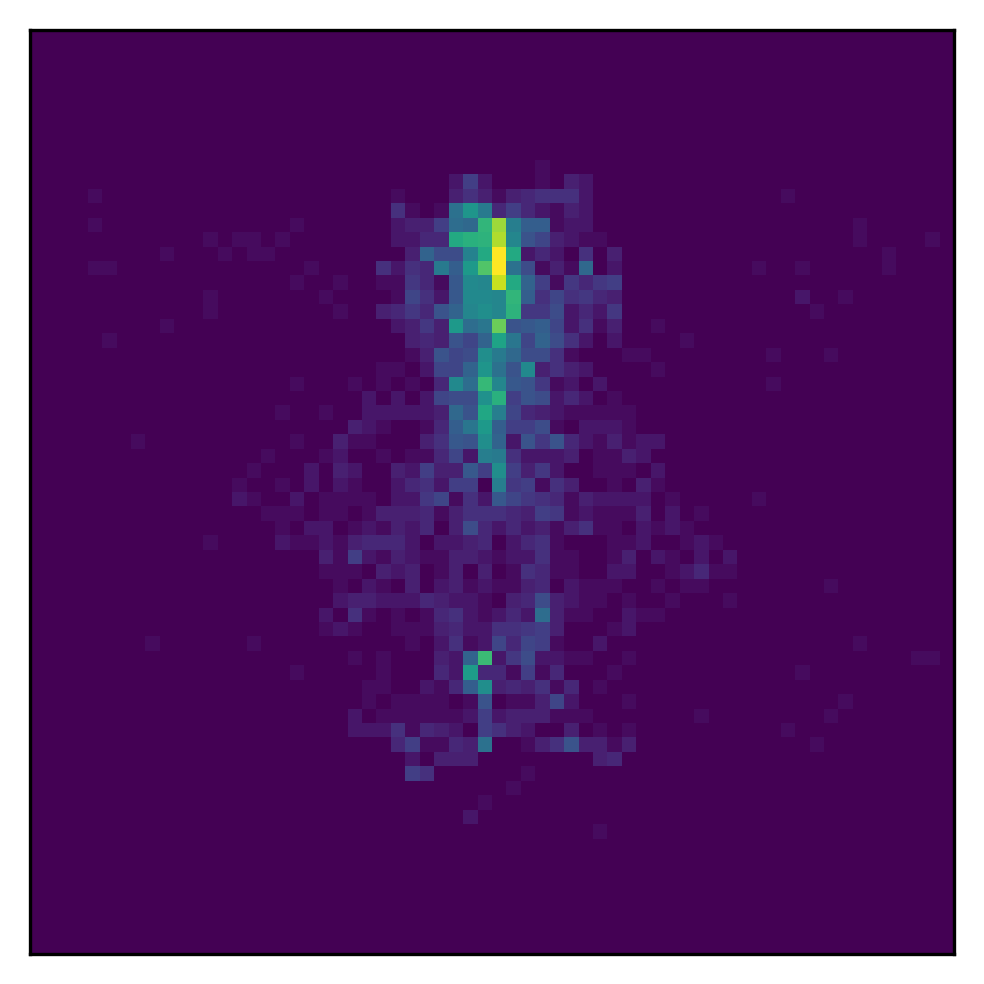}};
    \node[anchor=north, font=\tiny] at (4*\xstep,{-\ystep - 0.5*\imgh - \xlabsep}) {$\partial u/\partial x$};
    \node[rotate=90, anchor=south, font=\tiny] at ({4*\xstep - 0.5*\imgw - \ylabsep},-\ystep) {$\partial u/\partial t$};

    \node at (5*\xstep,-\ystep) {\includegraphics[width=\imgw,height=\imgh]{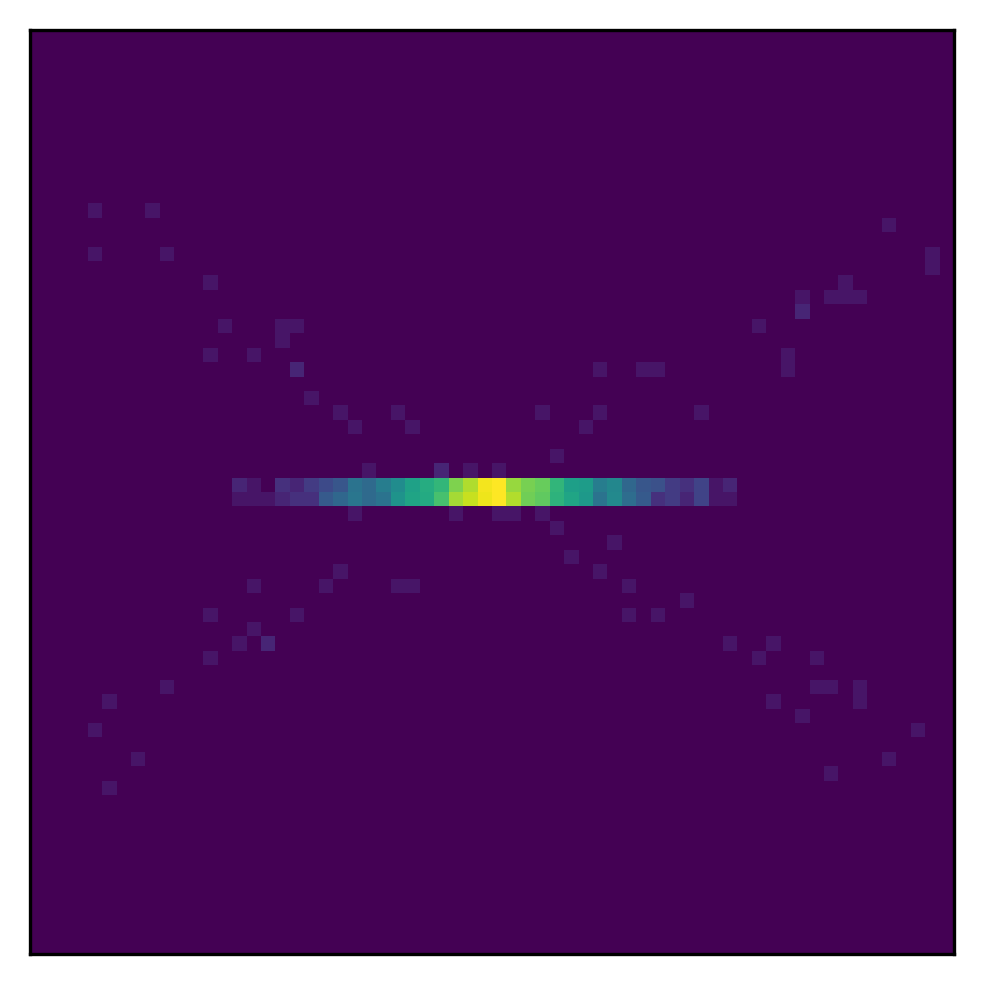}};
    \node[anchor=north, font=\tiny] at (5*\xstep,{-\ystep - 0.5*\imgh - \xlabsep}) {$\partial u/\partial x$};
    \node[rotate=90, anchor=south, font=\tiny] at ({5*\xstep - 0.5*\imgw - \ylabsep},-\ystep) {$\partial^2 u/\partial x^2$};

    \draw[black, dashed, rounded corners=2pt, line width=1pt]
      ({-0.5*\imgw - \gpadx}, {0.5*\imgh + 0.2*\gpady})
      rectangle
      ({2*\xstep + 0.5*\imgw + 0*\gpadx}, {-0.5*\imgh - 2.2*\gpady});

    \draw[black, dashed, rounded corners=2pt, line width=1pt]
      ({3*\xstep - 0.5*\imgw - \gpadx}, {0.5*\imgh +  0.2*\gpady})
      rectangle
      ({5*\xstep + 0.5*\imgw + 0*\gpadx}, {-0.5*\imgh - 2.2*\gpady});

    \draw[black, dashed, rounded corners=2pt, line width=1pt]
      ({-0.5*\imgw - \gpadx}, {-\ystep + 0.5*\imgh +  0.2*\gpady})
      rectangle
      ({2*\xstep + 0.5*\imgw + 0*\gpadx}, {-\ystep - 0.5*\imgh - 2.2*\gpady});

    \draw[black, dashed, rounded corners=2pt, line width=1pt]
      ({3*\xstep - 0.5*\imgw - \gpadx}, {-\ystep + 0.5*\imgh +  0.2*\gpady})
      rectangle
      ({5*\xstep + 0.5*\imgw +0* \gpadx}, {-\ystep - 0.5*\imgh - 2.2*\gpady});
    \end{tikzpicture}$}
    \caption{Single generated sample and its summary statistic in Case 2. For each three-image set, the first shows sampled emulator dynamics, and the last two show histograms of summary statistic distributions. First row: samples generated by the generative models. Left: FM, $\delta_{L_n} = 0.0734>\tau$. Right: PDFM, $\delta_{L_n} = 0.0354<\tau$. Second row: simulator (truth) samples initialized from the same initial states in the first row. }
    \label{fig:case2_individual}
\end{figure}

\begin{figure}[ht]
    \centering
    \begin{tikzpicture}

    \def\imgw{0.3\textwidth}
    \def\imgh{0.3\textwidth}
    \def\xsep{0.35cm}
    \def\ysep{0.55cm}
    \def\xlabsep{-0.08cm}
    \def\ylabsep{-0.15cm}

    \pgfmathsetlengthmacro{\xstep}{\imgw + \xsep}
    \pgfmathsetlengthmacro{\ystep}{\imgh + \ysep}

    \node at (0*\xstep,0) {\includegraphics[width=\imgw,height=\imgh]{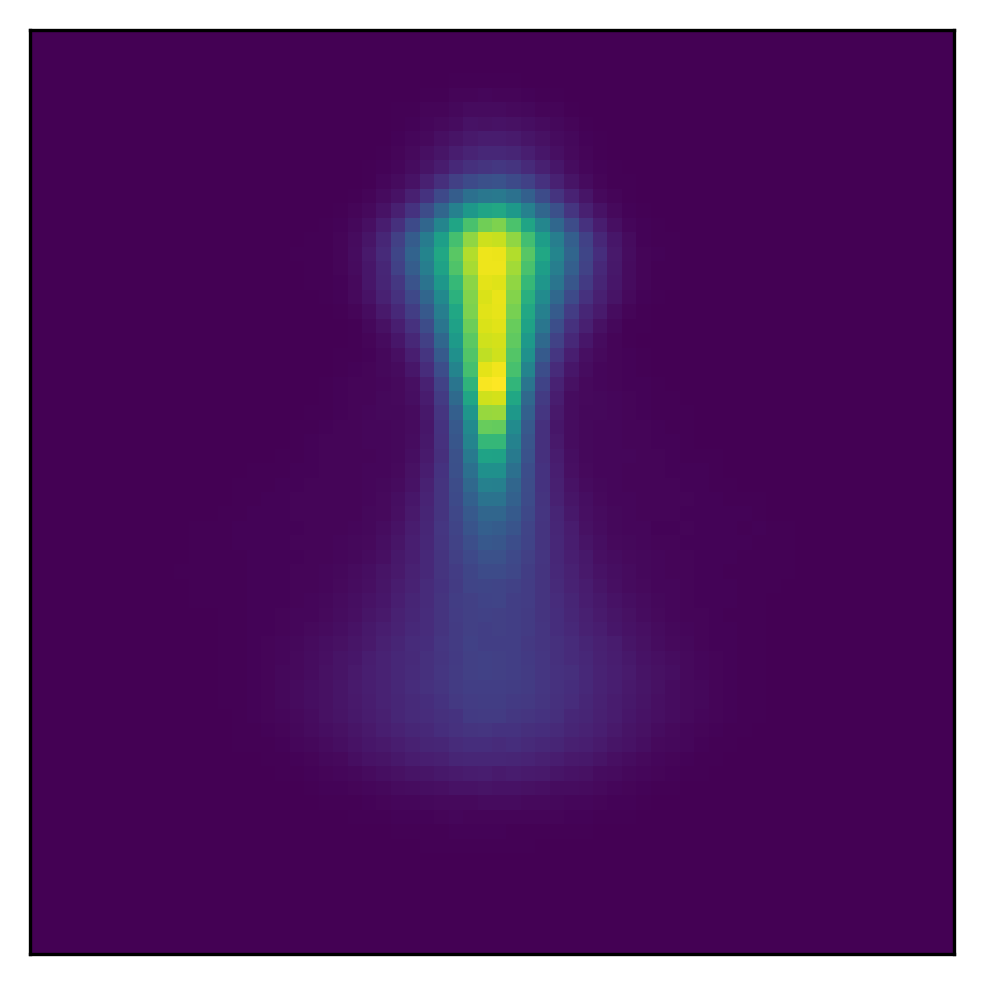}};
    \node[anchor=north, font=\tiny] at (0*\xstep,{0 - 0.5*\imgh - \xlabsep}) {$\partial u/\partial x$};
        \node[anchor=north, font=\small] at (0*\xstep,{0 + 0.6*\imgh - \xlabsep}) {$Q$};
    \node[rotate=90, anchor=south, font=\tiny] at ({0*\xstep - 0.5*\imgw - \ylabsep},0) {$\partial u/\partial t$};

    \node at (1*\xstep,0) {\includegraphics[width=\imgw,height=\imgh]{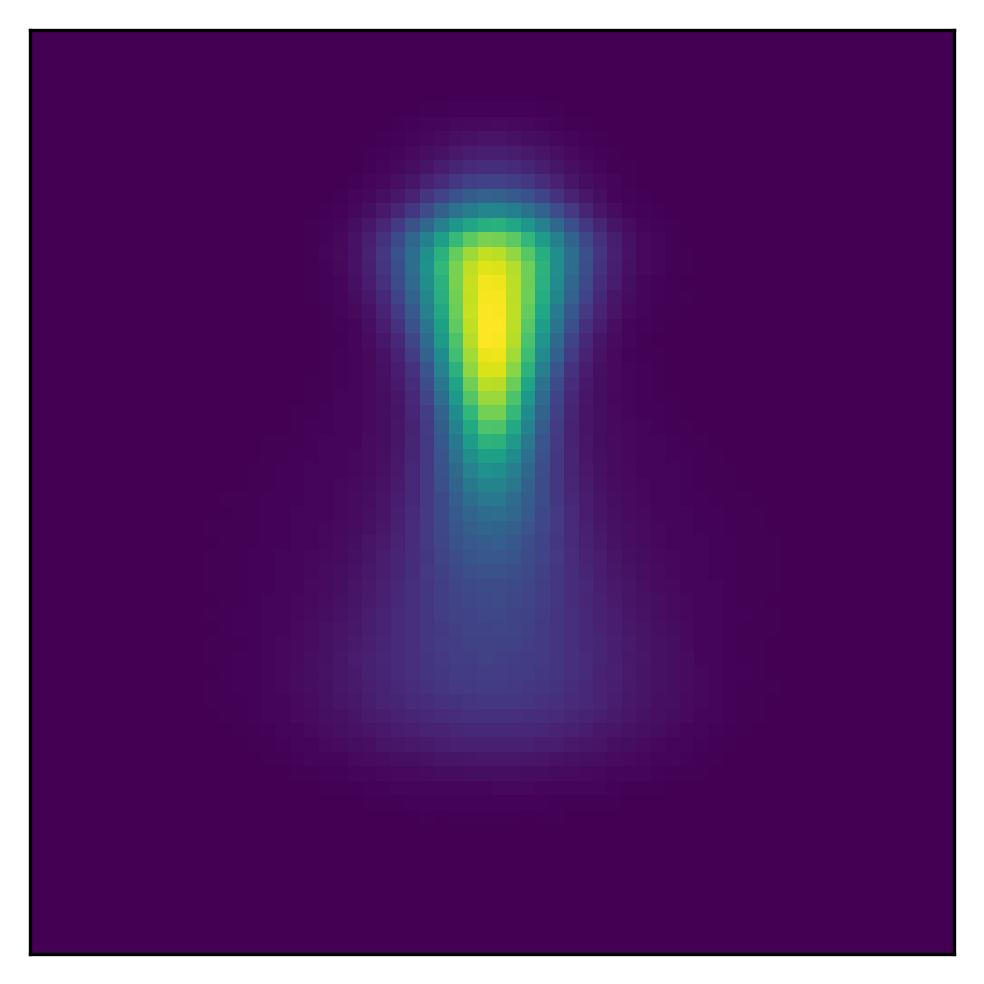}};
    \node[anchor=north, font=\tiny] at (1*\xstep,{0 - 0.5*\imgh - \xlabsep}) {$\partial u/\partial x$};

        \node[anchor=north, font=\small] at (1*\xstep,{0 + 0.6*\imgh - \xlabsep}) {FM};
        
    \node[rotate=90, anchor=south, font=\tiny] at ({1*\xstep - 0.5*\imgw - \ylabsep},0) {$\partial u/\partial t$};

    \node at (2*\xstep,0) {\includegraphics[width=\imgw,height=\imgh]{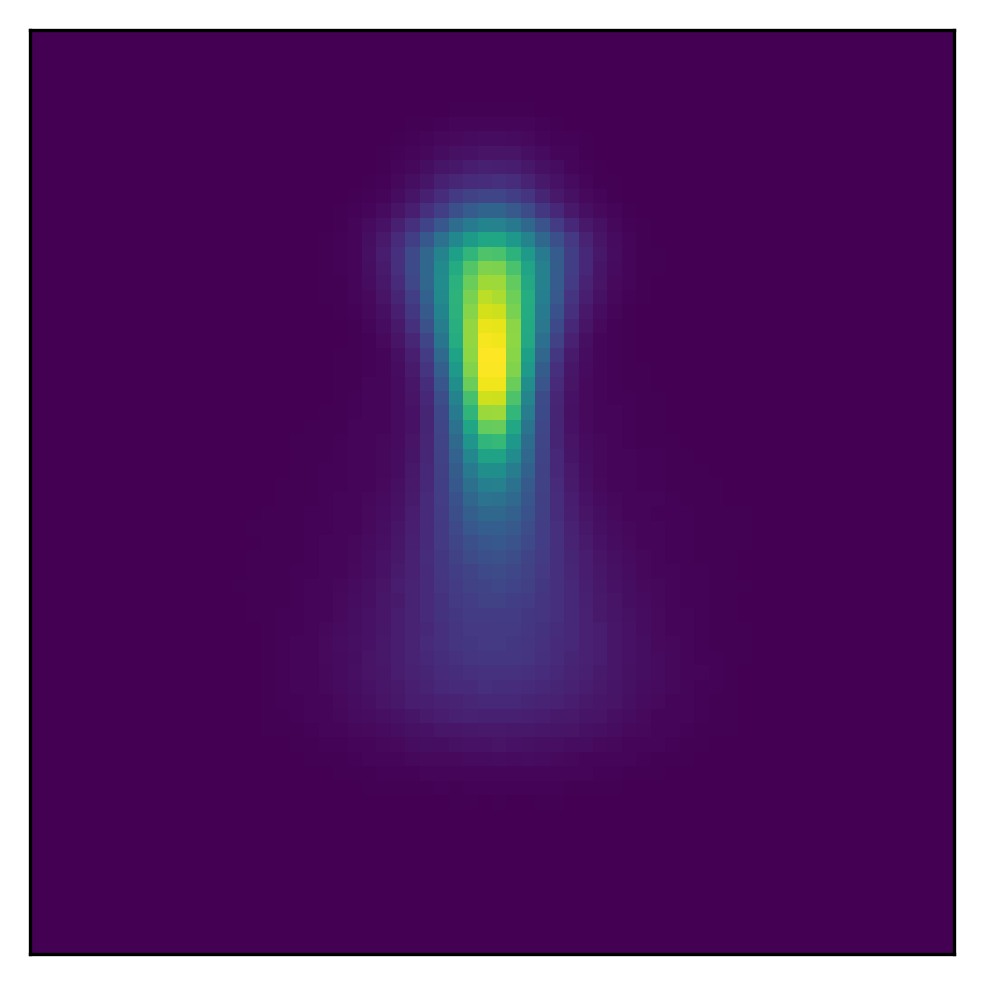}};
    \node[anchor=north, font=\tiny] at (2*\xstep,{0 - 0.5*\imgh - \xlabsep}) {$\partial u/\partial x$};

        \node[anchor=north, font=\small] at (2*\xstep,{0 + 0.6*\imgh - \xlabsep}) {PDFM};
        
    \node[rotate=90, anchor=south, font=\tiny] at ({2*\xstep - 0.5*\imgw - \ylabsep},0) {$\partial u/\partial t$};

    \node at (0*\xstep,-\ystep) {\includegraphics[width=\imgw,height=\imgh]{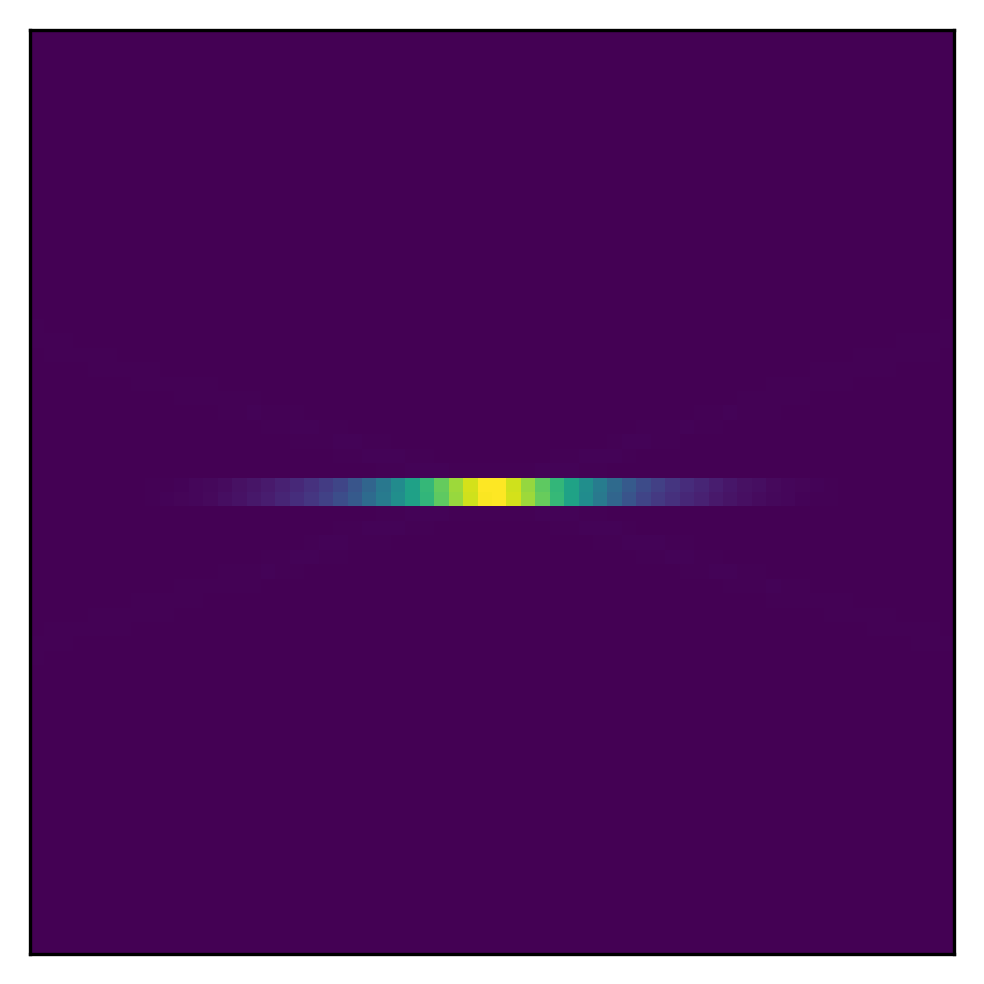}};
    \node[anchor=north, font=\tiny] at (0*\xstep,{-\ystep - 0.5*\imgh - \xlabsep}) {$\partial u/\partial x$};
    \node[rotate=90, anchor=south, font=\tiny] at ({0*\xstep - 0.5*\imgw - \ylabsep},-\ystep) {$\partial^2 u/\partial x^2$};

    \node at (1*\xstep,-\ystep) {\includegraphics[width=\imgw,height=\imgh]{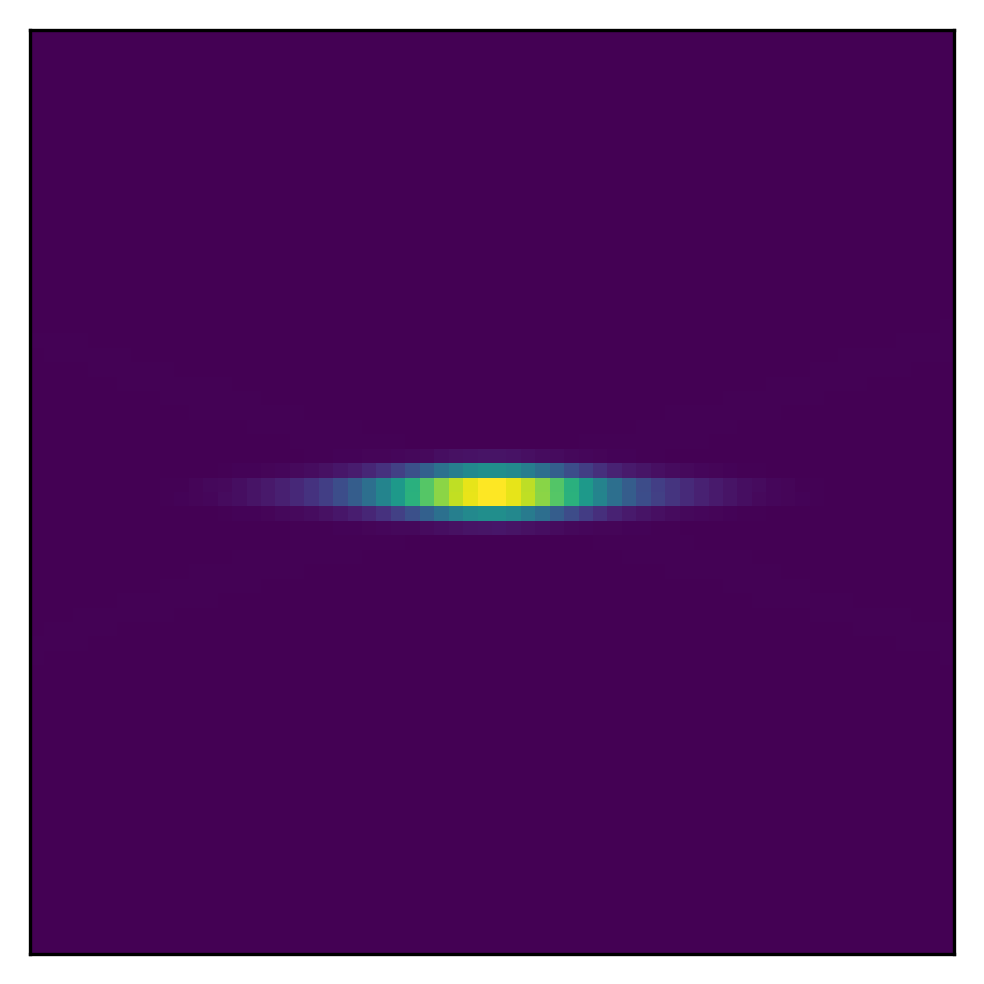}};
    \node[anchor=north, font=\tiny] at (1*\xstep,{-\ystep - 0.5*\imgh - \xlabsep}) {$\partial u/\partial x$};
    \node[rotate=90, anchor=south, font=\tiny] at ({1*\xstep - 0.5*\imgw - \ylabsep},-\ystep) {$\partial^2 u/\partial x^2$};

    \node at (2*\xstep,-\ystep) {\includegraphics[width=\imgw,height=\imgh]{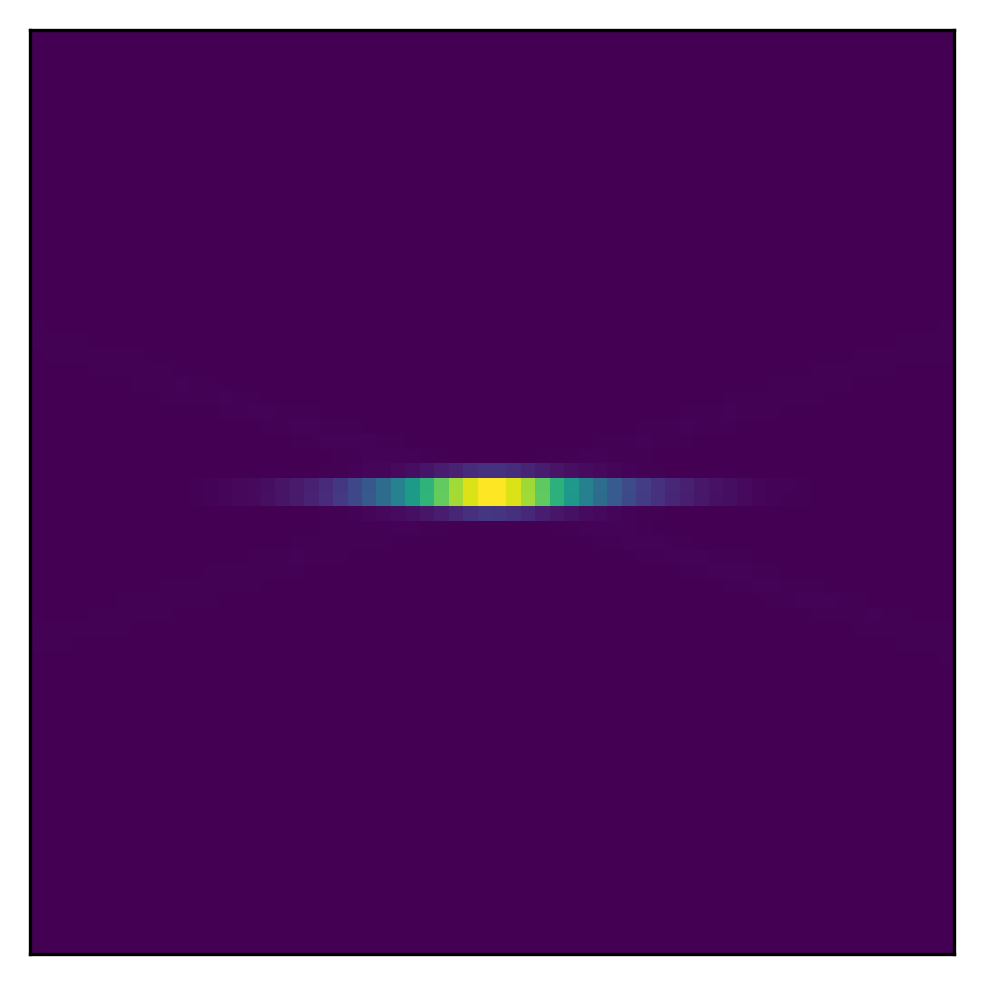}};
    \node[anchor=north, font=\tiny] at (2*\xstep,{-\ystep - 0.5*\imgh - \xlabsep}) {$\partial u/\partial x$};
    \node[rotate=90, anchor=south, font=\tiny] at ({2*\xstep - 0.5*\imgw - \ylabsep},-\ystep) {$\partial^2 u/\partial x^2$};

    \end{tikzpicture}
    \caption{Summary statistic over $2000$ samples in Case 1. First column: training samples. Second column: FM's generated samples. Third column: PDFM's generated samples.}
    \label{fig:case1_batch}
\end{figure}

\begin{figure}[ht]
    \centering
    \begin{tikzpicture}

    \def\imgw{0.22\textwidth}
    \def\imgh{0.22\textwidth}
    \def\xsep{0.3cm}
    \def\ysep{0.55cm}
    \def\xlabsep{-0.08cm}
    \def\ylabsep{-0.15cm}

    \pgfmathsetlengthmacro{\xstep}{\imgw + \xsep}
    \pgfmathsetlengthmacro{\ystep}{\imgh + \ysep}

    \node at (0*\xstep,0) {\includegraphics[width=\imgw,height=\imgh]{figs/KSPDE/batch_truth_ut_vs_ux.png}};
    \node[anchor=north, font=\tiny] at (0*\xstep,{0 - 0.5*\imgh - \xlabsep}) {$\partial u/\partial x$};

        \node[anchor=north, font=\small] at (0*\xstep,{0 + 0.6*\imgh - \xlabsep}) {Clean};
        \node[anchor=north, font=\small] at (1*\xstep,{0 + 0.6*\imgh - \xlabsep}) {$Q$};
        \node[anchor=north, font=\small] at (2*\xstep,{0 + 0.6*\imgh - \xlabsep}) {FM};
        \node[anchor=north, font=\small] at (3*\xstep,{0 + 0.6*\imgh - \xlabsep}) {PDFM};
        
    \node[rotate=90, anchor=south, font=\tiny] at ({0*\xstep - 0.5*\imgw - \ylabsep},0) {$\partial u/\partial t$};

    \node at (1*\xstep,0) {\includegraphics[width=\imgw,height=\imgh]{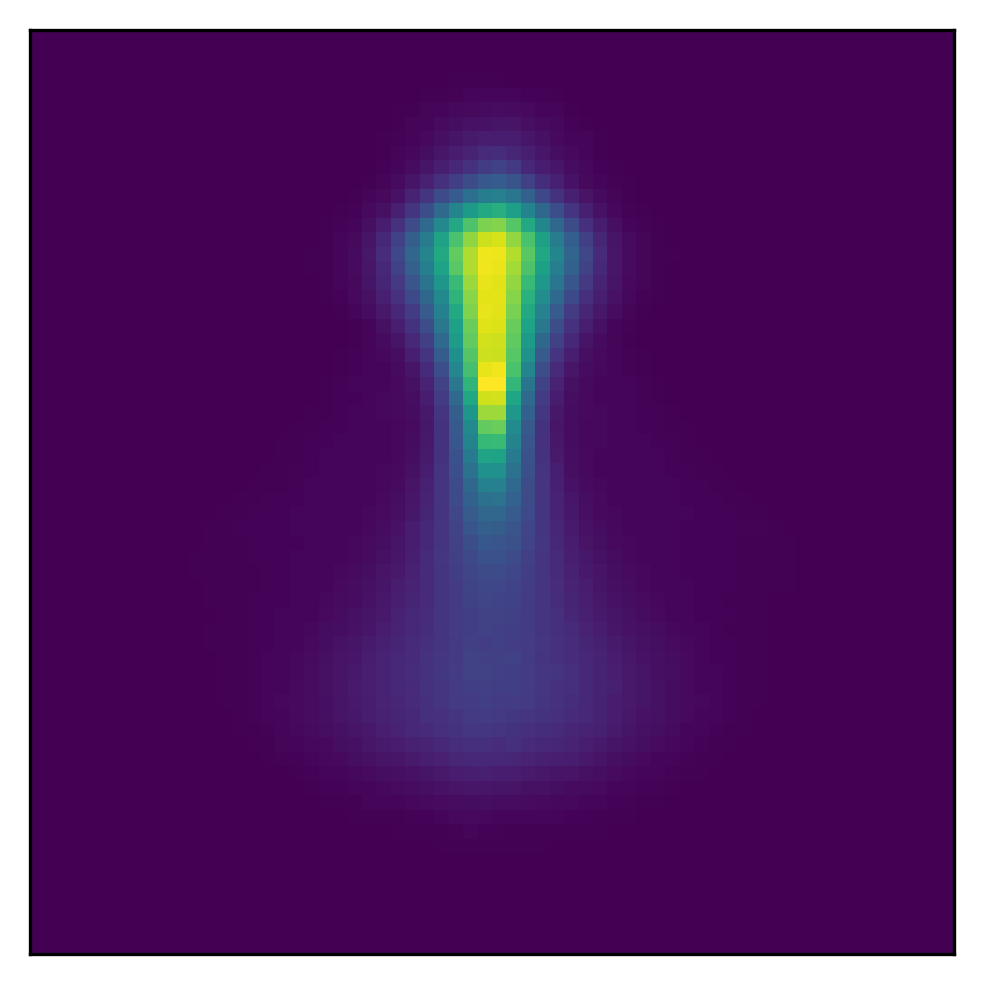}};
    \node[anchor=north, font=\tiny] at (1*\xstep,{0 - 0.5*\imgh - \xlabsep}) {$\partial u/\partial x$};
    \node[rotate=90, anchor=south, font=\tiny] at ({1*\xstep - 0.5*\imgw - \ylabsep},0) {$\partial u/\partial t$};

    \node at (2*\xstep,0) {\includegraphics[width=\imgw,height=\imgh]{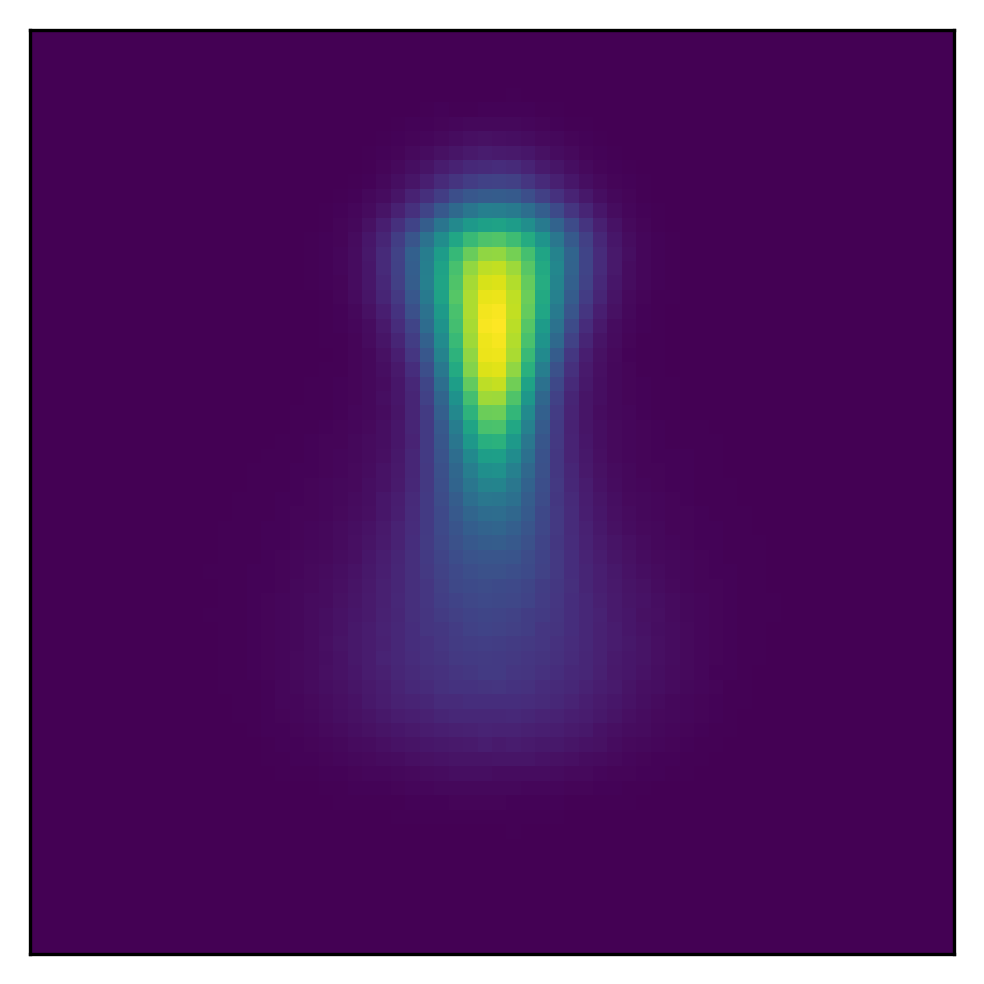}};
    \node[anchor=north, font=\tiny] at (2*\xstep,{0 - 0.5*\imgh - \xlabsep}) {$\partial u/\partial x$};
    \node[rotate=90, anchor=south, font=\tiny] at ({2*\xstep - 0.5*\imgw - \ylabsep},0) {$\partial u/\partial t$};

    \node at (3*\xstep,0) {\includegraphics[width=\imgw,height=\imgh]{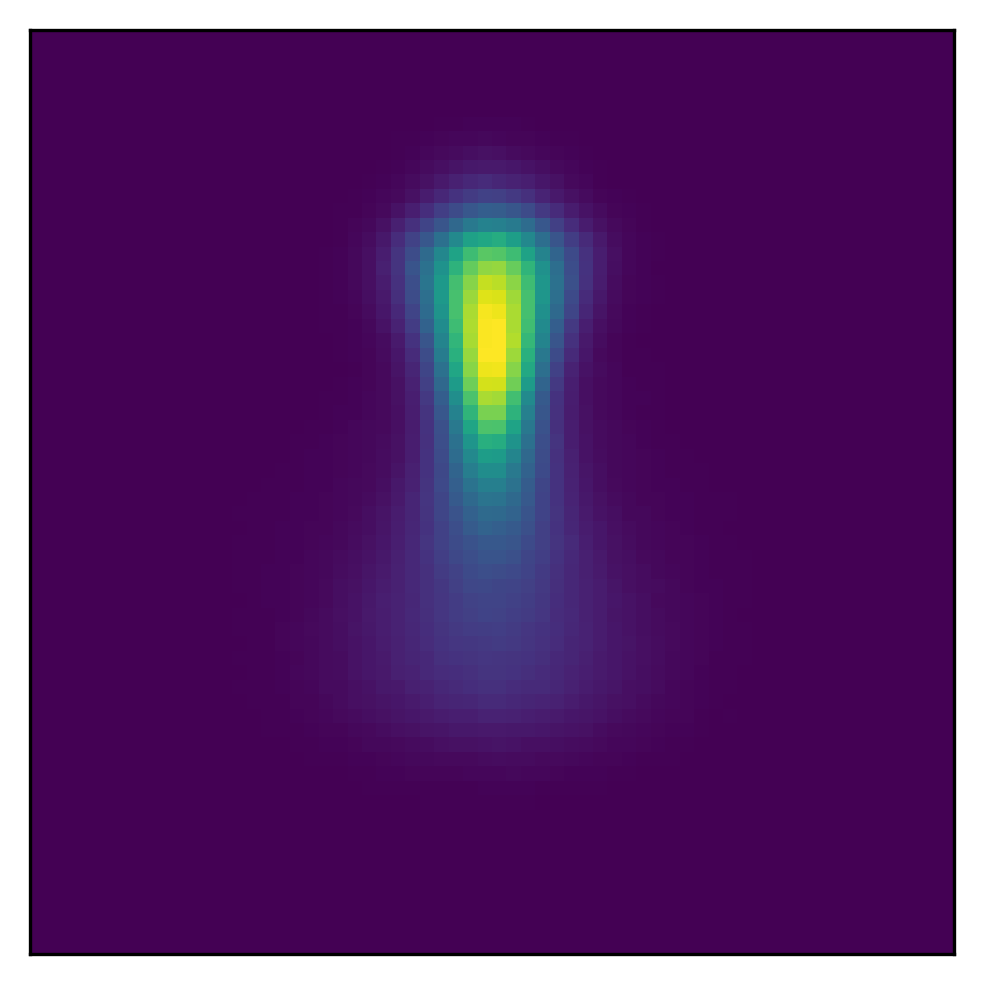}};
    \node[anchor=north, font=\tiny] at (3*\xstep,{0 - 0.5*\imgh - \xlabsep}) {$\partial u/\partial x$};
    \node[rotate=90, anchor=south, font=\tiny] at ({3*\xstep - 0.5*\imgw - \ylabsep},0) {$\partial u/\partial t$};

    \node at (0*\xstep,-\ystep) {\includegraphics[width=\imgw,height=\imgh]{figs/KSPDE/batch_truth_uxx_vs_ux.png}};
    \node[anchor=north, font=\tiny] at (0*\xstep,{-\ystep - 0.5*\imgh - \xlabsep}) {$\partial u/\partial x$};
    \node[rotate=90, anchor=south, font=\tiny] at ({0*\xstep - 0.5*\imgw - \ylabsep},-\ystep) {$\partial^2 u/\partial x^2$};

    \node at (1*\xstep,-\ystep) {\includegraphics[width=\imgw,height=\imgh]{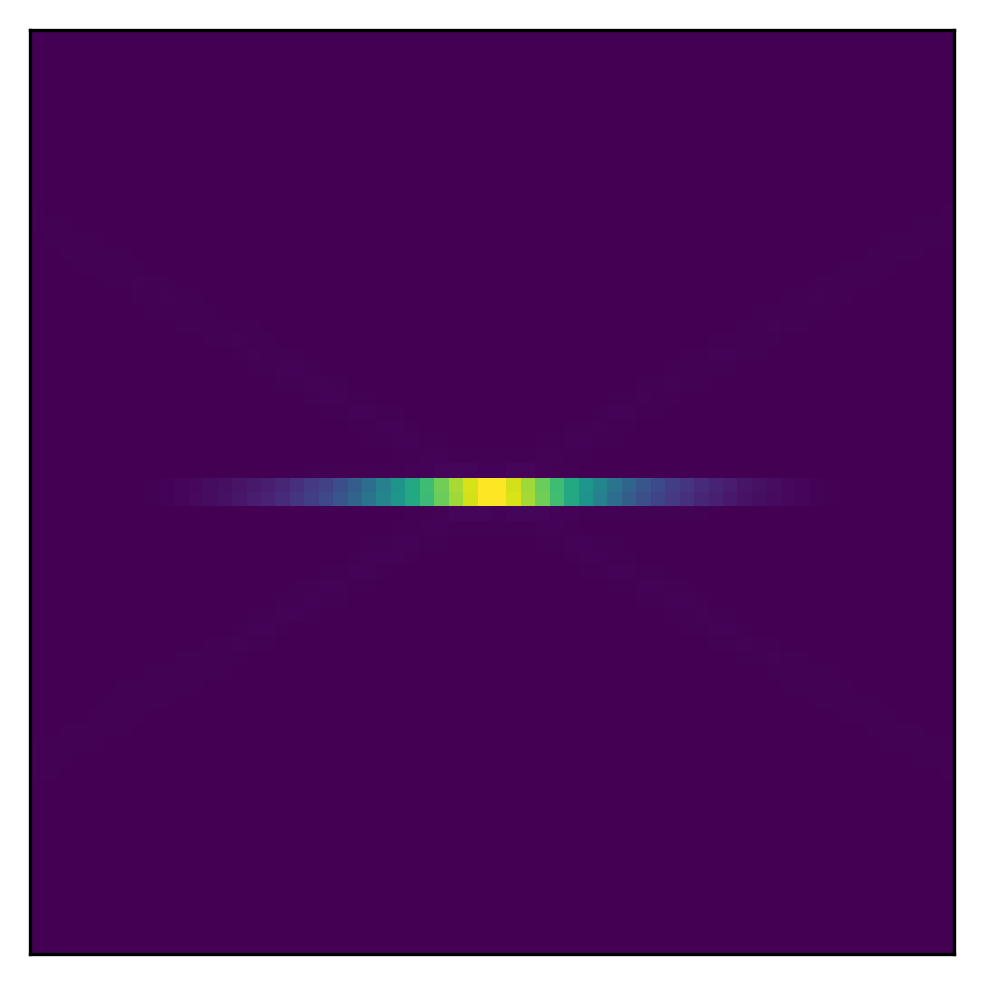}};
    \node[anchor=north, font=\tiny] at (1*\xstep,{-\ystep - 0.5*\imgh - \xlabsep}) {$\partial u/\partial x$};
    \node[rotate=90, anchor=south, font=\tiny] at ({1*\xstep - 0.5*\imgw - \ylabsep},-\ystep) {$\partial^2 u/\partial x^2$};

    \node at (2*\xstep,-\ystep) {\includegraphics[width=\imgw,height=\imgh]{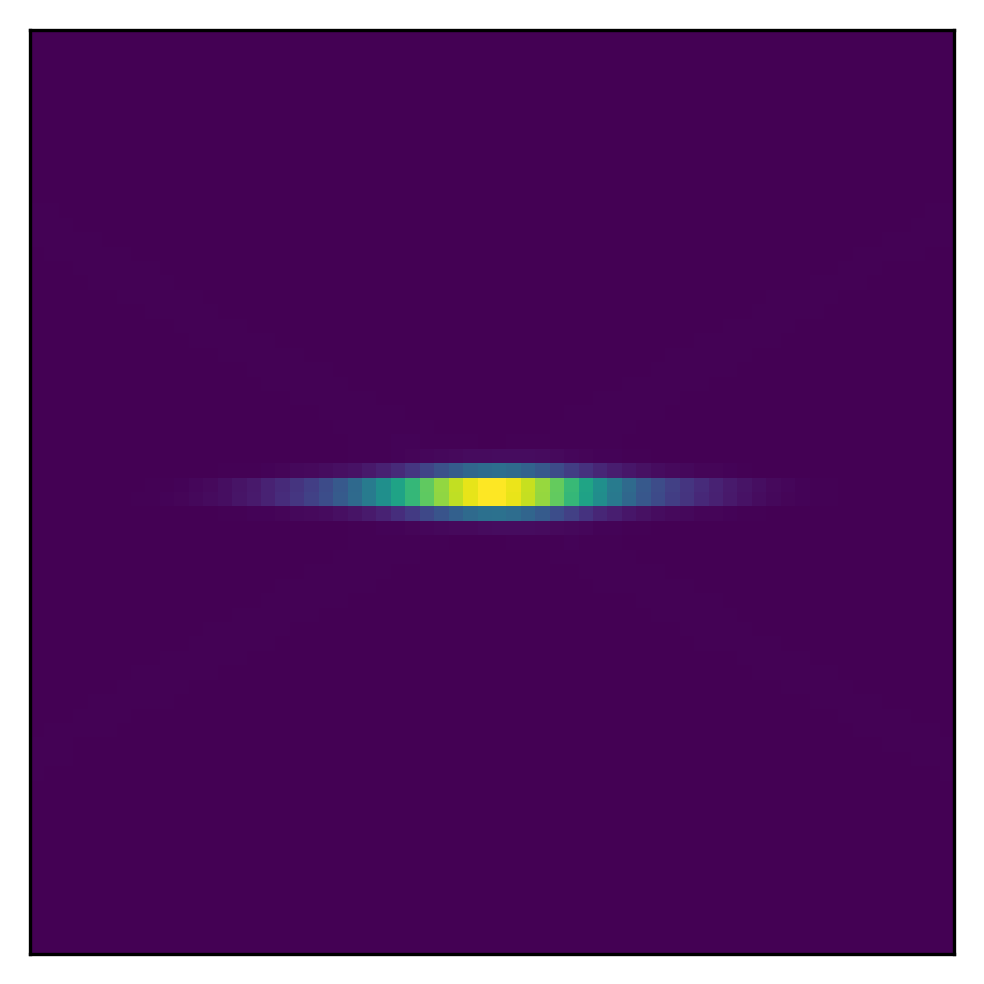}};
    \node[anchor=north, font=\tiny] at (2*\xstep,{-\ystep - 0.5*\imgh - \xlabsep}) {$\partial u/\partial x$};
    \node[rotate=90, anchor=south, font=\tiny] at ({2*\xstep - 0.5*\imgw - \ylabsep},-\ystep) {$\partial^2 u/\partial x^2$};

    \node at (3*\xstep,-\ystep) {\includegraphics[width=\imgw,height=\imgh]{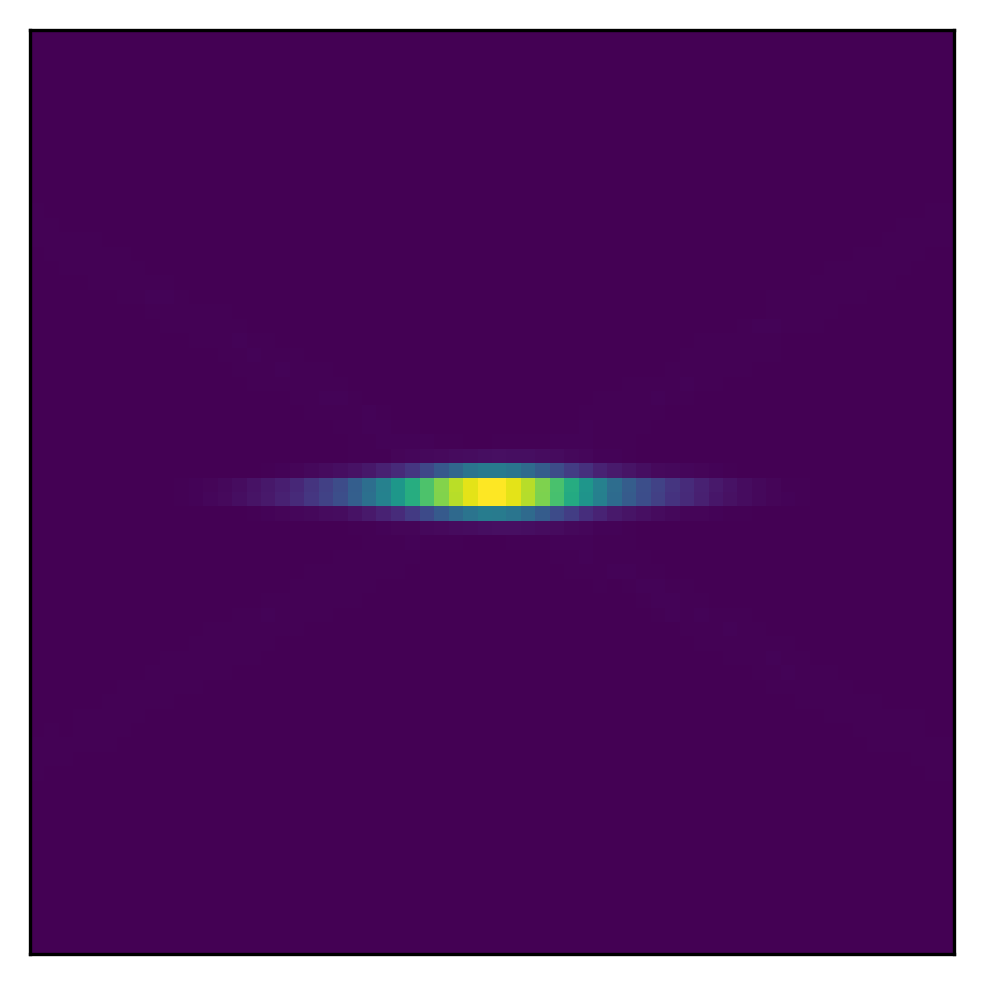}};
    \node[anchor=north, font=\tiny] at (3*\xstep,{-\ystep - 0.5*\imgh - \xlabsep}) {$\partial u/\partial x$};
    \node[rotate=90, anchor=south, font=\tiny] at ({3*\xstep - 0.5*\imgw - \ylabsep},-\ystep) {$\partial^2 u/\partial x^2$};

    \end{tikzpicture}
        \caption{Summary statistic over $2000$ samples in Case 2. First column: clean samples. Second column: degraded samples used during training. Third column: FM's generated samples. Fourth column: PDFM's generated samples.}
    \label{fig:case2_batch}
\end{figure}

\begin{figure}[htb]
    \centering
    \begin{tikzpicture}
        \node[anchor=center, rotate=90] at (0,0) {\textbf{Train}};
        \node[anchor=west] at (0.35,0)
        {\includegraphics[width=0.95\linewidth]{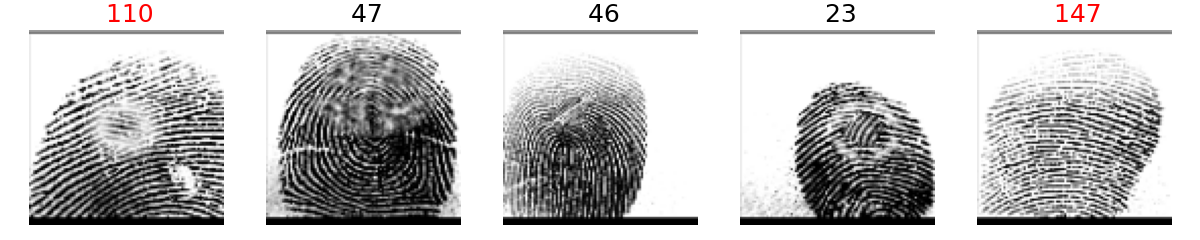}};

        \node[anchor=center, rotate=90] at (0,-2.2) {\textbf{FM}};
        \node[anchor=west] at (0.35,-2.2)
        {\includegraphics[width=0.95\linewidth]{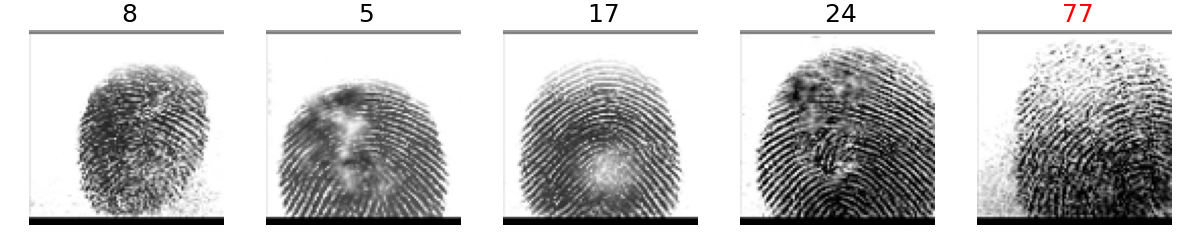}};

        \node[anchor=center, rotate=90] at (0,-4.4) {\textbf{FM-RE}};
        \node[anchor=west] at (0.35,-4.4)
        {\includegraphics[width=0.95\linewidth]{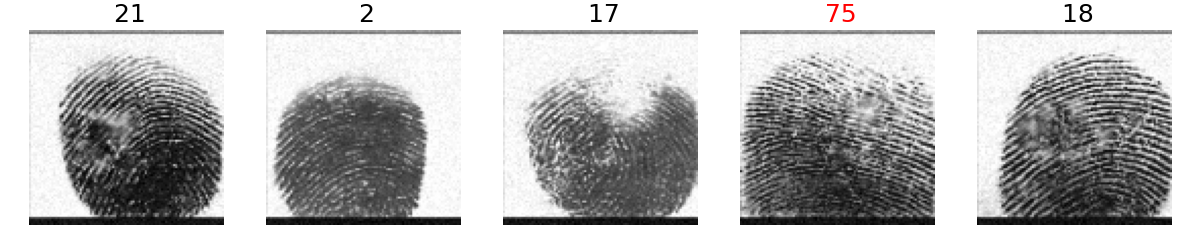}};

        \node[anchor=center, rotate=90] at (0,-6.6) {\textbf{PDFM}};
        \node[anchor=west] at (0.35,-6.6)
        {\includegraphics[width=0.95\linewidth]{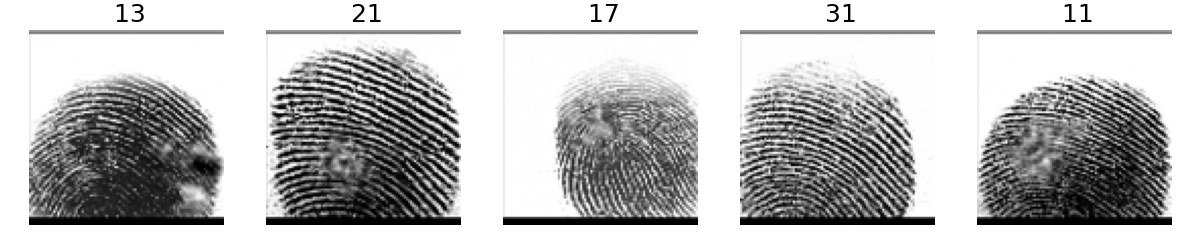}};
    \end{tikzpicture}
    \caption{Generated fingerprint samples from different methods. The title of each sample reports its number of connected components, with infeasible samples highlighted in red.}
    \label{fig:fingerprint}
\end{figure}

\begin{figure}
    \centering
    \includegraphics[width=1\linewidth]{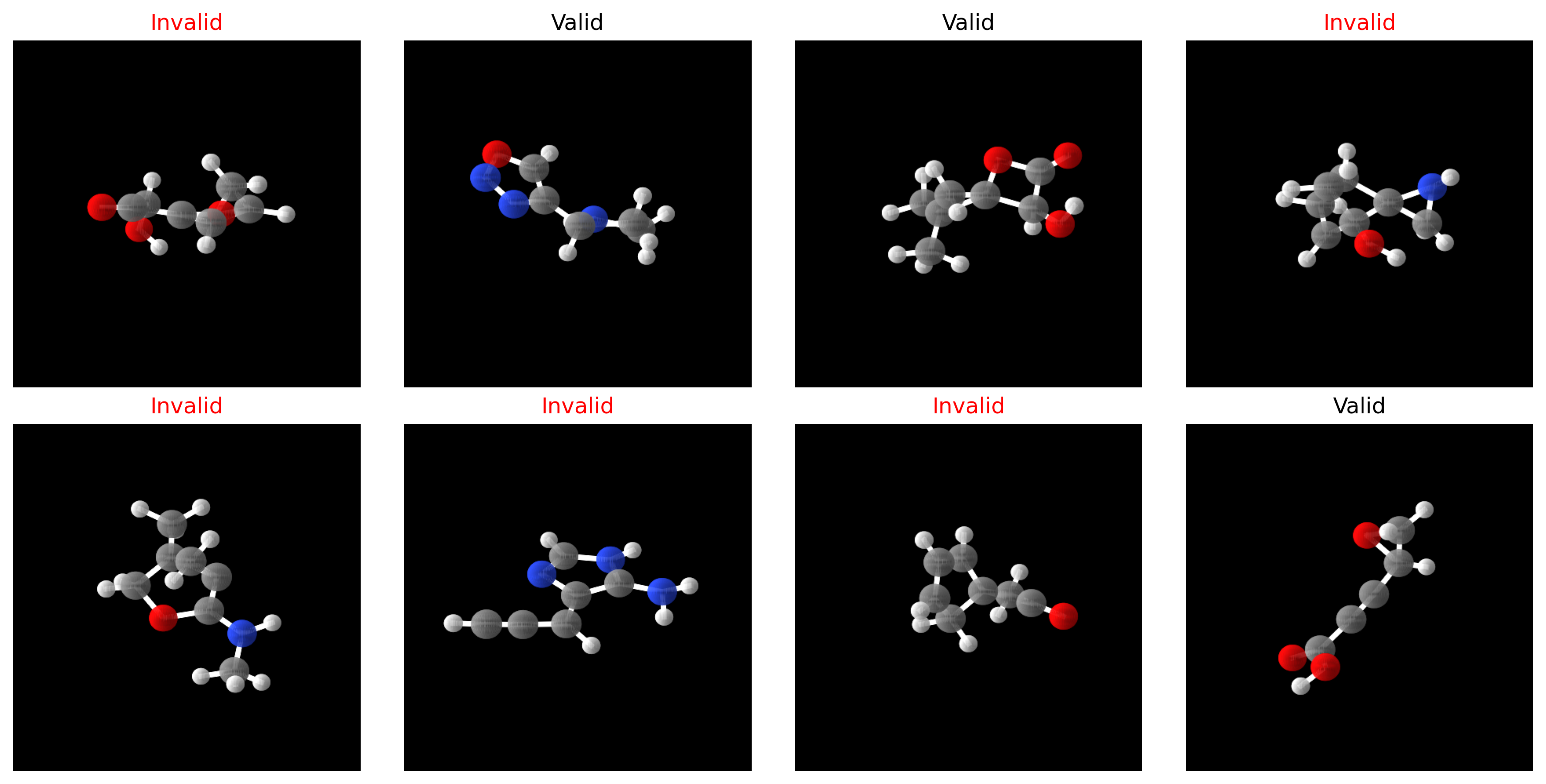}
    \caption{Visualization of the $3$-D molecules generated by FM. The titles indicate whether the generated sample satisfies the molecule-stability constraint.}
    \label{fig:FM_mole}
\end{figure}

\begin{figure}
    \centering
    \includegraphics[width=1\linewidth]{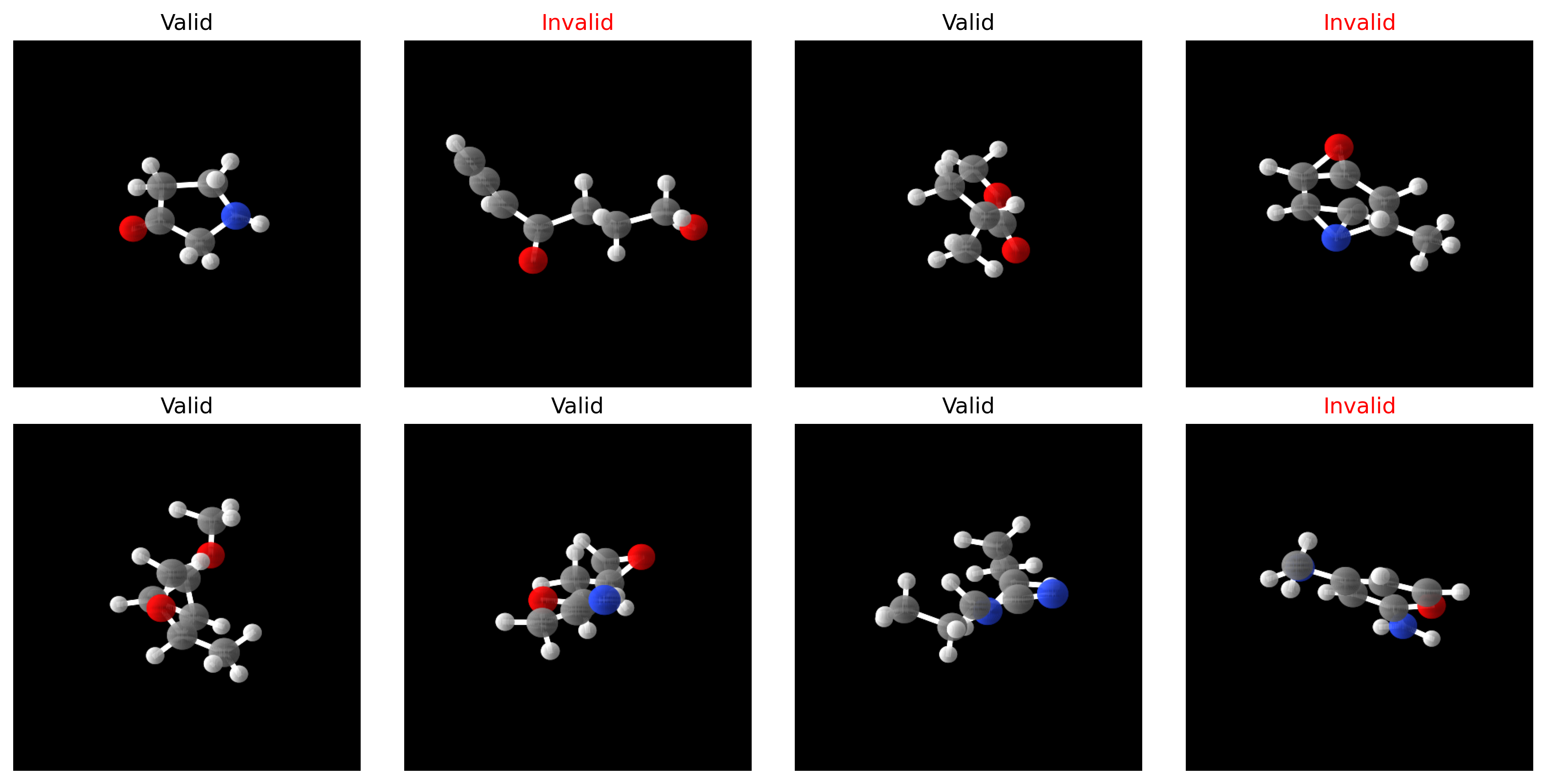}
    \caption{Visualization of the $3$-D molecules generated by PDFM. The titles indicate whether the generated sample satisfies the molecule-stability constraint.}
    \label{fig:PDFM_mole}
\end{figure}

~
\clearpage
\section*{NeurIPS Paper Checklist}


\begin{enumerate}

\item {\bf Claims}
    \item[] Question: Do the main claims made in the abstract and introduction accurately reflect the paper's contributions and scope?
    \item[] Answer: \answerYes{} 
    \item[] Justification: The theoretical and empirical claims in the abstract match the content in the main text.

\item {\bf Limitations}
    \item[] Question: Does the paper discuss the limitations of the work performed by the authors?
    \item[] Answer: \answerYes{}{} 
    \item[] Justification: Limitations are discussed in detail in Appendix~\ref{app:discussion}.
    
\item {\bf Theory assumptions and proofs}
    \item[] Question: For each theoretical result, does the paper provide the full set of assumptions and a complete (and correct) proof?
    \item[] Answer:\answerYes{}
    \item[] Justification: The assumptions are provided in the main text, with detailed proofs provided in the appendix.
    \item {\bf Experimental result reproducibility}
    \item[] Question: Does the paper fully disclose all the information needed to reproduce the main experimental results of the paper to the extent that it affects the main claims and/or conclusions of the paper (regardless of whether the code and data are provided or not)?
    \item[] Answer: \answerYes{} 
\item[] Justification: The code is released.

\item {\bf Open access to data and code}
    \item[] Question: Does the paper provide open access to the data and code, with sufficient instructions to faithfully reproduce the main experimental results, as described in supplemental material?
    \item[] Answer: \answerYes{} 
    \item[] Justification: The PDFM for implementation is provided in Alg.~\ref{alg:pd_practical}, with more details discussed in Appendix~\ref{app:discussion}. We will release the code if the paper is accepted.
    \begin{itemize}
        \item The answer \answerNA{} means that paper does not include experiments requiring code.
        \item Please see the NeurIPS code and data submission guidelines (\url{https://neurips.cc/public/guides/CodeSubmissionPolicy}) for more details.
        \item While we encourage the release of code and data, we understand that this might not be possible, so \answerNo{} is an acceptable answer. Papers cannot be rejected simply for not including code, unless this is central to the contribution (e.g., for a new open-source benchmark).
        \item The instructions should contain the exact command and environment needed to run to reproduce the results. See the NeurIPS code and data submission guidelines (\url{https://neurips.cc/public/guides/CodeSubmissionPolicy}) for more details.
        \item The authors should provide instructions on data access and preparation, including how to access the raw data, preprocessed data, intermediate data, and generated data, etc.
        \item The authors should provide scripts to reproduce all experimental results for the new proposed method and baselines. If only a subset of experiments are reproducible, they should state which ones are omitted from the script and why.
        \item At submission time, to preserve anonymity, the authors should release anonymized versions (if applicable).
        \item Providing as much information as possible in supplemental material (appended to the paper) is recommended, but including URLs to data and code is permitted.
    \end{itemize}
\item {\bf Experimental setting/details}
    \item[] Question: Does the paper specify all the training and test details (e.g., data splits, hyperparameters, how they were chosen, type of optimizer) necessary to understand the results?
    \item[] Answer: \answerYes{} 
    \item[] Justification: The settings are described in the main text as well as in the appendix.

\item {\bf Experiment statistical significance}
    \item[] Question: Does the paper report error bars suitably and correctly defined or other appropriate information about the statistical significance of the experiments?
    \item[] Answer: \answerYes{} 
    \item[] Justification: We report means and standard deviations over repeated trials in the results.

\item {\bf Experiments compute resources}
    \item[] Question: For each experiment, does the paper provide sufficient information on the computer resources (type of compute workers, memory, time of execution) needed to reproduce the experiments?
    \item[] Answer: \answerYes{} 
    \item[] Justification: The computation resource information is provided in Appendix~\ref{app:exp}.
    
\item {\bf Code of ethics}
    \item[] Question: Does the research conducted in the paper conform, in every respect, with the NeurIPS Code of Ethics \url{https://neurips.cc/public/EthicsGuidelines}?
    \item[] Answer: \answerYes{} 
    \item[] Justification: This paper conforms with the NeurIPS Code of Ethics.

\item {\bf Broader impacts}
    \item[] Question: Does the paper discuss both potential positive societal impacts and negative societal impacts of the work performed?
    \item[] Answer: \answerYes{} 
    \item[] Justification: The impact is discussed in the abstract, as well as in the contributions.
    
\item {\bf Safeguards}
    \item[] Question: Does the paper describe safeguards that have been put in place for responsible release of data or models that have a high risk for misuse (e.g., pre-trained language models, image generators, or scraped datasets)?
    \item[] Answer: \answerNA{} 
    \item[] Justification:The paper poses no such risks.

\item {\bf Licenses for existing assets}
    \item[] Question: Are the creators or original owners of assets (e.g., code, data, models), used in the paper, properly credited and are the license and terms of use explicitly mentioned and properly respected?
    \item[] Answer: \answerYes{} 
    \item[] Justification: Related works are properly cited.

\item {\bf New assets}
    \item[] Question: Are new assets introduced in the paper well documented and is the documentation provided alongside the assets?
    \item[] Answer: \answerNA{} 
    \item[] Justification:  Code will be released on GitHub once the paper is accepted.

\item {\bf Crowdsourcing and research with human subjects}
    \item[] Question: For crowdsourcing experiments and research with human subjects, does the paper include the full text of instructions given to participants and screenshots, if applicable, as well as details about compensation (if any)? 
    \item[] Answer: \answerNA{} 
    \item[] Justification: The paper does not involve crowdsourcing nor research with human subjects.
    
\item {\bf Institutional review board (IRB) approvals or equivalent for research with human subjects}
    \item[] Question: Does the paper describe potential risks incurred by study participants, whether such risks were disclosed to the subjects, and whether Institutional Review Board (IRB) approvals (or an equivalent approval/review based on the requirements of your country or institution) were obtained?
    \item[] Answer: \answerNA{} 
    \item[] Justification: The paper does not involve crowdsourcing nor research with human subjects.

\item {\bf Declaration of LLM usage}
    \item[] Question: Does the paper describe the usage of LLMs if it is an important, original, or non-standard component of the core methods in this research? Note that if the LLM is used only for writing, editing, or formatting purposes and does \emph{not} impact the core methodology, scientific rigor, or originality of the research, declaration is not required.
    \item[] Answer: \answerNA{} 
    \item[] Justification: The core method development in this research does not involve LLMs as any important, original, or non-standard components.

\end{enumerate}

\end{document}